\documentclass[10pt]{article}

\usepackage{url}
\usepackage{mathtools}
\usepackage{amssymb}
\usepackage{amsmath}
\usepackage{bbm}
\usepackage{mathrsfs}
\usepackage{graphicx}
\usepackage{amsthm}
\usepackage{float}
\usepackage{aliascnt}
\usepackage{empheq}
\usepackage{latexsym}
\usepackage{enumitem}
\usepackage{eurosym}
\usepackage{dsfont}
\usepackage{appendix}
\usepackage{color} 
\usepackage{frcursive}
\usepackage[utf8]{inputenc}
\usepackage[T1]{fontenc}
\usepackage{geometry}
\usepackage{multirow}
\usepackage{booktabs}
\usepackage{adjustbox}
\usepackage[colorinlistoftodos]{todonotes}
\usepackage{lmodern}
\usepackage{anyfontsize}
\usepackage{stmaryrd}
\usepackage{natbib}
\usepackage[unicode,breaklinks]{hyperref}
\usepackage[nameinlink,noabbrev]{cleveref}
\usepackage{comment}
\usepackage[most]{tcolorbox} 

\usepackage{graphicx}
\usepackage{subcaption}

\setcitestyle{numbers,open={[},close={]}}

\definecolor{red}{rgb}{0.7,0.15,0.15}
\definecolor{green}{rgb}{0,0.5,0}
\definecolor{blue}{rgb}{0,0,0.7}
\definecolor{darkcyan}{rgb}{0.0, 0.55, 0.55}
\definecolor{MidnightBlue}{RGB}{25,25,112}
\definecolor{MidnightBlueComplementingGreen}{RGB}{25,112,25}
\definecolor{MidnightBlueComplementingPurple}{RGB}{112,25,112}
\definecolor{MidnightBlueComplementingRed}{RGB}{112,25,69}
\definecolor{WowColor}{rgb}{.75,0,.75}
\definecolor{MildlyAlarming}{rgb}{0.85,0.25,0.1}
\definecolor{SubtleColor}{rgb}{0,0,.50}
\definecolor{antiquefuchsia}{rgb}{0.57, 0.36, 0.51}
\definecolor{fashionfuchsia}{rgb}{0.96, 0.0, 0.63}
\definecolor{jade}{rgb}{0.0, 0.66, 0.42}
\definecolor{caribbeangreen}{rgb}{0.0, 0.8, 0.6}
\definecolor{aquamarine}{rgb}{0.5, 0.8, 0.85}
\definecolor{darkmidnightblue}{rgb}{0.0, 0.2, 0.4}
\definecolor{attentioncolor}{RGB}{152,90,81}
\definecolor{burgred}{RGB}{40,3,22}
\definecolor{AKGreen}{RGB}{17,123,92}
\definecolor{egyptianblue}{rgb}{0.06, 0.2, 0.65}
\definecolor{Turquoise}{RGB}{64,224,208}
\definecolor{darkjade}{RGB}{0,122,84}
\definecolor{Window1}{RGB}{92,150,31}
\definecolor{Window1dark}{RGB}{41,67,13}
\definecolor{Window2}{RGB}{255,168,28}
\definecolor{Window2dark}{RGB}{114,75,12}
\definecolor{Window3}{RGB}{255,96,33}
\definecolor{Window3dark}{RGB}{97,36,12}
\definecolor{InputColor}{RGB}{20,255,177}
\definecolor{InputColorlight}{RGB}{222,237,229}
 \definecolor{richblack}{rgb}{0.0, 0.25, 0.25}

\hypersetup{colorlinks, linkcolor={red},citecolor={green}, urlcolor={blue}}
			
\makeatletter \@addtoreset{equation}{section}
\def\theequation{\thesection.\arabic{equation}}

\newtheorem{theorem}{Theorem}[section]

\newaliascnt{assumption}{theorem}
\newtheorem{assumption}[assumption]{Assumption}
\aliascntresetthe{assumption}

\newaliascnt{proposition}{theorem}
\newtheorem{proposition}[proposition]{Proposition}
\aliascntresetthe{proposition}

\newaliascnt{definition}{theorem}
\newtheorem{definition}[definition]{Definition}
\aliascntresetthe{definition}

\newaliascnt{lemma}{theorem}
\newtheorem{lemma}[lemma]{Lemma}
\aliascntresetthe{lemma}

\newaliascnt{example}{theorem}
\newtheorem{example}[example]{Example}
\aliascntresetthe{example}

\newaliascnt{corollary}{theorem}

\aliascntresetthe{corollary}

\newaliascnt{setting}{theorem}
\newtheorem{setting}[setting]{Setting}
\aliascntresetthe{setting}

\newaliascnt{remark}{theorem}

\aliascntresetthe{remark}

\newaliascnt{condition}{theorem}

\aliascntresetthe{condition}

\crefname{theorem}{theorem}{theorems}
\Crefname{theorem}{Theorem}{Theorems}

\crefname{assumption}{assumption}{assumptions}
\Crefname{assumption}{Assumption}{Assumptions}

\crefname{proposition}{proposition}{propositions}
\Crefname{proposition}{Proposition}{Propositions}

\crefname{definition}{definition}{definitions}
\Crefname{definition}{Definition}{Definitions}

\crefname{lemma}{lemma}{lemmas}
\Crefname{lemma}{Lemma}{Lemmas}

\crefname{example}{example}{examples}
\Crefname{example}{Example}{Examples}

\crefname{corollary}{corollary}{corollaries}
\Crefname{corollary}{corollary}{Corollaries}

\crefname{remark}{remark}{remarks}
\Crefname{remark}{remark}{Remarks}

\crefname{setting}{setting}{settings}
\Crefname{setting}{setting}{Settings}

\crefname{condition}{condition}{conditions}
\Crefname{condition}{Condition}{Conditions}

\def \E{\mathbb{E}}
\def \F{\mathbb{F}}

\def \L{\mathbb{L}}

\def \N{\mathbb{N}}
\def \P{\mathbb{P}}

\def \R{\mathbb{R}}
\def \S{\mathbb{S}}

\def \Z{\mathbb{Z}}

\def\Bc{{\cal B}}

\def\Dc{{\cal D}}

\def\Fc{{\cal F}}

\def\Lc{{\cal L}}

\def\Xc{{\cal X}}

\newcommand{\smallertext}[1]{\text{\fontsize{5}{5}\selectfont$#1$}} 
\newcommand{\smalltext}[1]{\text{\fontsize{4}{4}\selectfont$#1$}}
\newcommand{\tinytext}[1]{\text{\fontsize{3}{3}\selectfont$#1$}}

\def\eps{\varepsilon}

\newcommand{\Diffuse}{
 \Sigma
}

\definecolor{britishracinggreen}{rgb}{0.0, 0.26, 0.15}
\definecolor{darkcyan}{rgb}{0.0, 0.55, 0.55}
\definecolor{MidnightBlue}{RGB}{25,25,112}
\definecolor{MidnightBlueComplementingGreen}{RGB}{25,112,25}
\definecolor{MidnightBlueComplementingPurple}{RGB}{112,25,112}
\definecolor{MidnightBlueComplementingRed}{RGB}{112,25,69}
\definecolor{WowColor}{rgb}{.75,0,.75}
\definecolor{MildlyAlarming}{rgb}{0.85,0.25,0.1}
\definecolor{SubtleColor}{rgb}{0,0,.50}
\definecolor{antiquefuchsia}{rgb}{0.57, 0.36, 0.51}
\definecolor{fashionfuchsia}{rgb}{0.96, 0.0, 0.63}
\definecolor{jade}{rgb}{0.0, 0.66, 0.42}
\definecolor{caribbeangreen}{rgb}{0.0, 0.8, 0.6}
\definecolor{aquamarine}{rgb}{0.5, 0.8, 0.85}
\definecolor{lightseagreen}{rgb}{0.13, 0.7, 0.67}
\definecolor{darkgreen}{rgb}{0.0, 0.2, 0.13}
\definecolor{darkspringgreen}{rgb}{0.09, 0.45, 0.27}
\definecolor{attentioncolor}{RGB}{152,90,81}
\definecolor{burgred}{RGB}{40,3,22}
\definecolor{AnnieGreen}{RGB}{17,123,92}
\definecolor{Turquoise}{RGB}{64,224,208}

\definecolor{darkjade}{RGB}{0,122,84}
\definecolor{Window1}{RGB}{92,150,31}%
    \definecolor{Window1dark}{RGB}{41,67,13}%
\definecolor{Window2}{RGB}{255,168,28}
    \definecolor{Window2dark}{RGB}{114,75,12}
\definecolor{Window3}{RGB}{255,96,33}
    \definecolor{Window3dark}{RGB}{97,36,12}
\definecolor{InputColor}{RGB}{20,255,177}
    \definecolor{InputColorlight}{RGB}{222,237,229}

\definecolor{RedAlizarin}{rgb}{0.82, 0.1, 0.26}
\usepackage{soul}

\usepackage[colorinlistoftodos]{todonotes}

\NewDocumentCommand{\Takashi}{mo}{
    \IfValueF{#2}{
                        {{
                            \textcolor{magenta}{ 
                            [\textbf{Takashi:}
                            \text{\rm {#1}}]
                            }
                        }}
        }
    \IfValueT{#2}{
                        \marginnote{{\scriptsize
                            \textcolor{magenta}{ 
                            \textbf{T:}
                            \textit{{#1}}
                            }
                        }}
        }
                    }

\definecolor{darkmidnightblue}{rgb}{0.0, 0.2, 0.4}

\usepackage{xcolor}
\usepackage[most]{tcolorbox} 
\usepackage{marginnote}
\definecolor{lightturquoise}{RGB}{175, 238, 238} 

\NewDocumentCommand{\AK}{mo}{%
    \IfValueF{#2}{%
        \hfill\\
        \begin{tcolorbox}[colback=lightturquoise!50, colframe=darkmidnightblue, 
                          boxrule=0.5pt, arc=4pt, left=2pt, right=2pt, top=2pt, bottom=2pt,
                          sharp corners=south, enhanced]
            \textcolor{darkmidnightblue}{\textbf{AK:} \textit{#1}}
        \end{tcolorbox}%
    }%
    \IfValueT{#2}{%
        \marginnote{%
            \begin{tcolorbox}[colback=lightturquoise!50, colframe=darkmidnightblue, 
                              boxrule=0.5pt, arc=4pt, left=2pt, right=2pt, top=2pt, bottom=2pt,
                              enhanced]
                \textcolor{darkmidnightblue}{\textbf{AK:} \textit{#1}}
            \end{tcolorbox}%
        }%
    }%
}

\NewDocumentCommand{\add}{m}{{\color{britishracinggreen}{{#1}}}}

\newcommand\numberthis{\addtocounter{equation}{1}\tag{\theequation}}
\newcommand{\uu}{{\textrm{u}}}
\newcommand{\up}{{\textrm{u}^{\prime}}}
\newcommand{\upp}{{\textrm{u}^{\prime \prime}}}

\allowdisplaybreaks

\author{Takashi {\sc Furuya} \footnote{Department of Biomedical Engineering, Doshisha University, Kyōto, Japan, takashi.furuya0101@gmail.com.} \and Anastasis {\sc Kratsios} \footnote{Department of Mathematics, McMaster University, McMaster University, Hamilton, Canada, kratsiosa@mcmaster.ca.} \and Dylan {\sc Possama\"{i}} \footnote{ETH Z\"urich, Department of Mathematics, Zürich, Switzerland, dylan.possamai@math.ethz.ch.} \and Bogdan {\sc Raoni\'{c}} \footnote{ETH Z\"urich, Seminar for Applied Mathematics and ETH AI Center, Zürich, Switzerland, braonic@ethz.ch.}}

\title{One model to solve them all: 2BSDE families via neural operators}

\date{\today}	

\begin{document}

\maketitle

\begin{abstract} 
We introduce a mild generative variant of the classical neural operator model, which leverages Kolmogorov--Arnold networks to solve infinite families of second-order backward stochastic differential equations ($2$BSDEs) on regular bounded Euclidean domains with random terminal time. Our first main result shows that the solution operator associated with a broad range of $2$BSDE families is approximable by appropriate neural operator models. We then identify a structured subclass of (infinite) families of $2$BSDEs whose neural operator approximation requires only a polynomial number of parameters in the reciprocal approximation rate, as opposed to the exponential requirement in general worst-case neural operator guarantees.

\end{abstract}

\section{Introduction}
\label{s:Intro}
Fix a positive integer $d\in\N^\star$. We work on a filtered probability space $\big(\Omega,\Fc,\F\coloneqq (\Fc_t)_{t\in[0,\infty)},\P\big)$ carrying a $d$-dimensional $(\F,\P)$--Brownian motion $W$. Fix a sufficiently regular bounded open domain $\mathcal{D}\subset \mathbb{R}^d$, as well as maps $\mu:\mathbb{R}^d\longrightarrow \mathbb{R}^d$, $\Sigma:\mathbb{R}^d\longrightarrow \mathbb{R}^{d\times d}$, and $f:\mathbb{R}^d\times \mathbb{R}\times \R^d\times\R^{d\times d}\longrightarrow \mathbb{R}$, as well as an initial point $x\in \mathcal{D}$.
We are interested in \textit{simultaneously} approximately solving each 2BSDE in the (non-empty) compact infinite family $\mathcal{B}\subseteq  (X_{\cdot},Y_{\cdot}^{g,f_\smalltext{0}},Z_{\cdot}^{g,f_\smalltext{0}},\Upsilon^{g,f_\smalltext{0}},A^{g,f_\smalltext{0}})\}_{(g,f_\smalltext{0})\in\mathfrak W}$ where $\mathfrak W$ is a suitable subset of the Sobolev spaces $W^{1,p}(\partial \mathcal{D})\times W^{1,p}(\mathcal{D})$. These 2BSDEs are defined through the system
\begin{align}
\tag{SDE}
\label{eq:FBSDE_ForwardProcess}
        X_t 
    & = 
        x+\int_0^t \beta(X_s)\mathrm{d}s + \int_0^t\,
        \Diffuse
        (X_s)
        \mathrm{d}W_s,\; t\geq 0,\; \P\text{\rm--a.s.},\; \tau\coloneqq\inf\big\{t\geq 0: X_t\notin \Dc\big\},
\\
\notag
Y_t^{g,f_\smalltext{0}} & = 
        \underbrace{
            g(X_{\tau})
        }_{\text{\tiny Perturbation}}
    + \int_{t\wedge \tau}^{\tau}
        \bigg(
            \underbrace{
                f\big(X_s,Y_s^{g,f_\smalltext{0}},Z_s^{g,f_\smalltext{0}},\Upsilon_s^{g,f_\smalltext{0}}\big)
            }_{\text{\tiny Reference generator}}
            +
            \underbrace{
                f_0(X_s)
            }_{\text{\tiny Perturbation}}-\frac12\mathrm{Tr}\big[\Sigma(X_s)\Sigma^\top(X_s)\Upsilon_s^{g,f_\smalltext{0}}\big]
        \bigg)
    \mathrm{d}s\\
    \label{eq:FBSDE}
\tag{FBSDE}
    &\quad  - \int_{t\wedge \tau}^{\tau}Z_s^{g,f_\smalltext{0}} 
     \cdot 
    \mathrm{d}X_s,\; t\in[0,\tau),\; \P\text{\rm--a.s.},
\\
\label{eq:FBSDE_Martingale}
\tag{2BSDE}
Z_t^{g,f_\smalltext{0}} & = z_0 + \int_0^t A_s^{g,f_\smalltext{0}} \mathrm{d}s + \int_0^t \Upsilon_s^{g,f_\smalltext{0}} \mathrm{d}X_s, \; t\in[0,\tau),\; \P\text{\rm--a.s.}.
\end{align}

Using a variant (see \Cref{s:MainResults__ss:PDEForm} below for the proof) of the results of~\citeauthor*{cheridito2007second} \cite{cheridito2007second} for 2BSDEs with random terminal time $\tau$, as above, for each pair $(g,f_0)\in\mathfrak W$, if the following elliptic PDE 
\begin{gather}
\label{eq:AssociatePDE_General}
 f\big(x,u(x), \nabla u(x), \nabla^2 u(x)\big)  = -f_0(x),\; x\in \mathcal{D}
\;
u(x) = g(x),\; x\in \partial \mathcal{D},
\end{gather}
admits a smooth enough solution, then the 2BSDE system~\eqref{eq:FBSDE_ForwardProcess}, \eqref{eq:FBSDE}, \eqref{eq:FBSDE_Martingale} admits a solution of the form
\begin{equation}
\label{eq:PDE_to_2BSDE}
\begin{aligned}
    Y_t^{g,f_\smalltext{0}} 
    & = 
    u(X_t),\; Z_t^{g,f_\smalltext{0}}  = 
    \nabla u(X_t),\; \Upsilon_t^{g,f_\smalltext{0}}  = 
    \nabla ^2u(X_t),\; A_t^{g,f_\smalltext{0}} =
    \mathcal{L}\nabla u(X_t),\; t\in[0,\tau),\; \P\text{\rm--a.s.},
\end{aligned}
\end{equation}
where $\mathcal{L}$ denotes the generator associated to the forward process $X$ (without the drift term), defined for any continuous bounded test function $f$ on $\mathbb{R}^d$ by
\[
\mathcal{L}(f)\coloneqq  
\frac{1}{2} \mathrm{Tr}\big[ 
\Diffuse(x)\Diffuse(x)^{\top}
\nabla^2 f(x) \big],\; x\in\R^d,
\]
see~\cite[Equations (2.9) and (2.11)]{cheridito2007second} for a similar result in the parabolic case. 

\medskip
Our first main result, \Cref{thrm:generalapprox}, guarantees that the following solution map is approximable by a neural operator
\begin{equation}
\label{eq:Solution_Map}
\begin{aligned}
    {\Gamma^\smallertext{+}}: W^{1,\infty}(\mathcal{D};\mathbb{R})\times W^{1,\infty}(\mathcal{D};\mathbb{R})
    & \longrightarrow 
    W^{1,\infty}(\mathcal{D};\mathbb{R})
    \\
    (f_0,g)& \longmapsto u
\end{aligned}
\end{equation}
where $f_0$ and $g$ are the source and boundary data of the PDE in~\eqref{eq:AssociatePDE_General}, respectively; which equivalently perturb the generator and the terminal condition of the associated 2BSDEs with random terminal time $\tau$ in~\eqref{eq:FBSDE}.

\medskip
Consequently, the solution map associated to the family of second-order BSDEs is approximable by our stochastic neural operator model (which extends the neural operator model of \citeauthor*{furuya2024simultaneously} in \cite[Definition 4]{furuya2024simultaneously} from the classical BSDE setting to $2$BSDEs). This result thus provides a $2$BSDE analogue of neural operator approximability results, which typically follow a two-step strategy: first, establish a quantitative universal approximation theorem for general H\"{o}lder-continuous functions with the same source and target as the solution map (see \emph{e.g.} \citeauthor*{lu2021learning} \cite{lu2021learning}, \citeauthor*{korolev2022two} \cite{korolev2022two}, \citeauthor*{galimberti2022designing} \cite{galimberti2022designing}, \citeauthor*{yu2021arbitrary} \cite{yu2021arbitrary}, \citeauthor*{lanthaler2022error} \cite{lanthaler2022error}, \citeauthor*{lu2019deeponet} \cite{lu2019deeponet}, \citeauthor*{lanthaler2025nonlocality} \cite{lanthaler2025nonlocality}, \citeauthor*{kratsios2024mixture} \cite{kratsios2024mixture}, \citeauthor*{schwab2025deep} \cite{schwab2025deep} \citeauthor*{godeke2025universal} \cite{godeke2025universal}, \citeauthor*{furuya2025quantitative} \cite{furuya2025quantitative}, and \citeauthor*{adcock2025near} \cite{adcock2025near}); second, show that the solution map is sufficiently continuous, for instance H\"{o}lder-continuous, often via a perturbation analysis, verifying in turn it is in the scope of the main theorem, see \citeauthor*{alvarez2024neural} \cite{alvarez2024neural}, \citeauthor*{horvath2025transformers} \cite{horvath2025transformers}, \citeauthor*{lanthaler2025parametric} \cite{lanthaler2025parametric} or \citeauthor*{firouzi2025neural} \cite{firouzi2025neural}.

\medskip
One may ask if favourable approximation rates are achievable if the reference generator $f$ is simple enough, while still of course having a meaningful structure for several applications in optimal control and mathematical finance. Indeed, in 
\Cref{thm:semilinear}
we show that this is the case when the reference generator is of the {simplified} form
\begin{equation}
\label{eq:semilinear_generator}
f(x,y,z,w) := 
-\mathrm{Tr}\big[\gamma(x) w\big]
-\mathrm{div}(\gamma)(x)\cdot  z
+ \mu(x) \cdot z + \lambda(x) 
y +\tilde f(x,y)
\end{equation}
for some smooth enough maps $\lambda:\R^d\longrightarrow\R$, $\gamma:\R^d\longrightarrow \R^{d\times d}$, and $\mu:\R^d\longrightarrow\mathbb{R}^d$ and where $\tilde{f}:\R^d\times\R\longrightarrow\R$ is still sufficiently smooth. In this setting, we reduce the general fully non-linear elliptic PDE in~\eqref{eq:AssociatePDE_General} to the following semi-linear form
\begin{align}
\label{eq:semilinear}
- \nabla \cdot \gamma \nabla u(x) + \mu(x) \cdot \nabla u(x) + \lambda(x) u(x) +\tilde f(x,u)   = 
\underbrace{-f_0(x)}_{\text{\tiny Perturbation}}
,\;  x\in \mathcal{D},
\;
u(x)  =
\underbrace{g(x)}_{\text{\tiny Perturbation}}
,\;  x\in \partial \mathcal{D}.
\end{align}

\Cref{thm:semilinear} both extends \cite[Theorem 1]{furuya2024simultaneously} by allowing $\mu$ and $\lambda$ to be non-zero and $\Sigma$ to be non-constant and positive-definite, while no longer requiring any \emph{a priori} knowledge of the PDE itself to be hard-coded into our design of the NO. 
This is because the latter authors use explicit knowledge of Green's function associated with the PDE 
$-\nabla \cdot \gamma \nabla u(x) + \mu(x) \cdot \nabla u(x) + \lambda(x) u(x)$ 
to show that it admits a decomposition 
$\Phi(x-y) + \Psi(x,y)$, 
where $\Phi$ is a `difficult to approximate' singular part and $\Psi$ is an `easily approximated' smooth part. 
The convolution with the singular component $\Phi$ is then hard-coded into each of their NO architectures by leveraging the explicit closed form for $\Phi$ obtained in~\cite{cao2022expansion}. 
In contrast, in our approach no such closed-form nor \emph{a priori} knowledge of the PDE is required in our NO build. As should be expected, these extensions also come at the cost of devising an entirely different proof strategy.

\medskip
The PDE in~\eqref{eq:semilinear} can be connected back to the $2$BSDE~\eqref{eq:FBSDE} either when the divergence of $\gamma$ is absorbed into $\mu$, or
in the special case where $\gamma$ is divergence-free, \emph{i.e.} $\mathrm{div}(\gamma)^{\top}=0$, implying that $\nabla \cdot \gamma \nabla u = \mathrm{Tr}[\gamma \nabla u)$. In addition, when $\gamma$ is valued in the set of semi-definite matrices, and we take for $\Sigma$ any matrix square root of $e2\gamma$ (that is to say $\Diffuse \Diffuse^{\top}= 2\gamma$), then \eqref{eq:semilinear} reduces to the more standard Hamilton--Jacobi--Bellman--type semilinear equation
\begin{equation}
\label{eq:semilinear__DiffusionType}
\tilde{f}(x,u) 
+
\lambda(x) u(x) 
+ \mu(x) \cdot \nabla u(x)
- 
\frac{1}{2}\, 
\mathrm{Tr}\big[\Diffuse(x)\Diffuse(x)^{\top}\nabla^2 u(x)\big]
= 
-f_0(x)
,\; x\in \mathcal{D}
.
\end{equation}
In dimension $d\ge 2$, there is a whole zoology of divergence-free $\gamma$; thus this special case completely subsumes the case where $\gamma$ is constant, as considered in \cite{furuya2024simultaneously}.  For example, when $d=2$, if $\gamma$ is positive-definite--valued then there exist a twice continuously differentiable potential $\varphi_{\gamma}:\mathbb{R}^2\longrightarrow \mathbb{R}$ (this is the so-called Airy potential) such that 
$
        \gamma(x) 
    = 
        R^{\top} \big(\nabla^2 \varphi_{\gamma}(x)\big) R
$
for the symplectic matrix $R\coloneqq e_1e_2^{\top}-e_2e_1^{\top}$ (where $(e_1,e_2$ is the canonical basis of $\mathbb{R}^2$).  A simple non-constant example of such an Airy potential $\varphi_{\gamma}$ which additionally yields a positive-definite $\gamma$ is $\varphi_{\gamma}(x,y)\coloneqq (x^2+y^2)^2$.

\medskip
Our first objective is, therefore, to \textit{simultaneously} approximate the solution operator to general families of fully non-linear elliptic problems~\eqref{eq:AssociatePDE_General} and to obtain favourable rates for semi-linear special cases of the form~\eqref{eq:semilinear}.  Our strategy will be to construct a neural operator (NO) model which directly approximates (\Cref{thrm:generalapprox} resp.~\Cref{thm:semilinear}) the coefficient-to-solution operator mapping any $(g,f_0)\in \mathfrak{W}$ to the elliptic PDE it defines via~\eqref{eq:AssociatePDE_General} (resp.~\eqref{eq:semilinear}).  
Then, using the connections between elliptic PDEs and 2BSDEs with random terminal time in~\eqref{eq:PDE_to_2BSDE} formalised by our non-linear Feynman--Kac formula in \Cref{prop:NL_FK}, we construct an adapter transforming the functions output for our NO to tuples of stochastic processes approximating the solution to the family of associated 2BSDEs, see \Cref{thrm:Main_Stochastic}.  

\subsection{Related literature}
\label{s:Prelims__ss:RelatedLiterature}

There is a mature numerical literature on second–order BSDEs (2BSDEs), including weak approximation and time–discretisation schemes by \citeauthor*{possamai2015weak}~\cite{possamai2015weak}, \citeauthor*{ren2015convergence}~\cite{ren2015convergence}, \citeauthor*{yang2019explicit}~\cite{yang2019explicit}, and the recent non-equidistant scheme of \citeauthor*{pak2025nonequidistant}~\cite{pak2025nonequidistant}.
Learning–based approaches have also appeared (\emph{e.g.}, \citeauthor*{beck2019machine}~\cite{beck2019machine}, \citeauthor*{pereira2020feynman}~\cite{pereira2020feynman}, \citeauthor*{duong2023solving}~\cite{duong2023solving}, \citeauthor*{xiao2024numerical}~\cite{xiao2024numerical}),
but these methods are essentially \emph{per–instance}: they must be re–run (or re–trained) whenever coefficients or boundary data change.
By contrast, we learn a \emph{solution operator} that acts on the entire compact family of problems indexed by $(g,f_0)$, so a single trained model simultaneously solves all members of the family, both at the PDE and at the 2BSDE level via the PDE--(2)BSDE correspondence (\citeauthor*{cheridito2007second}~\cite{cheridito2007second}; see also \citeauthor*{pardoux1998backward}~\cite{pardoux1998backward}, \citeauthor*{soner2012wellposedness}~\cite{soner2012wellposedness}).

\medskip
When it comes to finite--dimensional ML for non-linear PDEs, a large body of work trains a finite–dimensional network for each target PDE separately (\emph{e.g.},
\citeauthor*{nusken2021solving}~\cite{nusken2021solving},
\citeauthor*{pham2021neural}~\cite{pham2021neural},
\citeauthor*{germain2022deepsets,germain2022approximation,germain2023neural}~\cite{germain2022deepsets,germain2022approximation,germain2023neural},
\citeauthor*{lefebvre2023differential}~\cite{lefebvre2023differential},
\citeauthor*{zhou2021actor}~\cite{zhou2021actor},
\citeauthor*{hu2024recent}~\cite{hu2024recent},
\citeauthor*{nguwi2024deep}~\cite{nguwi2024deep}).
Provable \emph{exponential} behaviour in this setting typically requires strong structure:
either linear second–order elliptic operators (\citeauthor*{marcati2023exponential}~\cite{marcati2023exponential,marcati2025expression}) or analyticity of the \emph{single} target solution, so that one may invoke classical exponential approximation of analytic functions by neural networks (\citeauthor*{mhaskar1995degree}~\cite{mhaskar1995degree}, \citeauthor*{mhaskar1996neural}~\cite{mhaskar1996neural}, \citeauthor*{e2018exponential}~\cite{e2018exponential}).

\medskip
On the other hand, neural operators (NOs) learn the infinite–dimensional coefficient–to–solution map and hence \emph{simultaneously} solve all PDEs in a parametric class with a single model; see the early universality observation of \citeauthor*{chen1995universal}~\cite{chen1995universal}, the DeepONet/FNO line (\citeauthor*{lu2019deeponet}~\cite{lu2019deeponet}, \citeauthor*{kovachki2023neural}~\cite{kovachki2023neural}), the CNO universality \citeauthor*{raonic2023convolutional}~\cite{raonic2023convolutional}, and a large set of abstract guarantees in Banach/Besov/Sobolev and non–linear metric settings
(\citeauthor*{yu2021arbitrary}~\cite{yu2021arbitrary},
\citeauthor*{lu2021learning}~\cite{lu2021learning},
\citeauthor*{lanthaler2022error}~\cite{lanthaler2022error},
\citeauthor*{adcock2022efficient}~\cite{adcock2022efficient},
\citeauthor*{korolev2022two}~\cite{korolev2022two},
\citeauthor*{cuchiero2023global}~\cite{cuchiero2023global},
\citeauthor*{neufeld2023universal}~\cite{neufeld2023universal},
\citeauthor*{kratsios2024mixture}~\cite{kratsios2024mixture},
\citeauthor*{adcock2025near}~\cite{adcock2025near},
\citeauthor*{godeke2025universal}~\cite{godeke2025universal},
\citeauthor*{lanthaler2025parametric}~\cite{lanthaler2025parametric},
\citeauthor*{schwab2025deep}~\cite{schwab2025deep},
\citeauthor*{de2022deep}~\cite{de2022deep},
\citeauthor*{furuya2025quantitative}~\cite{furuya2025quantitative},
\citeauthor*{kratsios2025deep}~\cite{kratsios2025deep},
\citeauthor*{acciaio2024designing}~\cite{acciaio2024designing},
\citeauthor*{kratsios2023universal}~\cite{kratsios2023universal}).
Within this line, \emph{exponential} (sometimes `exponential–in–depth') expression rates are known for holomorphic operator classes (\citeauthor*{adcock2024optimal}~\cite{adcock2024optimal}), for certain linear elliptic PDEs (including polytopal domains) (\citeauthor*{marcati2023exponential}~\cite{marcati2023exponential,marcati2025expression}), and for specific semilinear elliptic equations on smooth domains (\citeauthor*{furuya2024simultaneously}~\cite{furuya2024simultaneously}).
Other exponential statements rely either on super–expressive activations with effectively infinite pseudo–dimension (\citeauthor*{shen2022deep}~\cite{shen2022deep}, \citeauthor*{pollard1984convergence}~\cite{pollard1984convergence}, \citeauthor*{alvarez2024neural}~\cite{alvarez2024neural}) or on implicit/equilibrium–layer constructions exploiting convex variational structure (\citeauthor*{kratsios2025generative}~\cite{kratsios2025generative}).

\medskip
Our contribution in this landscape is that we design a NO that \emph{simultaneously} $(i)$ approximates the solution operator of a broad class of second–order elliptic PDEs/2BSDEs and $(ii)$ retains \emph{exponential–in–depth} rates in a substantially more general semi-linear regime than in the closest prior work.
Concretely
\begin{enumerate}[label=$(\roman*)$,leftmargin=1.65em]
\item \emph{family–level learning.}
We approximate the coefficient–to–solution map $\Gamma^\smallertext{+}$ on a compact infinite family indexed by $(f_0,g)$, hence a single training phase serves the whole family (PDEs and the associated 2BSDEs).
For fully non–linear elliptic equations we obtain general operator–level approximability (algebraic rates) by combining quantitative NO universality on Besov/Sobolev scales (\citeauthor*{yu2021arbitrary}~\cite{yu2021arbitrary}, \citeauthor*{lu2021learning}~\cite{lu2021learning}, \citeauthor*{lanthaler2022error}~\cite{lanthaler2022error}, \citeauthor*{adcock2022efficient}~\cite{adcock2022efficient}, \citeauthor*{korolev2022two}~\cite{korolev2022two}, \citeauthor*{galimberti2022designing}~\cite{galimberti2022designing}) with stability of the solution map (Krylov–type assumptions; cf.\ \citeauthor*{krylov2018sobolev}~\cite{krylov2018sobolev}).

\item \emph{Exponential rates for semi-linear equations under \emph{general} forward dynamics.}
In the semi-linear case
\[
-\nabla\!\cdot\!\gamma(x)\nabla u+\mu(x)\!\cdot\!\nabla u+\lambda(x)u+\tilde f(x,u)=-f_0(x),\; u|_{\partial\Dc}=g,
\]
with smooth, uniformly elliptic $\gamma$ and smooth $\mu$, $\lambda$, we implement the classical fixed–point map by a non-local NO layer built from (approximated) Green kernels; existence/regularity of Green functions for variable–coefficient operators is standard (\citeauthor*{kim2019green}~\cite{kim2019green}).
This yields accuracy $\varepsilon$ with \emph{logarithmic} depth $L=O(\log(1/\varepsilon))$, \emph{constant} width, and a finite non-local rank that scales polynomially in $1/\varepsilon$.
Compared to \citeauthor*{furuya2024simultaneously}~\cite{furuya2024simultaneously}, which hard–codes the singular part of the Green’s kernel and effectively assumes a driftless, constant–diffusion forward (so that the singular $\Phi$ is known in closed form), our construction does \emph{not} require a closed–form kernel split and therefore covers far more general, state–dependent Itô diffusions in the forward process and variable–coefficient elliptic operators, while preserving exponential depth–rates.

\item \emph{From {\rm PDE} to $(2)${\rm BSDE} at the operator level.}
Because each 2BSDE in the family admits the PDE representation, our NO for the elliptic map transfers directly to a NO for the $(Y,Z,\Upsilon,A)$–processes associated with the \emph{entire} 2BSDE family.
\end{enumerate}

\medskip
Building upon the finite-dimensional lower bounds of~\citeauthor*{yarotsky2017error} \cite{yarotsky2017error} 
, it was recently shown in \citeauthor*{lanthaler2025parametric} \cite{lanthaler2025parametric} that arbitrary continuous---or even several times continuously Fr\'{e}chet differentiable---non-linear operators between Sobolev spaces cannot be uniformly approximated on compact sets by NOs without requiring an exponential number of trainable parameters in the reciprocal approximation error.
Consequently, without additional structure beyond simple smoothness, there are insurmountable obstructions to operator learning.  Thus, even if one could establish H\"{o}lder-continuity of the coefficient-to-solution operator in the fully non-linear setting (\emph{e.g.} using results of \citeauthor*{taylor2023partial} \cite{taylor2023partial}, which we do show in \Cref{s:StabilityEstimate}) the solution operator would still not be regular enough to permit meaningful approximation rates. In such cases, any quantitative result is practically no more meaningful than an existential statement on the approximability of the coefficient-to-solution operator (see \Cref{thrm:generalapprox}), akin to the qualitative (rate-free) universal abstract approximation results of \citeauthor*{chen1993approximations} \cite{chen1993approximations}, \citeauthor*{benth2023neural} \cite{benth2023neural}, or \citeauthor*{bilokopytov2025universal} \cite{bilokopytov2025universal} for other NO architectures.

\medskip
When it comes to the closest exponential–rate results available in the literature, relative to linear/holomorphic NO rates (\citeauthor*{marcati2023exponential}~\cite{marcati2023exponential,marcati2025expression}, \citeauthor*{adcock2024optimal}~\cite{adcock2024optimal}), we require neither analyticity nor specialised domains; and unlike exponential claims relying on super–expressive activations or implicit/equilibrium layers (\citeauthor*{shen2022deep}~\cite{shen2022deep}, \citeauthor*{pollard1984convergence}~\cite{pollard1984convergence}, \citeauthor*{alvarez2024neural}~\cite{alvarez2024neural}, \citeauthor*{kratsios2025generative}~\cite{kratsios2025generative}), our architecture maintains finite capacity per layer with explicit depth/width/rank scaling.
Crucially, compared to \citeauthor*{furuya2024simultaneously}~\cite{furuya2024simultaneously}, our exponential regime permits markedly \emph{more general} forward dynamics and variable–coefficient elliptic operators, because the Green–kernel is learned/approximated rather than injected in closed form.

\section{Preliminaries}
\label{s:Prelim}

\subsection{Notation}
\label{s:Prelim__ss:Notation}
Let $p \in (1, \infty)$. 
We denote by $p^\prime \in (1, \infty)$ the conjugate of $p$ such that $1/p + 1/p^\prime = 1$.  We let $\N$ be set of non-negative integers, $\N^\star$ the set of positive integers, and $\Z$ the set of all negative and non-negative integers. We henceforth fix an ambient dimension\footnote{
In~\cite{furuya2024simultaneously} an explicit expression for the singular part of the Green's function associated to the stopped forward process's induced elliptic PDE was required, which additionally constrained $d\ge 3$ there, but not herein.} %
$d\in \mathbb{N}^\star$; and let $\S_d^\smallertext{+}$ denote the set of $d\times d$ (real) positive-definite matrices.  Recall that, every symmetric positive definite matrix $A\in \S_d^+$ has a unique well-defined square-root given by $\sqrt{A} \coloneqq \log(\exp(A)/2)$ where $\exp$ is the matrix exponential and $\log$ is its (unique) inverse on $\S_d^\smallertext{+}$, see \emph{e.g}. \citeauthor*{arabpour2024low} \cite[Lemma C.5]{arabpour2024low}.  For any $d\in \mathbb{N}^\star$ denote the Fr\"{o}benius norm of any $d\times d$ matrix $A$ by $\|A\|_\smallertext{F}$.  Given any metric space $(\mathcal{X},\rho)$, any $x\in \mathcal{X}$, and any radius $r\ge 0$, we define the open ball $B_{(\mathcal{X},\rho)}(x,r)\coloneqq \{u\in \mathcal{X}: \rho(x,u)<r\}$.  Given any two vector spaces $V$ and $W$, and any $x\in V$ and $y\in W$, we write $x\oplus y\coloneqq (x,y)=V\times W$.  

\medskip
For any $p\geq 1$, we let $\ell^p(\Z)$ be the set of real-valued sequences $(u_n)_{n\in\Z}$ indexed by $\Z$ such that
\[
\sum_{n\in\Z}|u_n|^p<\infty.
\]
We also let $\L^p(\R)$ be the set of $p$-integrable Lebesgue-measurable functions on $\R$.

\medskip
For any $I\in \mathbb{N}$, we use $C^I(\mathbb{R})$ to denote the vector space of real-valued at-least $I$-times continuously differentiable functions on $\mathbb{R}$, and $C_c^I(\mathbb{R})$ for the subset thereof consisting of those compactly supported functions therein.  
For any $(s,d,D)\in (\mathbb{N}^\star)^3$, we write $C^s(\mathcal{D},\mathbb{R}^D)$ (resp. $C^{\infty}(\mathcal{D},\mathbb{R}^D)$) for set of functions from $\mathbb{R}^d$ to $\mathbb{R}^D$ which are at-least $s$-times (resp.\ smooth) continuously differentiable when restricted to $\mathcal{D}$.  We refer the reader to Appendix~\ref{s:Besovbaby} for wavelet-centric definitions of Besov, and thus Sobolev, spaces.  

\medskip
Throughout this paper, $(\Omega,\mathcal{F},\mathbb{F}\coloneqq  (\mathcal{F}_t)_{t\ge 0},\mathbb{P})$ will denote a filtered probability space satisfying the usual conditions.  
For any $T>0$ we use $\mathcal{H}_T^2$ to denote the class of square-integrable predictable processes $X:[0,T]\times \Omega\longrightarrow\mathbb{R}$.

\subsection{Deep learning}
\label{s:Prelim__ss:DL}

Neural operators (NOs) extend deep learning from finite-dimensional vector spaces to infinite-dimensional Banach spaces, with standard NOs specialising in function-to-function mappings. Broadly speaking, there are three types of NO builds between function spaces: the Fourier neural operator--type builds (FNO), which iteratively use finitely parametrised integral-kernel affine transformations between their non-linearities; DeepONet-type architectures (see \citeauthor*{lu2019deeponet} \cite{lu2019deeponet}) which learn to adaptively regress against learnable bases; and encoder–-processor–-decoder-type models, such as PCA--Net (see \citeauthor*{chan2015pcanet} \cite{chan2015pcanet}) which project infinite-dimensional data using a Schauder basis before processing it via a standard finite-dimensional neural network, and then reassembles finite-dimensional basis functions using the network’s outputs as coefficients. 

\medskip
The first and last of these models tend to be more numerically stable, the middle construction can exhibit advantageous approximation rates, and the third model is more readily generalisable to non–-function space settings (see \emph{e.g.} \citeauthor*{galimberti2022designing} \cite{galimberti2022designing})
by directly lifting the approximation guarantees for classical neural networks (see e.g.~\citeauthor*{yarotsky2017error,bolcskei2019optimal,devore2021neural,gribonval2022approximation,kratsios2022do,shen2022deep,hong2024bridging,schneider2025nonlocal}) to infinite dimensions. Our neural-operator build combines the best of the first two models using a two-branch structure: the top branch of an FNO-type, the bottom branch inspired by DeepONets, with coefficients shared between layers. Moreover, we map into non–-function space targets when applying our deep-learning model in the 2BSDE setting by transforming its function space--valued outputs into processes via a `Feynman–-Kac adapter', that is to say a custom non-trainable readout layer encoding our nonlinear Feynman-–Kac representation (\Cref{prop:NL_FK}). Finally, we allow the non-linearities injecting structure at each layer of our NO to be adaptive rather than fixed, as in classical NO builds, thereby maximizing their flexibility, for instance granting them the ability to exactly perform multiplication, a property not shared by classical piecewise-linear ReLU activation functions.

\subsubsection{Residual Kolmogorov--Arnold networks (Res--KANs)}
\label{s:Prelim__ss:DL___sss:KANs}
The key idea behind Kolmogorov--Arnold networks (KANs) is to make the activation function itself trainable. In KANs, one typically focuses on the spline part of the following definition \citeauthor*{liu2025kan} \cite{liu2025kan}, with the role of the remaining part of the activation function being an afterthought, normally taken to some standard non-linearity such as the Swish or Sigmoid functions.  In this paper, we explicitly exploit both parts of KANs activation functions, and as such, we add some basic structural requirements to the `non-spline' part of the activation function (below in~\eqref{eq:KAN_r}) which serves a pointed role in our approximation theory in connection with the multi-resolution analysis (MRA); see \emph{e.g}. \citeauthor*{mallat1989multiresolution} \cite{mallat1989multiresolution}.

\medskip
Specifically, the activation $\sigma_{\beta:\smallertext{I}}:\R\longrightarrow \R$ maps any $x \in \mathbb{R}$ to a mixture of spline basis functions of varying degrees
\begin{equation}
\label{eq:KAN_r}
    \sigma_{\beta:\smallertext{I}}(x)
\coloneqq
        \underbrace{
            \beta_{\smallertext{-}1}
            \sigma_\smallertext{S}(x)
            +
            \beta_0
            \sigma_\smallertext{W}(x)
        }_{\text{\tiny Spectral structure}}
    +
        \underbrace{
            \sum_{i=1}^{\smallertext{I}} 
                \beta_i
                \mathcal{N}_i(x)
        }_{\text{\tiny Local structure}}
\end{equation}
where $I\in\N$, $\beta=(\beta_{\smallertext{-}1},\beta_0,\cdot,\beta_\smallertext{I})^\top \in \mathbb{R}^{\smallertext{I}+2}$ is a trainable vector of coefficients, and where for $i\in\{1,\dots,I\}$, $\mathcal{N}_i:\mathbb{R}\longrightarrow \mathbb{R}$ are the cardinal B-splines which, following \citeauthor*{mhaskar1992approximation} \cite[Equation (4.28)]{mhaskar1992approximation}, can be defined by $\mathcal{N}_0(x)\coloneqq \mathbf{1}_{[0,1)}$ and for any $i\in \mathbb{N}^\star$
\begin{equation}
	\label{lem:MichelliRepresentation}
	   \mathcal{N}_I(x)
	\coloneqq
    	\sum_{j=0}^{\smallertext{I}\smallertext{+}1}
    	\frac{(-1)^j \binom{I+1}{j} }{I!}
    	\mathrm{ReLU}(x-j)^I,\;x\in \mathbb{R}.
\end{equation}
Furthermore, $\sigma_\smallertext{S}:\mathbb{R}\longrightarrow \mathbb{R}$ as well as $\sigma_\smallertext{W}:\mathbb{R}\longrightarrow \mathbb{R}$ and satisfy the spectral properties in \Cref{ass:wavlets} below.  However, before turning to the properties, we elucidate the first few wavelets in \Cref{fig:SplineIllustrations}.
\begin{figure}[H]
	\centering
        \includegraphics[width=.5\linewidth]{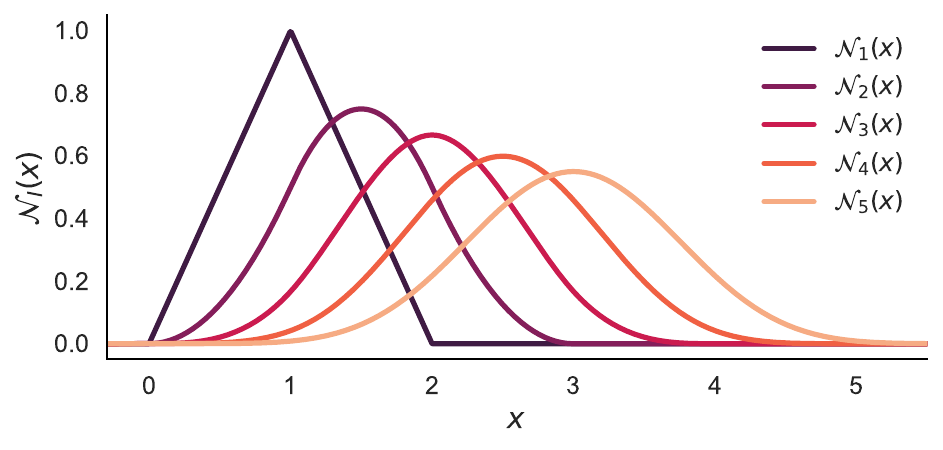}
	\caption{The cardinal $B$-splines of orders $I=0,1$, and $2$.}
	\label{fig:SplineIllustrations}
\end{figure}
\begin{assumption}[Daubechies properties of order $I$]
\label{ass:wavlets}
Fix $I\in \mathbb{N}$.  The {respective} $`$scale' {and $`$wavelet'} activation function $\sigma_\smallertext{S}$
and ${\sigma_\smallertext{W}}$ both belong to $C^I_c(\mathbb{R})$ if $I>0$ {$($resp. $L^2(\mathbb{R})$ when $I=0$ with compact essential support$)$} and satisfy the refinement equation of {\rm \citeauthor*{daubechies1988orthonormal} \cite[Equation (3.47)]{daubechies1988orthonormal}}, that is to say that there is a sequence of \emph{low-pass filters} $(h_k)_{k\in\Z}\in \ell^2(\mathbb{Z})$ summing to $\sqrt{2}$, satisfying the \emph{orthogonality condition}\footnote{See~\cite[Equation (3.18)]{daubechies1988orthonormal}} 
\[
\sum_{k\in \mathbb{Z}}h_{k\smallertext{-}2i}h_{k\smallertext{-}2j}=\mathbf{1}_{\{i=j\}},\; \forall (i,j)\in \mathbb{Z}^2,\]
and such that
\begin{align}
\label{eq:Daubechies_Refinement}
        \sigma_\smallertext{S}(x) 
    &=
        \sqrt{2}
        \sum_{k\in \mathbb{Z}}
            h_k \sigma_\smallertext{S}(2x - k),\; x\in\R,\\
\label{eq:Father_to_Mother}
        \sigma_\smallertext{W}(x) 
    &=
        \sqrt{2}
        \sum_{k\in \mathbb{Z}}
            (-1)^kh_{1\smallertext{-}k} \sigma_\smallertext{S}(2x-k),\; x\in\R
.
\end{align}
\end{assumption}
The existence of such activation functions (called Daubechies father and mother wavelets respectively), for arbitrary $I$, is guaranteed by~\citeauthor*{triebel2006theory} \cite[Theorem 1.61.$(ii)$]{triebel2006theory}, while algorithmic constructions can be found in~\citeauthor*{daubechies1992ten} \cite[Chapter 6.4]{daubechies1992ten}, and are standard in modern signal processing.  Nevertheless, in the very low regularity regime where $I=0$, the Haar system and the indicator function is a transparent example where \Cref{ass:wavlets} holds.
\begin{example}[Haar wavelets and indicator function for discontinuous regularity]
\label{ex:Haar}
If $I=0$ then, the indicator function of the unit interval $\sigma_\smallertext{S}\coloneqq \mathbf{1}_{[0,1)}$ and the Haar wavelet $\sigma_M\coloneqq \mathbf{1}_{[0,1/2)}-\mathbf{1}_{[1/2,1)}$ satisfy {\rm\Cref{ass:wavlets}} with $h_0=h_1=\frac{1}{\sqrt{2}}$ and $h_k=0$ whenever $|k|\ge 2$.  {Thus, $\sigma_M$ and $\sigma_S$ belong to $L^2(\mathbb{R})$ as expected since $I=0$.}
\end{example}

In a KAN, this activation operates component-wise, with parameters tailored to each neuron. That is, for {any integer $k$, any} $x \in \mathbb{R}^k$, and $\mathbf{\beta} \coloneqq  (\beta_1,\dots,\beta_k) \in \mathbb{R}^{(I+2)\times k}$, we define
\begin{align}
\label{eq:componentwise_action}
\nonumber\sigma_{\mathbf{\beta}:\smallertext{I}}\bullet:\R^k&\longrightarrow \R^k\\
        x=(x_1,\dots,x_k)^\top&\longmapsto 
        \big(
            \sigma_{\beta_\smalltext{1}:\smallertext{I}}(x_1),\dots,\sigma_{\beta_\smalltext{k}:\smallertext{I}}(x_k)
        \big)
        ^{\top}
.
\end{align}

We now introduce the core idea of \emph{residual} KAN networks. These networks incorporate an additional \emph{residual connection}, ensuring that signal is preserved during activation. Residual connections, standard in modern deep learning architectures, help stabilise training by preserving gradient flow and regularising the loss landscape, see \citeauthor*{riedi2023singular} \cite{riedi2023singular}. They also mitigate vanishing gradients that can be caused by normalisation layers. Following \citeauthor*{acciaio2024designing} \cite{acciaio2024designing}, we allow for flexible use of these residual paths, potentially modulated by a trainable gating mechanism. 

\medskip
More precisely, we fix positive integers $d_{\mathrm{out}}$ and $d_{\mathrm{in}}$, matrices $(A,G)\in\R^{d_{\smalltext{\rm out}}\times d_{\smalltext{\rm in}}}\times \R^{d_{\smalltext{\rm out}}\times d_{\smalltext{\rm in}}}$, with $G$ being diagonal (\emph{i.e.} $G_{i,j}=0$ for $(i,j)\in\{1,\dots,d_{\rm out}\}\times\{1,\dots,d_{\rm in}\}$ with $i\neq j$), as well as $b\in \mathbb{R}^{d_\smalltext{\mathrm{out}}}$, and $\beta \in \R^{({I}\smallertext{+}2) \times d_{\smalltext{\mathrm{out}}}}$, a matrix of trainable coefficients. We then define for $x\in \R^{d_{\smalltext{\rm in}}}$
\begin{equation}
\label{eq:norm_res_KAN_layers}
    \mathcal{L}(x|A,b,\beta,G:I)
\coloneqq  
    \underbrace{
        \sigma_{\mathbf{\beta}:\smallertext{I}}\bullet (Ax+b)
    }_{\text{\tiny KAN layer}}
    +   
    \underbrace{
            G x
    }_{\text{\tiny Residual connection}}
\end{equation}

Although compositions of such KAN layers define valid functions, these may lack higher-order smoothness—an issue for applications such as PDE solving that require high regularity. There are two ways to address this: (1) enforce that $\beta_i = 0$ for small $i$, or (2) apply a smoothing layer at the output. We adopt the first strategy to ensure that the functions realised by our \emph{smoothed residual {\rm KANs}} are infinitely differentiable.
\begin{definition}[Residual KANs (Res--KANs)]
\label{defn:SmoothedKAns}
Let $D$ and $I$ be positive integers, and let $\alpha>0$. A residual Kolmogorov--Arnold network $(${\rm Res--KAN}$)$ is a function $\widehat{f}:\mathbb{R}^d\longrightarrow \mathbb{R}^D$ with 
representation, for some $L\in\N^\star$
\begin{equation}
\label{eq:DEF_SRKs}
    \widehat{f}  = 
    A^{(L)} f^{(L)}+b^{(L)},
\end{equation}
with
\[
    f^{(0)}(x) = x,\; x\in\R^d,\; f^{(\ell)} = \mathcal{L}\big(f^{(\ell-1)}|A^{(\ell)},b^{(\ell)},\beta^{(\ell)},G^{(\ell)}:I\big),\; \ell\in\{1,\dots,L\},
\]
where, for $\ell\in\{1,\dots,L\}$, $A^{(\ell)}$ and $G^{(\ell)}$ are $d_{\ell+1}\times d_\ell$ matrices with $G^{(\ell)}$ \textit{diagonal}, $\beta^{(\ell)}$ is a $(I+2)\times d_{{\ell}{+}{1}}$ matrix, $b\in \mathbb{R}^{d_{\smalltext{\ell}\smalltext{+}\smalltext{1}}}$, for given positive integers $(d_0,\dots,d_{L+1})$ satisfying $d_0=d$ and $d_{L+1}=D$. In addition, for any $\ell\in\{1,\dots,L\}$, $\beta^{(\ell)}$ satisfies the \emph{sparsity} pattern ensuring smoothness%
\footnote{%
\add{The $\lceil \alpha \rceil$-time continuous differentiability of $\hat{f}$ follows from that of B-splines (see \citeauthor*{devore1993besov} \cite{devore1993besov}), and the chain rule.}%
}
\begin{equation}
\label{eq:DEF_SRKs__sparsity}
       \beta^{(\ell)}_{i,j} = 0, \;
        i < \lceil \alpha \rceil
        {\; \mbox{\rm and},\; j\in\{1,\dots,d_{\ell+1}\}}
.
\end{equation}
We denote the class of all {\rm Res--KANs} with $L$ hidden layers, width $W\coloneqq  \max_{\ell\in\{1,\dots,L+1\}}d_\ell$, adaptivity parameter $I$, and smoothness parameter $\alpha$, by $\operatorname{Res--KAN}_{\smallertext{L},\smallertext{W}}^{\smallertext{I},\alpha}(\mathbb{R}^d,\mathbb{R}^D)$.
\end{definition}

\subsubsection{Neural operator architectures}
\label{s:Prelim__ss:DL___sss:NeuralOperators}
We recall that we have fixed a constant $1<p<\infty$ and $\mathcal{D} \subset \mathbb{R}^d$, a bounded open domain. The classical neural operators are defined in, \emph{e.g.}, \citeauthor*{kovachki2023neural} \cite{kovachki2023neural} or \citeauthor*{lanthaler2025nonlocality} \cite{lanthaler2025nonlocality}.  

\medskip
Importantly, our NO architecture (see \Cref{fig:NOBuild}) contains both encoder--processor--decoder (EPD) type and Fourier neural operator (FNO)-type `branches' at each layer, whereby spectral features and physical features are iteratively processed in parallel, and then combined together using the adaptively activated neurons spearheaded by the KAN paradigm~\cite{kratsios2025kolmogorov}, rather than the static activation strategy of classical MLPs.  The resulting architecture thus exhibits beneficial properties both of FNO-type models and EPD-type models.

\begin{figure}[htp!]
    \centering
    \includegraphics[width=0.95\linewidth]{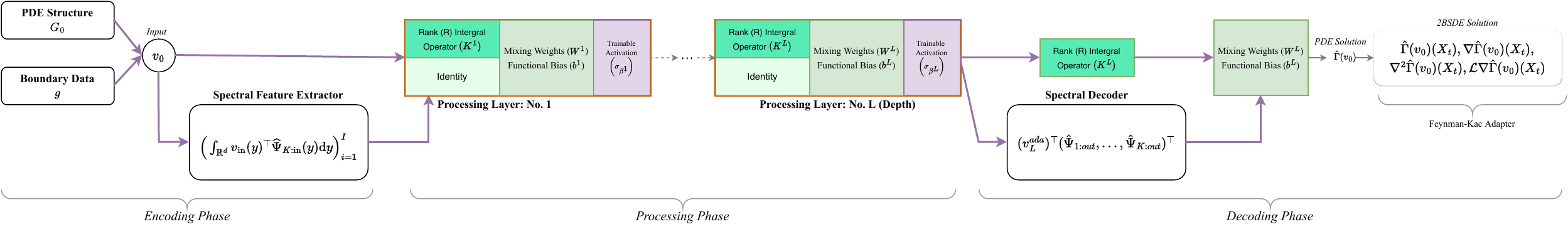}
    \caption{\footnotesize{\footnotesize The KANO (\Cref{def:neural-operator}) pipeline}.}
    \label{fig:NOBuild}
\end{figure}
What is illustrated in \Cref{fig:NOBuild} is as follows. 

\begin{itemize}
\item[$0)$] First boundary data $(g)$ and the PDE structure $(G_0)$ are concatenated into an input $v_0$.  

\medskip
\item[$1)$] Learnable spectral features akin to FNOs are then extracted from $v_0$ and concatenated thereto. 

\medskip
\item[$2)$] At each processing iteration, the top NO branch applies a finite rank (R) integral operator, then all features are mixed and adaptively activated.  

    \medskip
\item[$3)$] Finally the predictions are decoded via two branches: one applying another finite rank integral operator together as with to the FNO and the other leveraging a (trainable) spectral feature decoding akin to EPD, before both branches are mixed together to obtain the final prediction $\hat{\Gamma}$.
\end{itemize}
\medskip
In the 2FBNO variant (\Cref{defn:FBNO}): $\hat{\Gamma}(v_0)$ is passed through the \textit{Feynman--Kac adapter} (see \Cref{prop:NL_FK}).

\medskip
This being said, we can now proceed with the definition.  In the remainder of the paper $d_{\mathrm{in}}=2$, any tuple $v_{\mathrm{out}}\in W^{1,\infty}(\mathcal{D};\mathbb{R})^{d_\smalltext{\mathrm{in}}}$ will correspond to a pair of boundary and source data $(g,f_0)$, and $d_{\mathrm{out}}=1$.  However, since many of these result can be use in more general approximation theory of solutions operators to other PDEs, we keep the definition of our KANO model general enough to accomodate other applications.

\begin{definition}[Kolmogorov--Arnold neural operator (KANO)]\label{def:neural-operator}
Fix positive integers $d_{\mathrm{in}}$, $d_{\mathrm{out}}$, {$L$, $W$, $\widehat{L}$, $\widehat{W}$, $D_{\mathrm{ada}}$, $W_{\mathrm{ada}}$}, as well as smoothness parameters $\alpha>0$ and $I\in\mathbb{N}^\star$ with $3\le \alpha\le I$.
We define a \emph{Kolmogorov--Arnold neural operator (KANO)} {$\widehat{\Gamma} : W^{1,\infty}(\mathcal{D};\mathbb{R})^{d_\smalltext{\mathrm{in}}} \longrightarrow W^{1,\infty}(\mathcal{D};\mathbb{R})^{d_\smalltext{\mathrm{out}}}$} to be any map sending any $v_{\mathrm{in}} \in W^{1,\infty}(\mathcal{D},\mathbb{R})^{d_\smalltext{\mathrm{in}}}$ {to some $v_{L+1}\in W^{1,\infty}(\mathcal{D};\mathbb{R})^{d_\smalltext{\mathrm{out}}}$ where $v_{L+1}$ is defined iteratively} by



\begin{align}
\label{eq:NO_Neurons}
        v_{0}(x)
    & 
    \coloneqq  
        \begin{pmatrix}
                v_{0}^{\mathrm{crs}}(x)
        \\[0.5em]
                v_0^{\mathrm{ada}}(x)
        \end{pmatrix}
    \coloneqq  
    \begin{pmatrix}
        v_{\mathrm{in}}(x)
    \\
        \displaystyle
            \int_{\mathbb{R}^\smalltext{d}}
                v_{\mathrm{in}}(y)^{\top}
                \widehat{\Psi}_{1:\mathrm{in}}(y)
            \mathrm{d}y\\
            \vdots\\
         \displaystyle    
                \int_{\mathbb{R}^\smalltext{d}}v_{\mathrm{in}}(y)^{\top}\widehat{\Psi}_{\smallertext{K}:\mathrm{in}}(y)\mathrm{d}y
    \end{pmatrix},\; x\in\Dc,
\\[0.8em]
v_{\ell\smallertext{+}1}(x) 
&
    \coloneqq  
        \begin{pmatrix}
                v_{\ell+1}^{\mathrm{crs}}(x)
        \\[0.5em]
                v_{\ell+1}^{\mathrm{ada}}(x)
        \end{pmatrix}
\coloneqq 
    \sigma_{\mathbf{\beta_{\ell}}:I}\bullet
    \Bigg(
        W^{\ell}
        \begin{pmatrix}
        v_{\ell}^{\mathrm{crs}}(x) + \big(K^{(\ell)} v_{\ell}\big)(x) 
        \\[0.5em]
        v_{\ell}^{\mathrm{ada}}(x)
        \end{pmatrix}
        +b^{\ell}(x)
    \Bigg),\; \ell\in\{0,\dots,L-1\},\; x\in\Dc,
\\[0.8em]
    v_{\smallertext{L}\smallertext{+}1}(x) 
&\coloneqq  
    W^{(\smallertext{L})}
    \begin{pmatrix}
    v_{\smallertext{L}}^{\mathrm{crs}}(x)+\big(K^{({L})}v_{\smallertext{L}}\big)(x)
    \\[0.5em]
    \big(v_\smallertext{L}^{\mathrm{ada}}\big)^{\top}(x)
    \begin{pmatrix}
            \widehat{\Psi}_{1:\mathrm{out}}(x)\\
            \vdots\\
            \widehat{\Psi}_{\smallertext{K}:\mathrm{out}}(x)
    \end{pmatrix}
    \end{pmatrix}
 +  
    b^{({L})}(x)
    ,\; x\in\Dc,
\end{align}
where $ \sigma_{\mathbf{\beta}:I}$ is as in~{\rm\Cref{eq:componentwise_action} }and acts as in~\eqref{eq:KAN_r}. In particular, $\beta_{\ell}\in \mathbb{R}^{(I+2)\times d_{\ell+1}}$, each $\big(\widehat{\Psi}_{k:\mathrm{in}}\big)_{k\in\{1,\dots,K\}}$ and $\big(\widehat{\Psi}_{k:\mathrm{out}}\big)_{k\in\{1,\dots,K\}}$ are {\rm Res--KANs} of depth $D_{\mathrm{ada}}$ and width $W_{\mathrm{ada}}$, and for any $\ell\in\{0,\dots,L+1\}$, we have $W^{(\ell)} \in \mathbb{R}^{d_{\smalltext{\ell}\smalltext{+}\smalltext{1}}\times d_{\smalltext{\ell}}}$
\begin{align*} 
& 
\big(K^{(\ell)} v\big)(x) \coloneqq  
\int_{\mathcal{D}} k_{\smallertext{N}\smallertext{N}}^{(\ell)}(x,y) v(y) \mathrm{d}y, \; x \in \mathcal{D}, \; v \in L^{p}(\mathcal{D};\mathbb{R})^{d_{\smalltext{\ell}}},\; b^{(\ell)}(x) \coloneqq b_{\smallertext{N}\smallertext{N}}^{(\ell)}(x),\; x\in\Dc,
\end{align*}
where $k_{\smallertext{N}\smallertext{N}}^{(\ell)} 
 \in \text{\rm Res--KAN}_{\hat{\smallertext{L}},\hat{\smallertext{W}}}^{\smallertext{I},\alpha}(\mathbb{R}^{d\times d},\mathbb{R}^{d_{\smalltext{\ell}\smalltext{+}\smalltext{1}}\times d_{\smalltext{\ell}}}) 
$ and $b_{\smallertext{N}\smallertext{N}}^{(\ell)} 
 \in \text{\rm Res--KAN}_{\hat{\smallertext{L}},\hat{\smallertext{W}}}^{\smallertext{I},\alpha}(\mathbb{R}^{d},\mathbb{R}^{d_{\smalltext{\ell}}})$ { are  {\rm Res--KANs} of depth $\widehat{L}$ and width $\widehat{W}$.} 
We denote the above class of {\rm KANOs} by
\[
\mathcal{NO}^{\smallertext{L}, \smallertext{W}, \smallertext{I},\alpha}_{\hat{\smallertext{L}}, \hat{\smallertext{W}}}\big(W^{1,\infty}(\mathcal{D};\mathbb{R})^{d_{\smalltext{\rm in}}}, W^{1,\infty}(\mathcal{D};\mathbb{R})^{d_\smalltext{{\rm out}}}\big),
\]
which we abbreviate to $\mathcal{NO}^{\smallertext{L}, \smallertext{W},\smallertext{I},\alpha}_{\hat{\smallertext{L}}, \hat{\smallertext{W}}}$ when the dimensions and domains are contextually evident.
\end{definition}
For $I\coloneqq  \lceil s\rceil$, we henceforth abbreviate 
\begin{equation}
\label{eq:NO-union}
\mathcal{NO}_{\smallertext{I},\alpha}\coloneqq  \bigcup_{(\smallertext{L},\hat{\smallertext{L}},\smallertext{W},\hat{\smallertext{W}},\alpha)\in (\mathbb{N}^\smalltext{\star})^\smalltext{4}\times(0,1)} \mathcal{NO}^{\smallertext{L}, \smallertext{W},\smallertext{I},\alpha}_{\hat{\smallertext{L}},\hat{\smallertext{W}}},
\end{equation}
Motivated by the PDE representation of the solutions to each member of our family of second-order BSDEs, given in~\eqref{eq:PDE_to_2BSDE}, and due to~\cite{cheridito2007second}, we extend the (semi-)classical class of neural operators above to the following stochastic model as follows.
\begin{definition}[$2$Generative neural operators (2FBNO)]
\label{defn:FBNO}
Fix dimensions $d$, and $d_{\rm in}$, as well as smoothness parameters $3\le \alpha\le I$, with $I\in \mathbb{N}^\star$, and fix depths $L\in\N^\star$, $\widehat{L}\in \mathbb{N}^\star$, and widths $W\in\N^\star$, $\widehat{W}\in \mathbb{N}^\star$.  The class of forward--backward {\rm KANOs} $(2${\rm FBNOs}$)$ $\mathcal{FB}^{\smallertext{L}, \smallertext{W},\smallertext{I},\alpha}_{\hat{\smallertext{L}}, \hat{\smallertext{W}},X}$ consists of all 
\[
\begin{aligned}
    \widehat{\Gamma}:W^{1,\infty}(\mathcal{D},\mathbb{R})^{d_\smalltext{\rm in}} &\longrightarrow  (\mathcal{H}_T^2)^4 \coloneqq \prod_{i=1}^4\, \mathcal{H}_T^2
\\
    f\coloneqq  (f_1,\dots,f_{d_{\rm in}}) & \longmapsto (\widehat{Y}^{f},\widehat{Z}^{f},\widehat{\Upsilon}^{f},\widehat{A}^{f}),
\end{aligned}
\]
for which there is a $\Gamma\in\mathcal{NO}^{\smallertext{L}, \smallertext{W}, \smallertext{I},\alpha}_{\hat{\smallertext{L}}, \hat{\smallertext{W}}}(W^{1,\infty}(\mathcal{D};\mathbb{R})^{d_\smalltext{\rm in}}, W^{1,\infty}(\mathcal{D};\mathbb{R}))$ satisfying the representation
\[
\begin{aligned}
    Y_t^{f}
    = 
    \Gamma(f)(X_t)
    ,\;
    Z_t^{f} 
    = 
    \big(\nabla \Gamma(f)\big) (X_t)
    ,\;
    \Upsilon_t^{f} 
    = 
    \big(\nabla^2 \Gamma(f)\big) (X_t)
    ,\; \mbox{\rm and}\;
    A_t^{f} 
    =
    \big(\mathcal{L}\nabla \Gamma(f)\big) (X_t),
\end{aligned}
\]
where, as before, $\mathcal{L}$ denotes the generator of $X$, without the drift.
\end{definition}

\section{Main results}
\label{s:Main_Results}

\subsection{Elliptic PDE representation of the 2BSDE system}
\label{s:MainResults__ss:PDEForm}
For the reader's convenience, we repeat the PDE in~\eqref{eq:AssociatePDE_General}.
\begin{equation}
\label{eq:AssociatePDE_General__copy}
 f\big(x,u(x), \nabla u(x), \nabla^2 u(x)\big) = -f_0(x),\; x\in \mathcal{D},\; u(x) = g(x),\; x\in \partial \mathcal{D},
\end{equation}

\begin{proposition}[Non-linear Feynman--Kac's formula]
\label{prop:NL_FK}
Let $u$ be a classical solution to the {\rm PDE} \eqref{eq:AssociatePDE_General}, such that all the quantities below are defined and continuous in time 
\[
 Y_t 
     = 
    u(X_t),\; Z_t  = 
    \nabla u(X_t),\; \Upsilon_t  = 
    \nabla ^2u(X_t),\; A_t =
    \mathcal{L}\nabla u(X_t),\; t\in[0,\tau),\; \P\text{\rm--a.s.},
\]
where 
\[
  X_t 
    = 
        x+\int_0^t \beta(X_s)\mathrm{d}s + \int_0^t\gamma(X_s)\mathrm{d}W_s,\; t\geq 0,\; \P\text{\rm--a.s.},\; \tau\coloneqq\inf\big\{t\geq 0: X_t\notin \Dc\big\}.
\]
Then $(Y,Z,\Upsilon,A)$ is a solution to \eqref{eq:FBSDE}--\eqref{eq:FBSDE_Martingale}.
\end{proposition}

\begin{proof}
Since $u$ is smooth enough, we can apply It\^o's formula to obtain for any $t\in[0,\tau)$
\begin{align*}
u(X_t)&=u(X_\tau)-\int_t^\tau\frac12\mathrm{Tr}\big[\gamma(X_s)\gamma^\top(X_s)\nabla^2 u(X_s)\big]\mathrm{d}s-\int_t^\tau\nabla u(X_s)\cdot\mathrm{d}X_s,
\end{align*}
as well as
\begin{align*}
\nabla u(X_t)=\nabla u(x)+\int_0^t\nabla^2u(X_s)\mathrm{d}X_s+ \int_0^t\Lc\nabla u(X_s)\mathrm{d}s=\nabla u(x)+\int_0^t\Upsilon_s\mathrm{d}X_s+\int_0^tA_s\mathrm{d}s.
\end{align*}
it follows by the PDE satisfied by $u$ that
\[
u(X_t)=g(X_\tau)+\int_t^\tau\bigg(f(X_s,Y_s,Z_s,\Upsilon_s)+f_0(X_s)-\frac12\mathrm{Tr}\big[\gamma(X_s)\gamma^\top(X_s)\Upsilon_s\big]\bigg)\mathrm{d}s-\int_t^\tau Z_s^\top\mathrm{d}X_s,
\]
as desired.
\end{proof}

\subsection{General approximability guarantee}
\label{s:MainResults__ss:General}
Let $0<\delta \le 1$ and let $\S_d^{\delta}$ denote the subset of $\S^\smalltext{+}_d$ consisting of matrices satisfying the following near--norm preserving property: for every $x\in \mathbb{R}^d$ 
\[
		\delta \|x\|^2
	\le 
		xA^{\top}x
	\le 
		\frac{1}{\delta}\|x\|^2
.
\]
 We write generically $\up$ for $(x_0,\dots,x_d) \in \mathbb{R}^{1+d}$, $\upp$ for any element of $\S_d^{\delta}$, and $\uu\coloneqq  (\up,\upp)$.

\begin{setting}
\label{setting:General_PDE}
 and let $\bar{G}:\mathbb{R}^d\longrightarrow [0,\infty)$ be Borel measurable.  
Fix constants $K_0,K_F\ge 0$,
$L_F,C_g\ge 0$,
 and $0<\delta \le 1$.  We require the following of the domain $\mathcal{D}$.
\end{setting}
\begin{assumption}[Domain Regularity]
\label{ass:Regularity_GeneralCase__DomainRegularity}
The domain $\mathcal{D}\subseteq \mathbb{R}^d$ is a non-empty bounded domain with $C^{1,1}$-boundary
 satisfying the exterior ball condition.
\end{assumption}
Our general approximability result, for which favourable rates cannot generally be guaranteed, considers families of \textit{fully non-linear elliptic} PDEs
\begin{equation*}
    f\big(x,u(x),\nabla u(x),\nabla^2 u(x)\big) =0,  \; x\in \mathcal{D},\; 
    u(x) = g(x), \; x\in \partial \mathcal{D},
\end{equation*}
where the boundary data $g\in W^{k,p}(\partial \mathcal{D})$ is assumed to be sufficiently smooth, \emph{i.e.}\ $k\ge 2$.

\medskip
Following \citeauthor*{krylov2018sobolev} \cite[Chapter 14]{krylov2018sobolev}, our PDEs will have sufficiently regular solutions under the following conditions.
\begin{assumption}
\label{ass:Regularity_GeneralCase__EllipticPDEFunctionalStructure}
Assume that $p>d$, and fix constants $( c_1, c_2, R_0)\in(0, 1]^3$, $L_\smallertext{F} \ge 0$, a function $\omega_\smallertext{F}:[0,\infty) \longrightarrow [0,\infty)$ with $\omega_\smallertext{F}(0) = 0$, a Borel measurable function $\bar{G}:\mathbb{R}^d \longrightarrow [0,\infty)$, and Borel measurable functions $F$ and $G$ of the variables $(u_0, \up, x)$ and $(u, x)$ respectively. We have

\begin{itemize}
\item[$(i)$] $f = F + G$, and for all $\upp \in \S_d^\smallertext{+}$, $\up \in \mathbb{R}^{1\smallertext{+}d}$, and $x \in \mathbb{R}^d$, we have
\begin{equation}
\label{eq:structure_GrowthOfG}
    \big| G(\uu, x) \big| \le 
    c_1 \|\upp\|_\smallertext{F} + c_2 \|\up\| + \bar{G}(x),\; F(0,x) = 0;
\end{equation}
\item[$(ii)$] $F$ is $L_F$--Lipschitz continuous with respect to $\upp;$

\item[$(iii)$] for any $v \in \mathbb{R}$, $0 < r \le R_0$, and $x \in \mathcal{D}$, there exists a convex function $\bar{F}:\S_d \longrightarrow [0,\infty)$ such that
\begin{enumerate}
\item[$(a)$] $\bar{F}(0,x) = 0$, and $\nabla_{u^{\smalltext{\prime}\smalltext{\prime}}} \bar{F}$ has range in $\S_d^{\delta}$ at every point of twice differentiability of $\bar{F};$

\smallskip
\item[$(b)$] for every $\upp \in \S_d^\smallertext{+}$ with $\|\upp\|_F = 1$, we have
\begin{equation}
\label{eq:ConvexityconditionF}
    \inf_{B(r,x) \cap \mathcal{D}} 
    \sup_{\bar{r} > 0}
    \frac{
        \big| \bar{F}(\up_0, r \upp, u) - \bar{F}(\tau \upp) \big|
    }{r}
    \le 
    c_2 \mathrm{Vol} \big( \mathcal{D} \cap B(r,x) \big),
\end{equation}
where $\mathrm{Vol}(A)$ denotes the $d$-dimensional Lebesgue measure of a Lebesgue-measurable set $A \subseteq \mathbb{R}^d;$

\smallskip
\item[$(c)$] for any $(u,v) \in \mathbb{R}^2$, $x \in \mathcal{D}$, and $\up \in \S_d^\smallertext{+}$, we have
\begin{equation}
\label{eq:ModulosconditionF}
    \big| F(u, \upp, x) - F(v, \upp, x) \big|
    \le 
    \omega_\smallertext{F}(|u - v|) \|\upp\|_{\smallertext{F}}.
\end{equation}
\end{enumerate}
\end{itemize}
\end{assumption}

The next definition introduces appropriate perturbations of the original PDE we consider, and uses notations from \Cref{ass:Regularity_GeneralCase__EllipticPDEFunctionalStructure}.
\begin{definition}[PDE perturbation space $\mathcal{X}_k(r)$]
\label{def:pertrubation_set}
Fix $r>0$, $k\in \mathbb{N}^\star$ and let $\mathcal{X}_k(r)$ consist of all pairs $(\bar{G}_0,g)\in
W^{2,p}(\mathcal{D})\times W^{k,p}(\mathcal{D})$ with $\|g\|_{W^{\smalltext{k}\smalltext{,}\smalltext{p}}(\mathcal{D})}\le r$.  
Define $G_0\coloneqq  G+\bar{G}_0$ and, for every pair $(G_0,g)\in \mathcal{X}_k(r)$ denote the solution to the following associated fully non-linear elliptic {\rm PDE} by $u_{\bar{G}_\smalltext{0},g}$
\begin{equation}
\label{eq:EllipticPDE_FullyNonLinear}
\begin{aligned}
    \bigg(
    \underbrace{F+G}_{\mbox{\tiny \rm Structure \normalsize}}
    +
    \underbrace{\bar{G}_0}_{\mbox{\tiny \rm Perturbation}}
    \bigg)\big(x,u(x),\nabla u(x),\nabla^2 u(x)\big) =0,  \; \forall x \in \mathcal{D},
\;
    u(x) = 
    \underbrace{g(x)}_{\mbox{\tiny \rm Perturbation}},
        \; \forall x \in\partial \mathcal{D}.
        \end{aligned}
\end{equation}
\end{definition}
\begin{example}[Source perturbations only]
\label{source_perturbations}
We can, of course, restrict ourselves to perturbations of the source condition itself only, in which case we may restrict our attention to $\bar{G}_0$ which are constant in their first argument; \emph{i.e.} $\bar{G}_0(u,x)=f_0(x)$ for some $f_0\in W^{k,p}(\mathcal{D})$, similarly to the special case in \eqref{eq:semilinear}.
\end{example}

\begin{theorem}[Approximability of the perturbation-to-solution map]
\label{thrm:generalapprox}
Fix $q\in[1,+\infty)$, let $\mathcal{D}$ be a bounded exterior-thick domain in $\mathbb{R}^d$ with $C^{1,1}$-boundary, let $r>0$, $k>1+\max\big\{1,\tfrac{d}{p}\big\}$, and $\mathcal{X}\subseteq \mathcal{X}_k(r)$ be compact.  

\medskip
Suppose {\rm \Cref{ass:Regularity_GeneralCase__DomainRegularity,ass:Regularity_GeneralCase__EllipticPDEFunctionalStructure}} hold
and that both $\sigma_\smallertext{S}$ and $\sigma_\smallertext{W}$ satisfy {\rm \Cref{ass:wavlets}}.  
Then, for every approximation error $\varepsilon>0$, there exists some neural operator $\hat{\Gamma}\in \mathcal{NO}_{\lceil k\rceil,1}$, cf.~\eqref{eq:NO-union}, satisfying the uniform estimate
\begin{equation}
\label{eq:uniform_estimate}
    \sup_{(\bar{G}_\smalltext{0},g)\in\mathcal{X}}\,
    \big\|
            u_{\bar{G}_\smalltext{0},g}
        -
            \hat{\Gamma}(\bar{G}_0,g)
    \big\|_{W^{\smalltext{2}\smalltext{,}\smalltext{p}}(\mathcal{D})}
<
    \varepsilon
.
\end{equation}
\end{theorem}
 The proof of \Cref{thrm:generalapprox} is based on two ingredients.  First, we establish the local--Lipschitz regularity of the coefficient-to-solution map associated to our family of fully non-linear elliptic PDEs (\Cref{lem:HolderRegularity}) verifying the only necessary condition for approximability by continuous models classes; such as our NO, namely continuity---a property which need not be immediate for arbitrary coefficient-to-solution maps.  
 Next, we rely on \Cref{prop:Universality} which establishes a general universal approximation theorem for operators between Besov spaces. 
 
 \medskip
 In this sense, \Cref{prop:Universality} for our NO architecture which, among other things, can be regarded as a generalisation of \citeauthor*{kovachki2023neural} \cite[Theorem 11]{kovachki2023neural}, which does not cover Besov spaces $B_{q,r}^s(\mathcal{D})$ for finite values of $q$ and $r$ ( recall that $W^{s,p}(\mathcal{D})=B_{q,r}^s(\mathcal{D})$ \cite[Remark 1.2]{triebel2008function}).  We emphasise that here, the case of finite $q$ (and $r$) is necessary since $W^{s,\infty}(\mathcal{D})$-spaces are automatically excluded from both \Cref{prop:Universality} and \cite[Theorem 11]{kovachki2023neural}, as well as any encoder-decoder-type model using basis expansions (\emph{e.g.}~\cite{galimberti2022designing}), since $W^{s,\infty}$ is not separable and thus cannot admit a Schauder basis.  Additionally, since this space is non-separable and any realistic NO model must be parameterised by finitely many parameters and depend continuously on them, any realistic NO model defines a separable space, As such, it cannot be dense/universal in spaces of continuous functions between non-separable spaces---again by elementary topological considerations.  

\medskip
We now consider the approximation of a specialized family of elliptic PDEs, whose solution operator exhibits enough structure so that \textit{it} (not all continuous functions) can be approximated on non-separable space $W^{1, \infty}(\mathcal{D})$.

\subsection{Feasible rates}
\label{s:MainResults__ss:FeasibleRates___sss:PDE}
\subsubsection{Semi-linear elliptic PDE}

In what follows, we will make use of the map $S_{\gamma,\mu,\lambda}: W^{(d+3)/2,2}(\partial \mathcal{D}; \mathbb{R}) \longrightarrow W^{1, \infty}(\mathcal{D}; \mathbb{R})$ sending boundary data to domain data, and defined for each $g \in W^{(d+3)/2,2}(\partial \mathcal{D}; \mathbb{R})$ by 
\begin{equation}
\label{eq:Sgamma_boundary_to_domain_operator}
S_{\gamma,\mu,\lambda}(g)  \coloneqq   w_{g},
\end{equation}
where $w_{g} \in W^{(d + 4)/2,2} (\mathcal{D}; \mathbb{R}) \subset W^{1, \infty}(\mathcal{D}; \mathbb{R})$. 
is the unique solution of 
\[
-\nabla \cdot \gamma \nabla w_{g} + \mu \cdot \nabla w_{g} + \lambda w_{g} = 0 \ \mathrm{in} \ \mathcal{D}, \; 
w_{g}=g \ \mathrm{on} \ \partial \mathcal{D}.
\]
We assume the following on the maps $\gamma$, $\mu$ and $\lambda$.
\begin{assumption}
\label{ass:gamma}
The maps $\gamma : \mathcal{D} \longrightarrow \mathbb{R}^{d \times d}$, $\mu : \mathcal{D} \longrightarrow \mathbb{R}^d$, and $\lambda : \mathcal{D} \longrightarrow \mathbb{R}$ satisfy the following conditions
\begin{itemize}
    \item[$(i)$] $\gamma \in C^{\infty}(\bar{\mathcal{D}}; \mathbb{R}^{d\times d})$, $\mu \in C^{\infty}(\bar{\mathcal{D}}; \mathbb{R}^d)$, and $\lambda \in C^{\infty}(\bar{\mathcal{D}}; \mathbb{R})$ where $C^{\infty}(\bar{\mathcal{D}}; \mathbb{R}^d)$ and $C^{\infty}(\bar{\mathcal{D}}; \mathbb{R}^{d\times d})$ denote the spaces of all $d$-dimensional vector-valued and $d\times d$ matrix-valued functions that are infinitely differentiable on $\mathcal{D}$ and whose derivatives admit continuous extensions to the closure $\bar{\mathcal{D}}$;
    \item[$(ii)$] $\gamma$ is uniformly elliptic and bounded in the sense that there are positive constants $\gamma_0$ and $\gamma_1$ such that
\[
\gamma_0 \|\xi\|^2 \leq \xi^\top \gamma(x)\xi \leq \gamma_1 \|\xi\|^2, \; \forall (x,\xi) \in \mathcal{D}\times\mathbb{R}^d;
\]
    \item[$(iii)$] $\mu$ and $\lambda$ are such that
\[
\lambda \geq 0,\; \text{\rm and}\; \lambda \geq \nabla \cdot \mu \sum_{i=1}^{d} \partial_{x_i} \mu
.
\]
\end{itemize}
\end{assumption}
Next, we summaries our main assumptions on $\tilde{f}$.
\begin{assumption}
\label{ass:semilinear-term}
The map $\tilde{f}: \mathcal{D} \times \mathbb{R} \longrightarrow \mathbb{R}$ satisfies
\begin{itemize}
\item[$(i)$] 
there exists $\delta_0 >0$ and $H \in \mathbb{N}^\star\setminus\{1,2\}$ such that 
\[
\tilde{f}(x,z)=\sum_{h=0}^{H}\frac{\partial_{z}^{h}\tilde{f}(x,0)}{h!}z^h,\; \text{\rm for}\; \|z\|<\delta_0,\; \text{\rm and}\; x \in \mathcal{D};
\]
\item[$(ii)$] $\tilde{f}(\cdot,0)=\partial_{z}^{1}\tilde{f}(\cdot,0)=0;$
\item[$(iii)$] $ \partial_{z}^{h}\tilde{f}(\cdot,0) \in C^{\infty}(\bar{\mathcal{D}};\mathbb{R})$ for all $h \in\{2,\dots,H\}$. 
\end{itemize}
\end{assumption}
Assumption (i) posits that $\tilde f(x,z)$ is analytic at $z=0$ and represented by a finite power series truncated at order $H$. Assumption (ii) removes the zeroth- and first-order terms, which are already captured by $f_0(x)$ and $\lambda(x)u(x)$ in \eqref{eq:semilinear}. Assumption (iii) requires all coefficient functions to be smooth, ensuring a well-posed setting for the subsequent analysis.

Finally, we formulate a smallness assumption.
\begin{assumption}
\label{ass:choice-delta-p}
We take $0<\delta< \delta_0$ $($where $\delta_0$ comes from {\rm \Cref{ass:semilinear-term}.$(i)$}$)$ so that 
\begin{align*}
C_1 \delta < 1,
\;
\rho\coloneqq C_2 \delta < 1,
\;
C_3 \delta < 1,
\end{align*}
where the positive constants $C_1$, $C_2$, $C_3$ will appear in \eqref{eq:C-1}, \eqref{eq:C-2}, and \eqref{eq:C-7}, and depend only $p$, $d$, $\mathcal{D}$, $\tilde{f}$, $\gamma$, and $\mu$.
\end{assumption}

Under the above assumptions, we have the following approximation guarantee for the \textit{solution operator} of the PDE associated with our randomly stopped second-order BSDE system~\eqref{eq:FBSDE_ForwardProcess}, \eqref{eq:FBSDE}, \eqref{eq:FBSDE_Martingale}.
\begin{theorem}[{Exponential approximation rates: solution operator to the elliptic problem}]\label{thm:semilinear}
Let\footnote{This is need as our proof relies on the approximation results of~\cite{kim2019green} for the relevant Green's function associated to our PDEs.} $d \geq 3$.
Let {\rm \Cref{ass:gamma,ass:semilinear-term,ass:choice-delta-p}} hold.
Suppose that $\mathcal{D}$ is a bounded open set with Lipschitz boundary in $\mathbb{R}^d$.
Let $1<s<2$ and $1\leq p < \frac{d}{d-1}$.
Then, for any $0<\varepsilon <1$, there are positive integers $L$, $W$, $\widehat{L}$, $\widehat{W}$, and $\Gamma \in { \mathcal{NO}^{\smallertext{L}, \smallertext{W}, \smallertext{I},\alpha}_{\hat{\smallertext{L}}, \hat{\smallertext{W}}} }(W^{1,\infty}(\mathcal{D};\mathbb{R})^{2}, W^{1,\infty}(\mathcal{D};\mathbb{R}))$ such that
\[
\sup_{(f_0,g)\in\Bc} \big\|\Gamma^{\smallertext{+}}(f_0,g) - \Gamma(f_0, S_{\gamma,\mu,\lambda}(g)) \big\|_{W^{\smalltext{1}\smalltext{,} \smalltext{\infty}}(\mathcal{D};\mathbb{R})} \leq \varepsilon.
\]
where the supremum is taken over the set
\[
\Bc\coloneqq B_{W^{\smalltext{1}\smalltext{,}\smalltext{\infty}}(\mathcal{D};\mathbb{R})}(0, \delta^2) 
\times 
B_{W^{ \smalltext{(}\smalltext{d}\smalltext{+}\smalltext{3}\smalltext{)}\smalltext{/}\smalltext{2},2}(\partial \mathcal{D}; \mathbb{R})}(0, \delta^2).
\]
Moreover, we have the following estimates for parameters $L=L(\Gamma)$, $W=W(\Gamma)$, $\widehat{L}=\widehat{L}(\Gamma)$, and $\widehat{W} = \widehat{W}(\Gamma)$,
\[
L \leq C \log (\varepsilon^{-1}), \; W \leq C,\; \widehat{L} \leq C, \; \widehat{W} \leq C \varepsilon^{\smallertext{-} \frac{1}{(s-1)p}},
\]
where $C>0$ depends only on $s$, $p$, $d$, $\mathcal{D}$, $\tilde{f}$, $\gamma$, and $\mu$.
\end{theorem}
Our quantitative approximation rates are available because the family of elliptic PDEs considered here is well structured. In the fully general setting, however, since our NOs are continuous, one should not expect rates, as the solution operator should not even be expected to be continuous (let alone locally--Lipschitz continuous) which is necessary for approximability by the elementary uniform limit theorem from point-set topology, see \citeauthor*{munkres2000topology} \cite[Theorem 21.6]{munkres2000topology}. 
In that case—even if the solution operator is only continuous for general fully non-linear families—the best achievable rates are no better than worst-case bounds for approximating non-linear locally--Lipschitz continuous operators, see \cite{lanthaler2025parametric}, which require an exponential increase in trainable neurons to achieve a linear decrease in error. Thus, even when approximability holds, any such `rate' would be scarcely more informative than a simple existence statement.

\medskip
Consequently, the principal obstacle is approximability, which is twofold:

\medskip
$(i)$ the relevant solution operator in the fully non-linear elliptic case must be regular enough to be approximable by some universal deep-learning class;

\medskip
$(ii)$ our models must be universal on the specific function spaces on which this solution map acts.

\medskip
$(i)$ requires a stability analysis of our PDE family under coefficient perturbations, while $(ii)$ calls for a universal approximation theorem for our architecture, proved via basis-expansion techniques as in \Cref{prop:Universality}, akin in spirit to \cite[Theorem~11]{kovachki2023neural}, that holds on more general Besov spaces over regular Euclidean domains. This two-step scheme was introduced for deep learning in stochastic filtering \cite{horvath2023deep} and refined for differential games in \cite{alvarez2024neural,firoozi2025simultaneously}.

\subsubsection{Solutions to the family of second-order BSDEs}
\label{s:MainResults__ss:FeasibleRates___sss:2FBSDE}
We now derive the stochastic version of the above (deterministic) approximation theorem.
We additionally require the following regularity conditions.
\begin{assumption}[Regularity of forward process]
\label{ass:X_reg}
There is some $x_0\in \mathcal{D}$ such that for each $R>0$
\begin{enumerate}
    \item[$(i)$] \textit{$($local smoothness$)$}: $(\beta, \gamma) \in C_b^\infty(B_{\mathbb{R}^d}(x_0,5R); \mathbb{R}^d\times \mathbb{S}_d^+)^2$;
    \item[$(ii)$] \textit{$($local ellipticity$)$}: $\gamma(x)\gamma(x)^{\top} \geq c_{x_0,R} \mathrm{I}_d$, for every $x \in B_{\mathbb{R}^d}(x_0,3R)$, for some $0 < c_{x_0,R} < 1;$
    \item[$(iii)$] there exists a unique strong solution to~\eqref{eq:FBSDE_ForwardProcess}.
\end{enumerate}
\end{assumption}

\begin{theorem}
\label{thrm:Main_Stochastic}
Let {\rm\Cref{ass:gamma,ass:semilinear-term,ass:choice-delta-p,ass:X_reg}} hold, then, for any $0<\varepsilon <1$ and any time-window $0<T_\smallertext{-}<T_\smallertext{+}$, there are integers $L$, $W$, $\Delta$, $H$, and  $\widehat{\Gamma} \in \mathcal{FB}_{\hat{\smallertext{L}}, \hat{\smallertext{W}}, \hat{\smallertext{\sigma}}}^{\smallertext{L},\smallertext{W}, \smallertext{\rm ReQU}}$ satisfying
\[
        \sup_{(f,g)\in\Bc}\,
            \mathbb{E}^\P\biggl[
                \sup_{\tau\wedge T_{\smalltext{-}}\le t \le T_{\smalltext{+}}\wedge\tau }\,
                    \Big|
                        \widehat{\Gamma}(f,g)_t 
                        - 
                        (Y^x_t,Z^x_t)
                    \Big|
            \biggr]
    \lesssim 
        \varepsilon,
\]
where the supremum is taken over the set
\[
\Bc\coloneqq B_{W^{\smalltext{1}\smalltext{,}\smalltext{\infty}}_\smalltext{0}(\mathcal{D};\mathbb{R})}(0, \delta^2) 
\times 
B_{H^{\smalltext{1} \smalltext{+} \smalltext{(}\smalltext{d}\smalltext{+}\smalltext{1}\smalltext{)}\smalltext{/}\smalltext{2}}(\partial \mathcal{D}; \mathbb{R})}(0, \delta^2).
\]
We have the same estimates for the parameters $L=L(\Gamma)$, $W=W(\Gamma)$, $\widehat{L}=\widehat{L}(\Gamma)$, and $\widehat{W} = \widehat{W}(\Gamma)$ as in {\rm\Cref{thm:semilinear}}. 
\end{theorem}

\section{Experimental results}
In this section, we empirically validate our theoretical findings on two canonical benchmarks in the 2BSDE literature: the periodic semi-linear example of \citeauthor*{chassagneux2023learning} \cite{chassagneux2023learning} and the linear--quadratic control example of \citeauthor*{pham2021neural} \cite{pham2021neural}. We deploy the KANO architecture with a slight modification in the kernel layer (see \ref{sec:architecture} for details). Specifically, rather than jointly learning both the kernel basis and its coefficients, we fix the basis to a Fourier system, obtained via uniform discretisation of the spatial domain, while retaining trainable, Res--KAN--parametrised coefficients. Furthermore, skip connections parametrised by additional Res--KAN layers are introduced on top of the learnable Fourier kernel coefficients. The resulting spectral layer follows the kernel introduced in \citeauthor*{li2021fourier} \cite{li2021fourier}.

\subsection{Periodic semi-linear case}
In this experiment, we study the periodic semi-linear benchmark of \cite{chassagneux2023learning} in dimension $d=5$. This benchmark consists of trigonometric drift--diffusion and has a closed-form solution \(u(t,x)\) depending on \(\sum_{i=1}^5 x_i\). This enables exact supervision of \(u,\nabla u,\nabla^2u\) and pathwise validation under periodic boundary conditions. The forward–-backward SDE system and its closed-form solution are detailed in \Cref{sec:periodic_semilinear}. 

\medskip
A KANO model is trained on $4096$ samples drawn according to the procedure in \Cref{sec:training_pipeline}, and subsequently evaluated along independently generated trajectories using the Euler-–Maruyama sampler described in \Cref{sec:inference_pipeline}. \Cref{fig:semilinear_samples_a,fig:semilinear_samples_b} display the projections of two randomly selected trajectories onto the \((x_1, x_2)\)-plane, together with the corresponding ground-truth solutions \(u\), first and second partial derivatives \(\partial u / \partial x_1\) and \(\partial^2 u / \partial x_1^2\), and the respective predictions produced by the trained model along these trajectories. We observe that the model is generally able to accurately capture the solution, as well as the first and second partial derivatives along the entire trajectories, with only minor discrepancies in the second derivatives.

\subsection{Linear--quadratic case}

We next consider the LQ/Hamilton--Jacobi=-Bellman benchmark proposed in \cite{pham2021neural} in $d = 5$ (see \Cref{sec:linear_quadratic} for details). It represents a HJB-type problem with quadratic cost, whose value function remains quadratic \(u(t,x)=x^\top K(t)x\), and where $K(t)$ satisfies a Riccati ODE. It offers analytic targets for \(u,\nabla u,\nabla^2u\) and a clean test of learning constant-in-space Hessians and optimal-feedback structure. 

\medskip
The same training and inference pipeline as described in the semi-linear case is used, with a KANO network trained on $4096$ samples. \Cref{fig:lq_samples} presents two random trajectories projected onto the $(x_1, x_2)$-plane. The figure also compares the analytic solution $u$, its gradient components $\partial u / \partial x_1$, and the diagonal Hessian entries $\partial^2 u / \partial x_1^2$ with the corresponding model predictions along these paths. The predicted values of $u$ closely follow the analytical solution. The derivatives are recovered with satisfactory accuracy, and the Hessian, which is expected to remain constant in space, is also well captured. Although the estimated derivatives show some deviations from the smooth exact values, their overall accuracy remains high. In summary, the network effectively learns and reproduces the solution $u$ and its derivatives along the generated trajectories.

\begin{figure}[ht!]
\centering
\begin{subfigure}[b]{0.24\textwidth}
   \includegraphics[width=\linewidth]{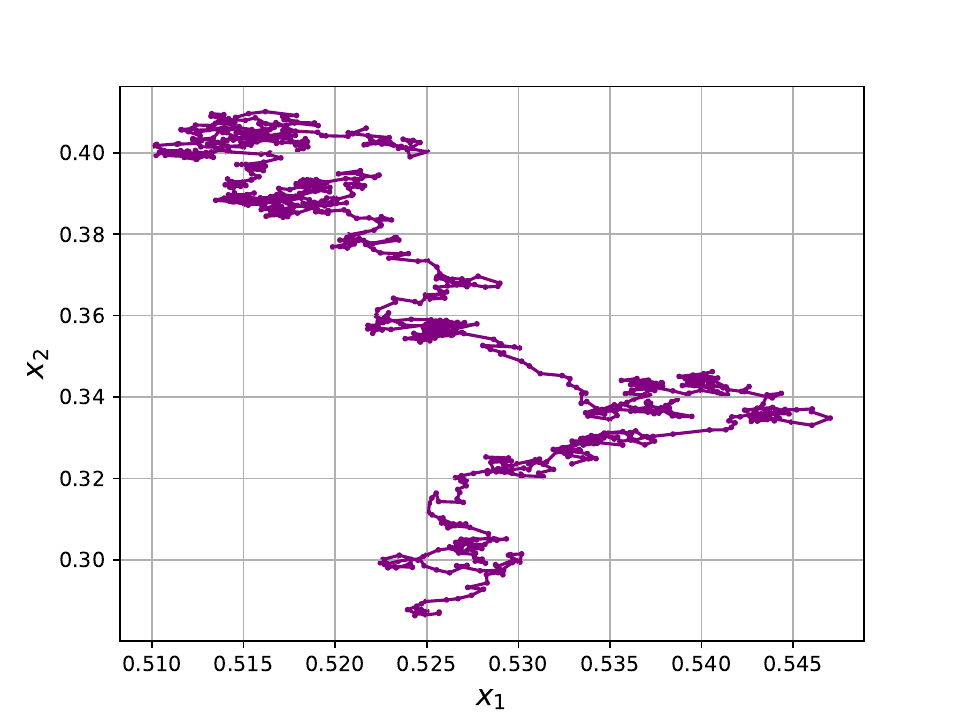}
   \caption{\footnotesize Random paths}
\end{subfigure}
\begin{subfigure}[b]{0.24\textwidth}
   \includegraphics[width=\linewidth]{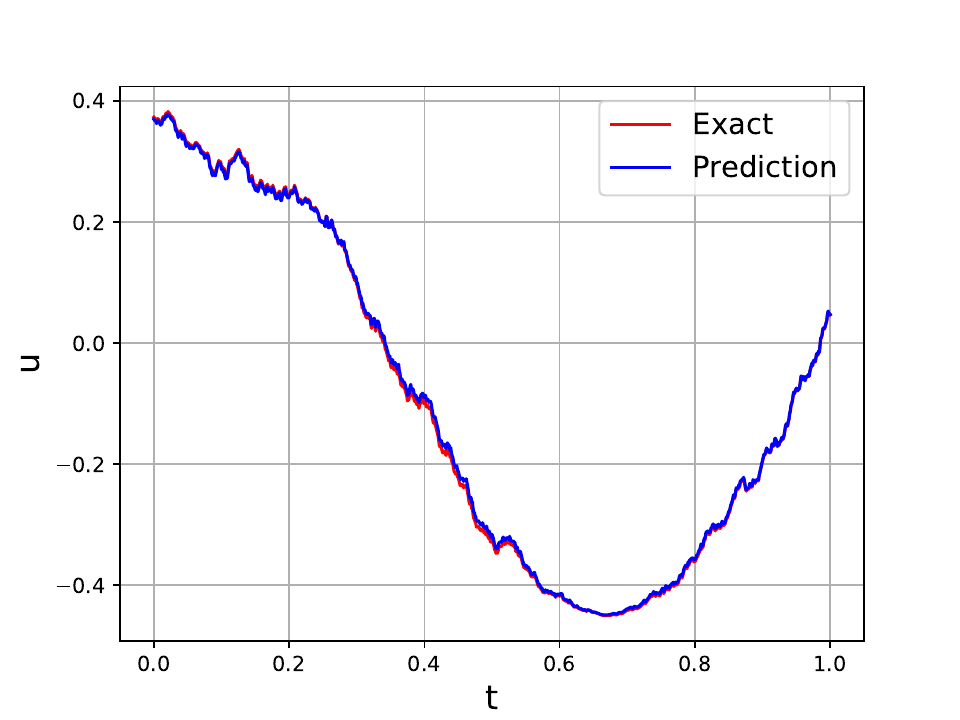}
   \caption{\footnotesize Solutions}
\end{subfigure}
\begin{subfigure}[b]{0.24\textwidth}
   \includegraphics[width=\linewidth]{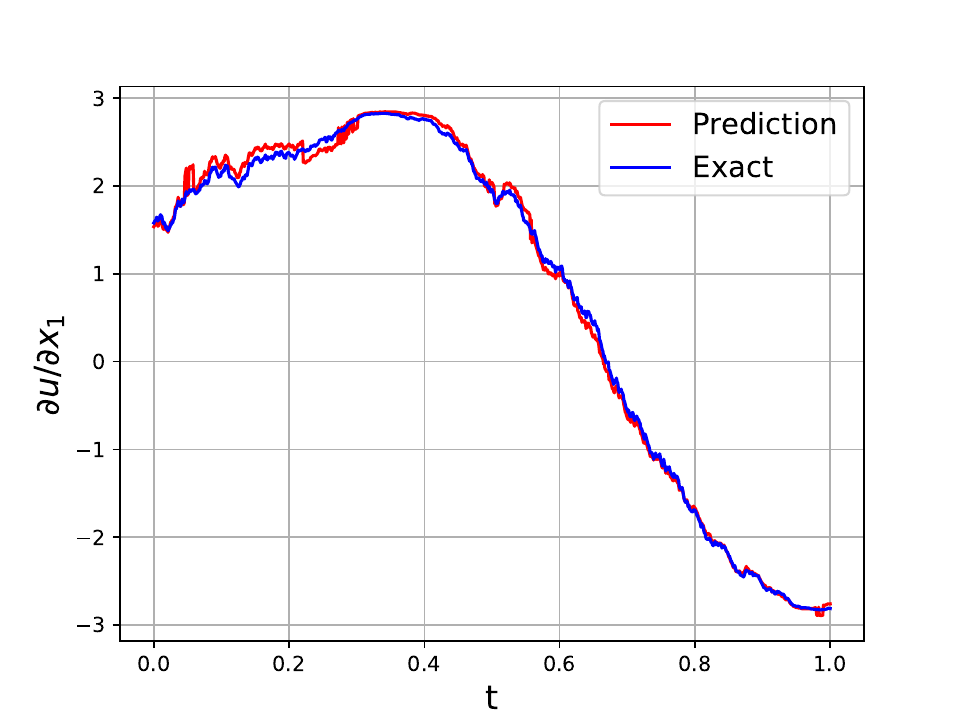}
   \caption{\footnotesize First derivatives}
\end{subfigure}
\begin{subfigure}[b]{0.24\textwidth}
   \includegraphics[width=\linewidth]{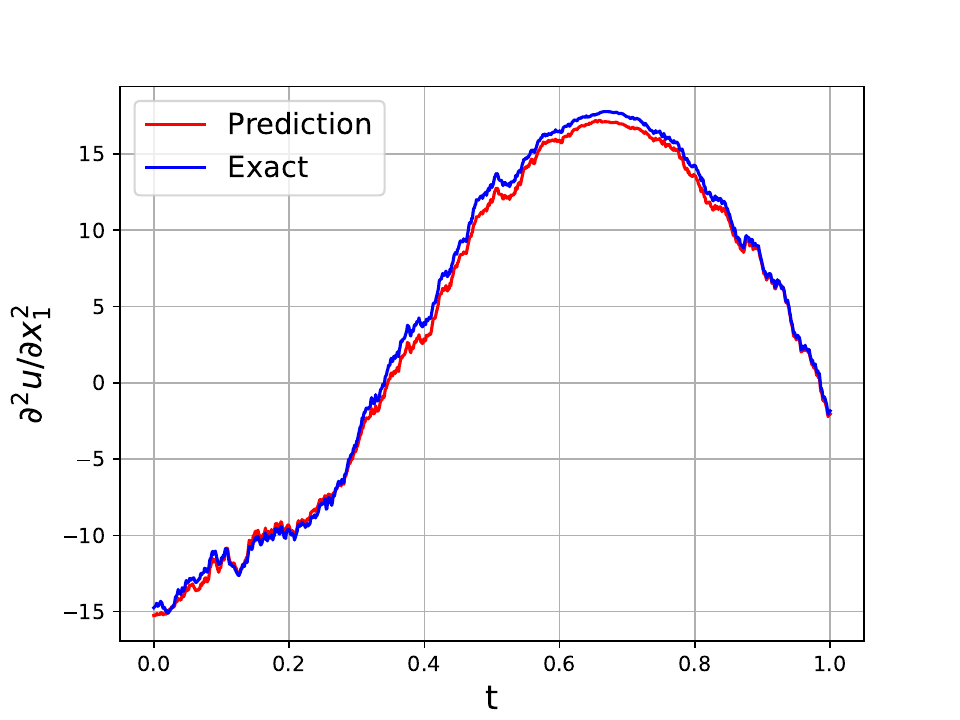}
   \caption{\footnotesize Second derivatives}
\end{subfigure}

\caption{\footnotesize Ground-truth and KANO-predicted solutions for the first randomly selected trajectory of the periodic semilinear example from \cite{chassagneux2023learning}. Each panel shows the projection onto the $(x_1, x_2)$-plane with $u$, $\partial u / \partial x_1$, and $\partial^2 u / \partial x_1^2$ along this path.}
\label{fig:semilinear_samples_a}
\end{figure}

\begin{figure}[ht!]
\centering
\begin{subfigure}[b]{0.24\textwidth}
   \includegraphics[width=\linewidth]{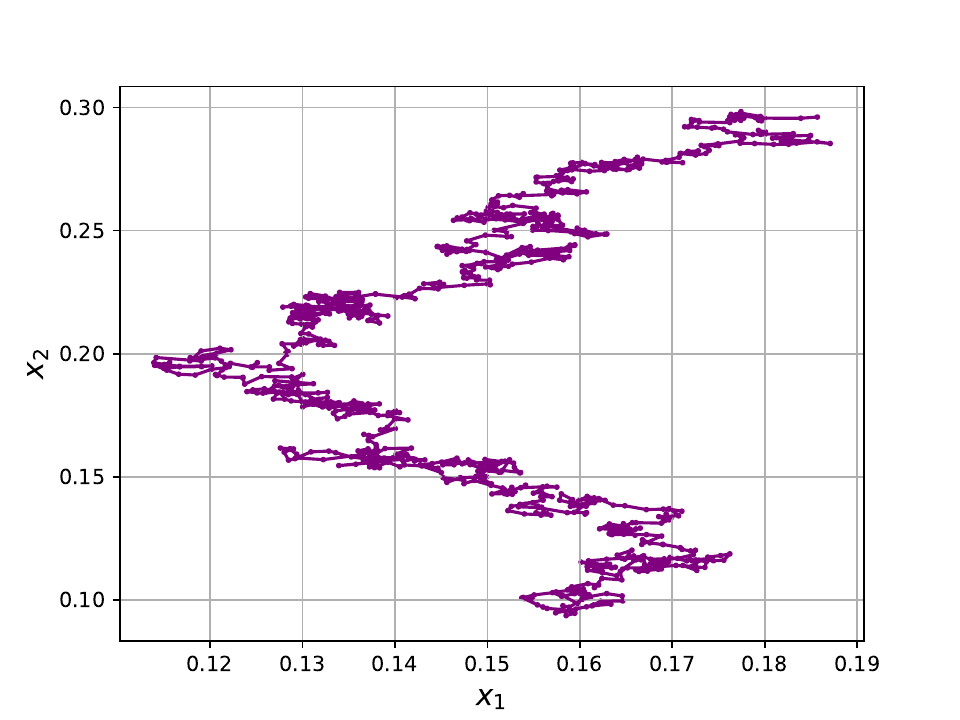}
   \caption{\footnotesize Random paths}
\end{subfigure}
\begin{subfigure}[b]{0.24\textwidth}
   \includegraphics[width=\linewidth]{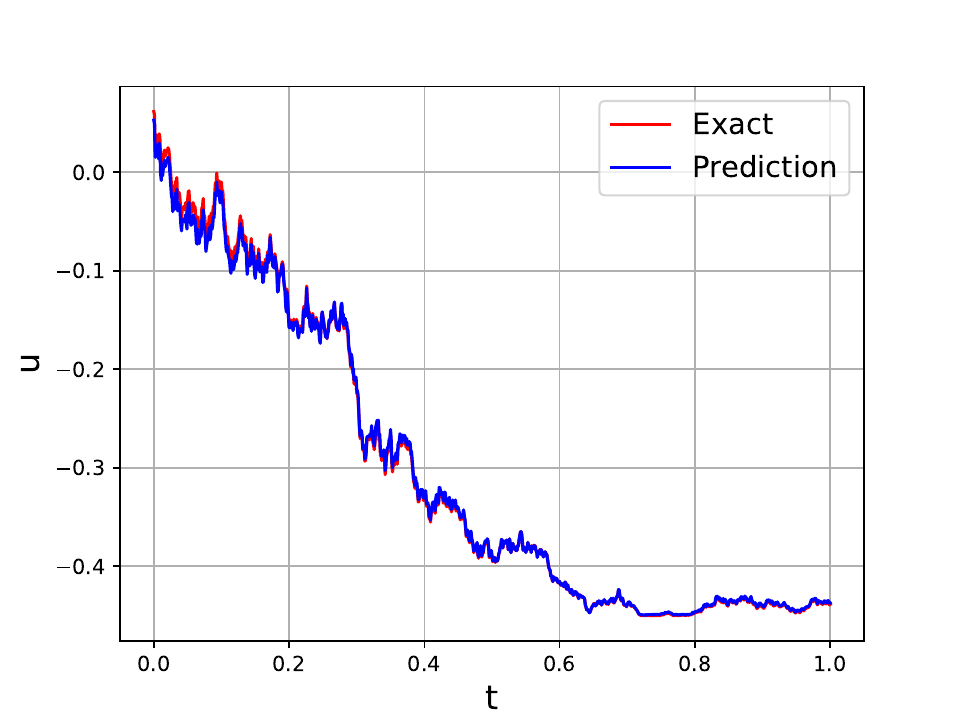}
   \caption{\footnotesize Solutions}
\end{subfigure}
\begin{subfigure}[b]{0.24\textwidth}
   \includegraphics[width=\linewidth]{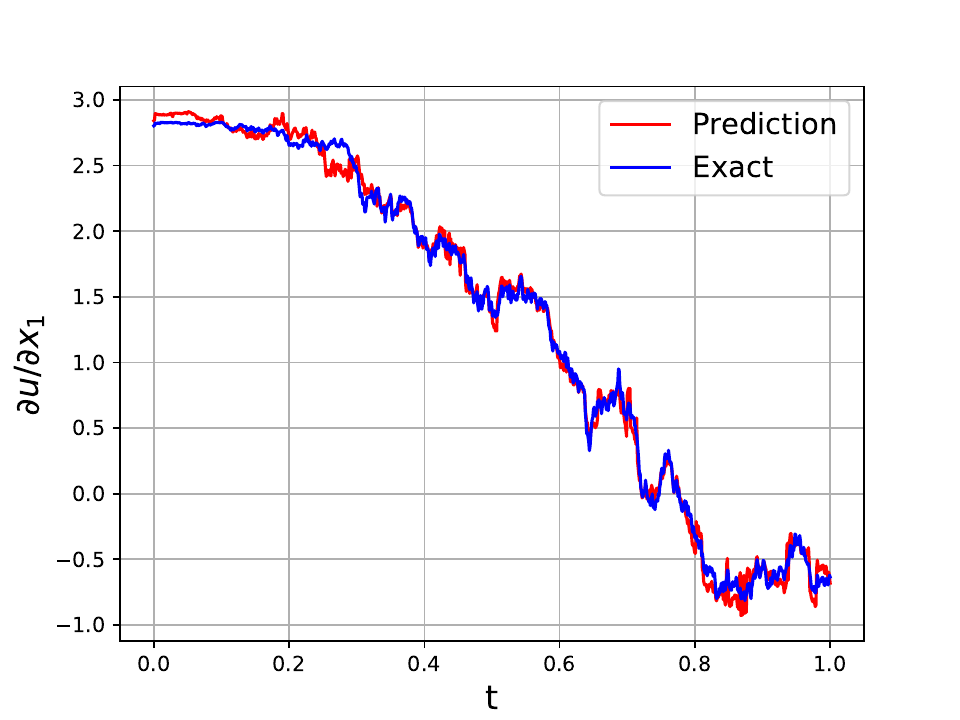}
   \caption{\footnotesize First derivatives}
\end{subfigure}
\begin{subfigure}[b]{0.24\textwidth}
   \includegraphics[width=\linewidth]{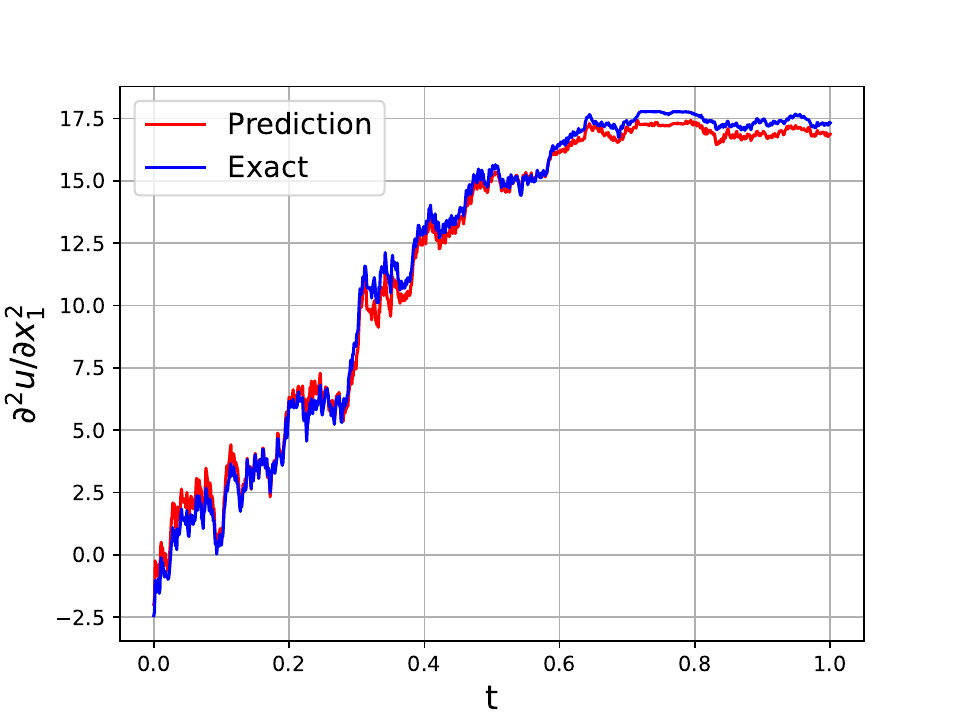}
   \caption{\footnotesize Second derivatives}
\end{subfigure}

\caption{\footnotesize Continuation of \Cref{fig:semilinear_samples_a}, showing the second randomly selected trajectory for the same semi-linear example.}
\label{fig:semilinear_samples_b}
\end{figure}

\begin{figure}[ht!]
\centering
\begin{subfigure}[b]{0.24\textwidth}
   \includegraphics[width=\linewidth]{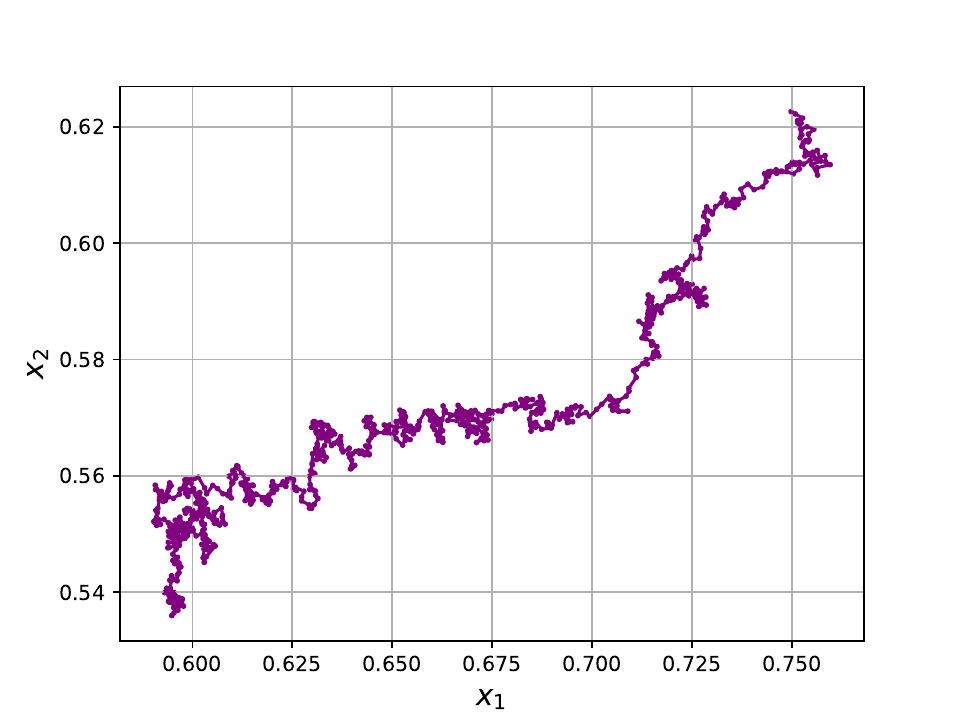}
\end{subfigure}
\begin{subfigure}[b]{0.24\textwidth}
   \includegraphics[width=\linewidth]{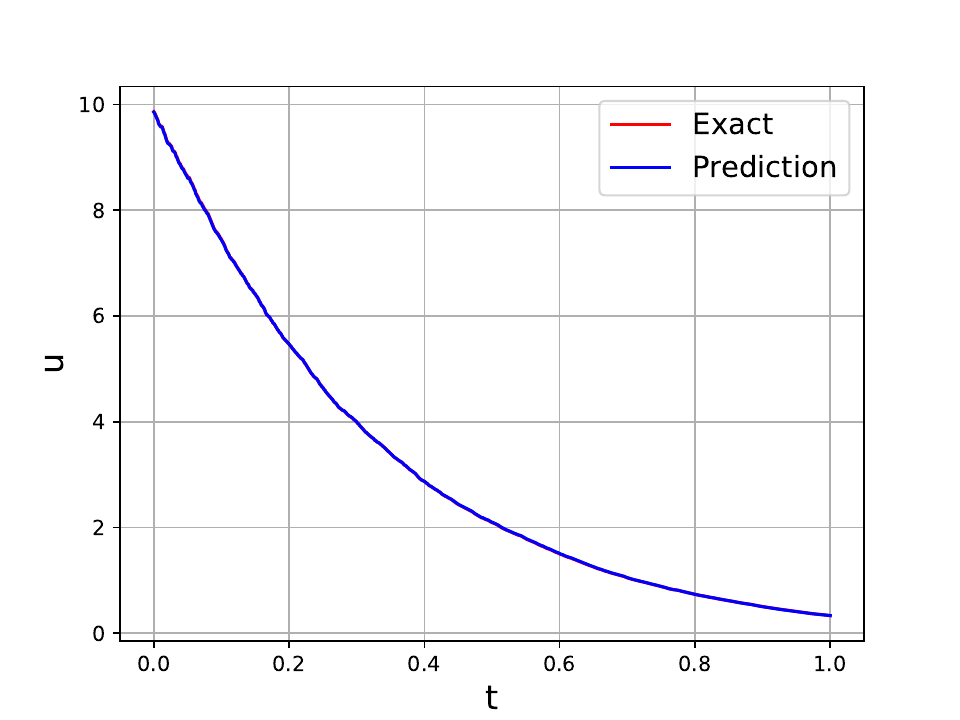}
\end{subfigure}
\begin{subfigure}[b]{0.24\textwidth}
   \includegraphics[width=\linewidth]{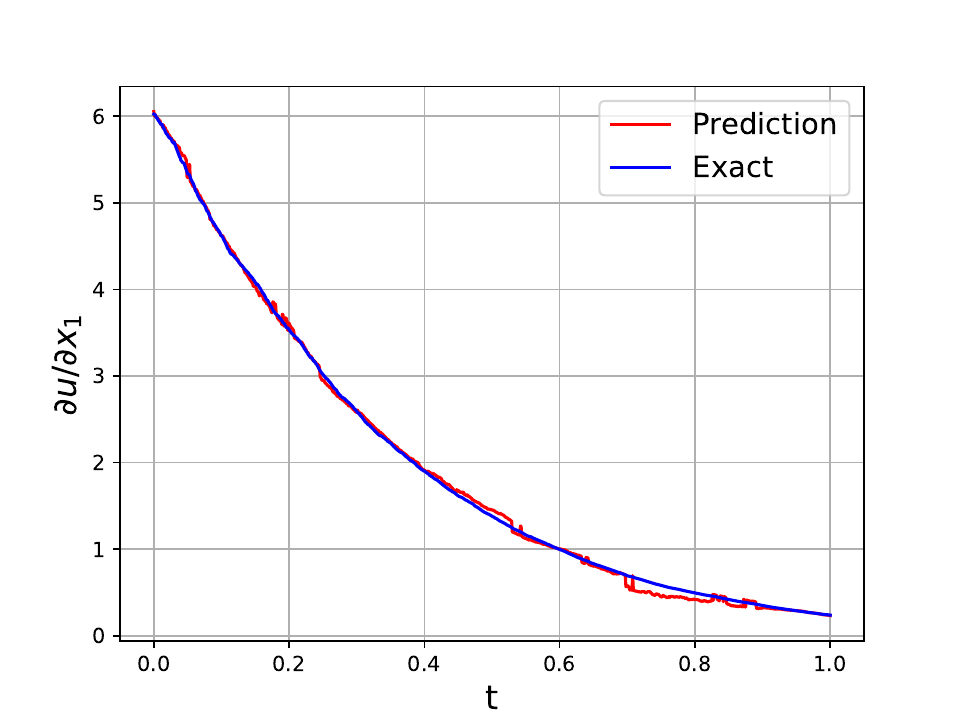}
\end{subfigure}
\begin{subfigure}[b]{0.24\textwidth}
   \includegraphics[width=\linewidth]{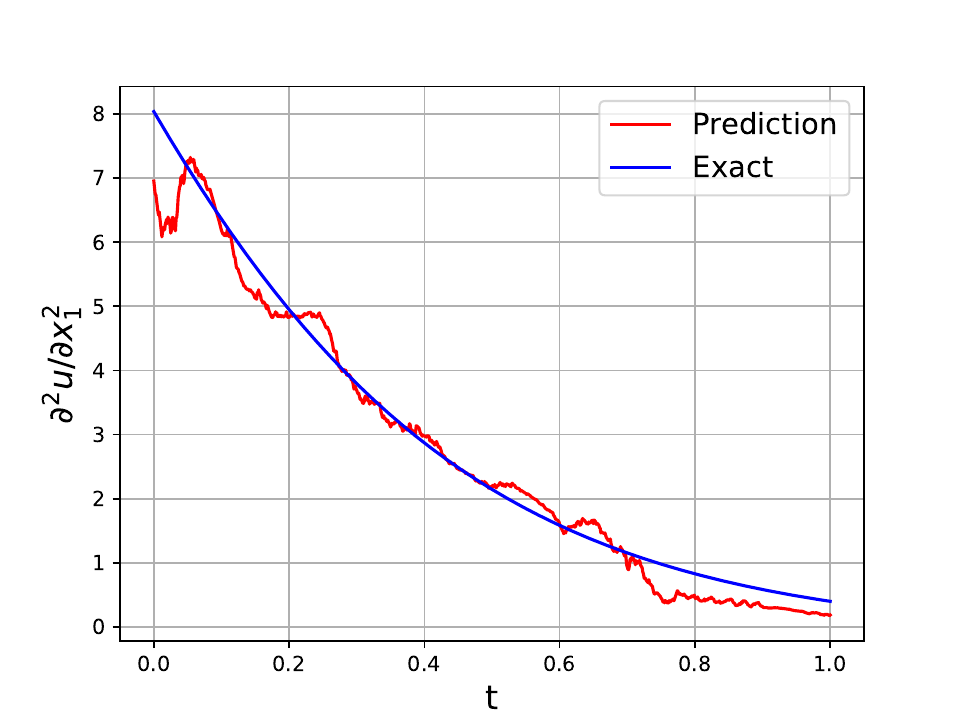}
\end{subfigure}

\begin{subfigure}[b]{0.24\textwidth}
   \includegraphics[width=\linewidth]{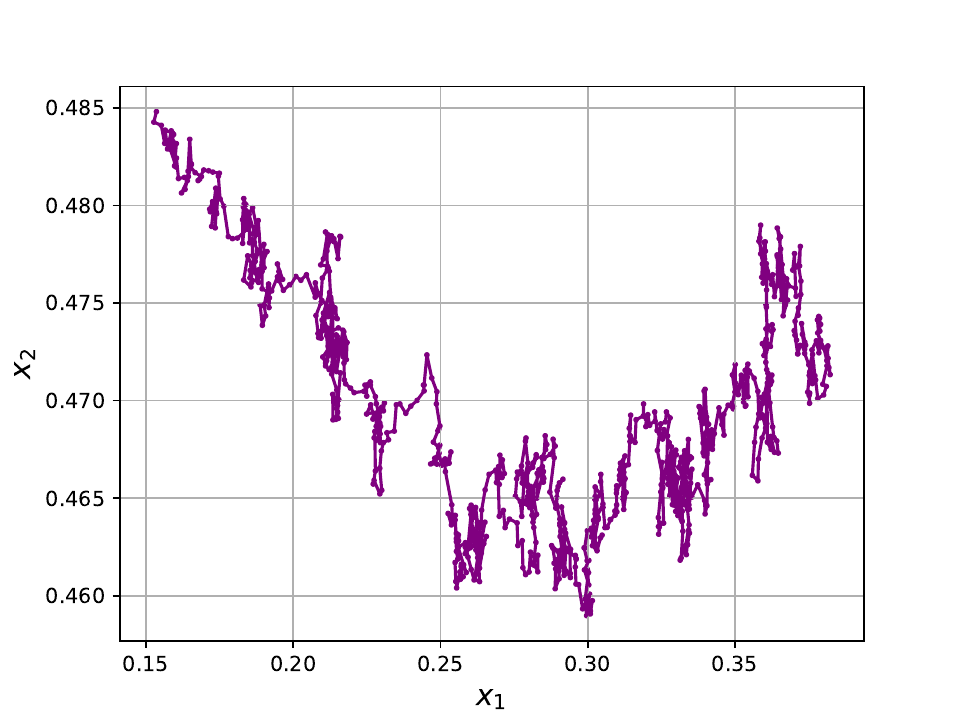}
   \caption{\footnotesize Random paths}
\end{subfigure}
\begin{subfigure}[b]{0.24\textwidth}
   \includegraphics[width=\linewidth]{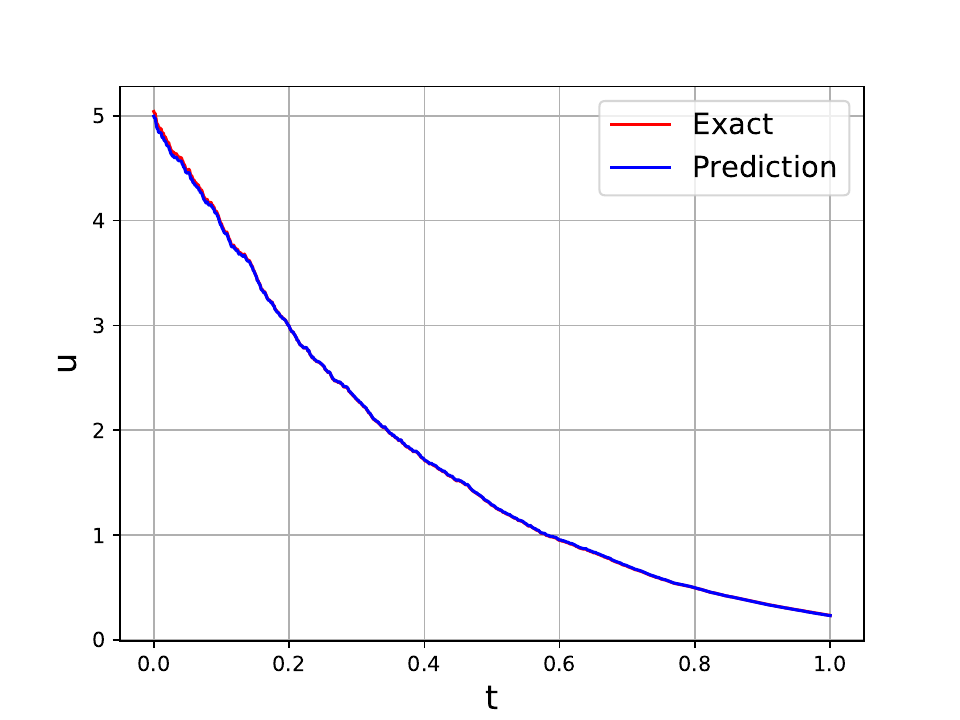}
   \caption{Solutions}
\end{subfigure}
\begin{subfigure}[b]{0.24\textwidth}
   \includegraphics[width=\linewidth]{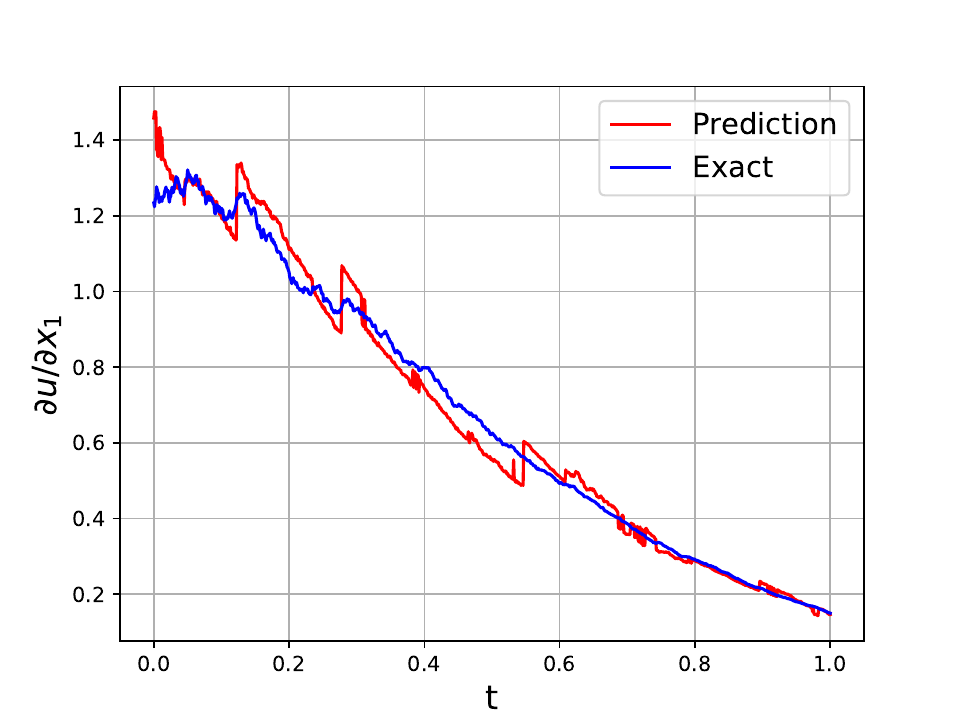}
   \caption{\footnotesize First derivatives}
\end{subfigure}
\begin{subfigure}[b]{0.24\textwidth}
   \includegraphics[width=\linewidth]{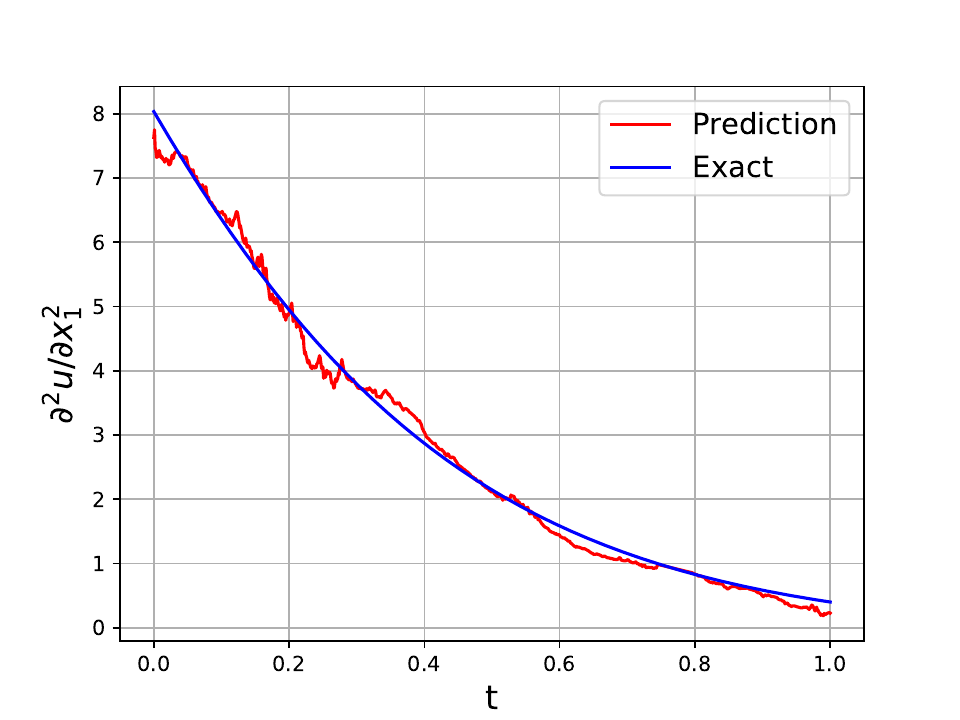}
   \caption{\footnotesize Second derivatives}
\end{subfigure}

\caption{\footnotesize Comparison between the ground-truth and KANO-predicted solutions for the periodic linear--quadratic example of \cite{pham2021neural}. The figure shows two randomly selected trajectories projected onto the \((x_1, x_2)\)-plane, together with the corresponding values of \(u\), \(\partial u / \partial x_1\), and \(\partial^2 u / \partial x_1^2\) along these paths.}
\label{fig:lq_samples}
\end{figure}

\subsubsection{Ablation on the sample size}
We next train a model using eight times fewer training samples than before \emph{i.e.}, 512 samples) and evaluate it following the same procedure as in previous experiments. The resulting quantities of interest are shown in \Cref{fig:lq_low_samples}. We observe that in the vicinity of $t = 0$, the solution $u$ is not well approximated, which in turn affects the accuracy of its first- and second-order partial derivatives. This behaviour is consistent with the theoretical discussion presented earlier: a sufficient number of training samples is required in the high-dimensional space $\mathbb{R}^d$ for the model to accurately capture the solution near $t = 0$.

\begin{figure}[ht!]
\centering
\begin{subfigure}[b]{0.24\textwidth}
   \includegraphics[width=\linewidth]{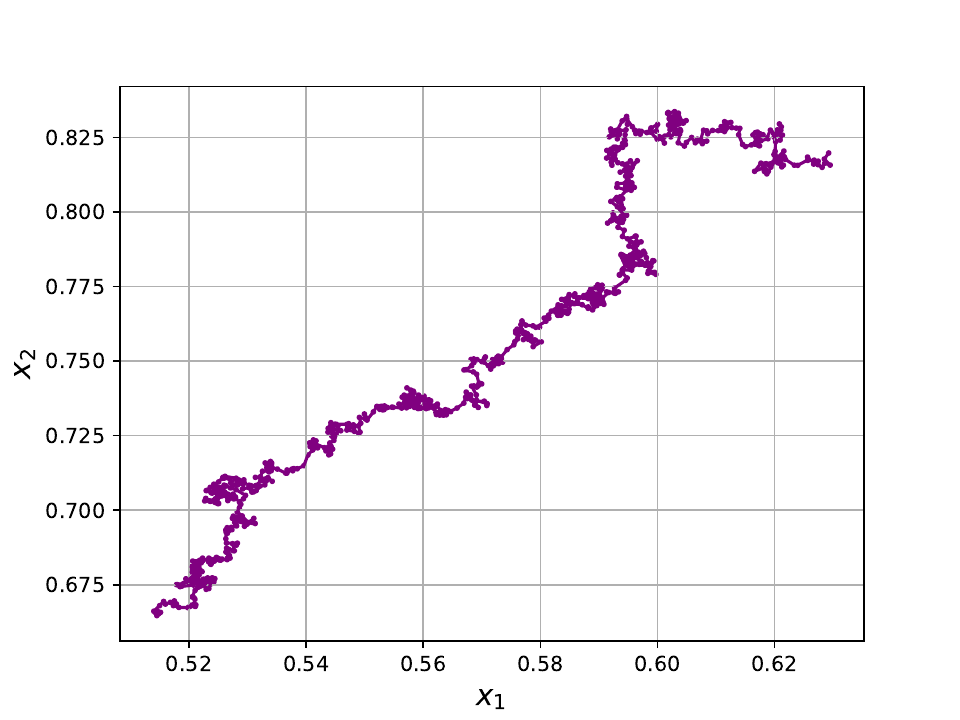}
   \caption{\footnotesize Random path}
\end{subfigure}
\begin{subfigure}[b]{0.24\textwidth}
   \includegraphics[width=\linewidth]{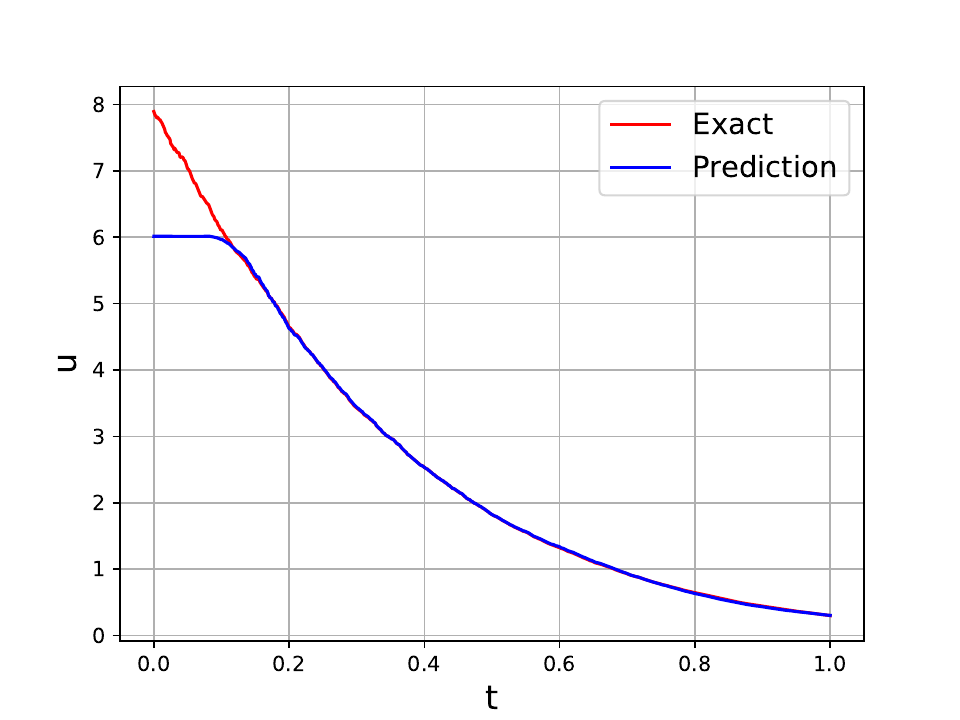}
   \caption{\footnotesize Solution}
\end{subfigure}
\begin{subfigure}[b]{0.24\textwidth}
   \includegraphics[width=\linewidth]{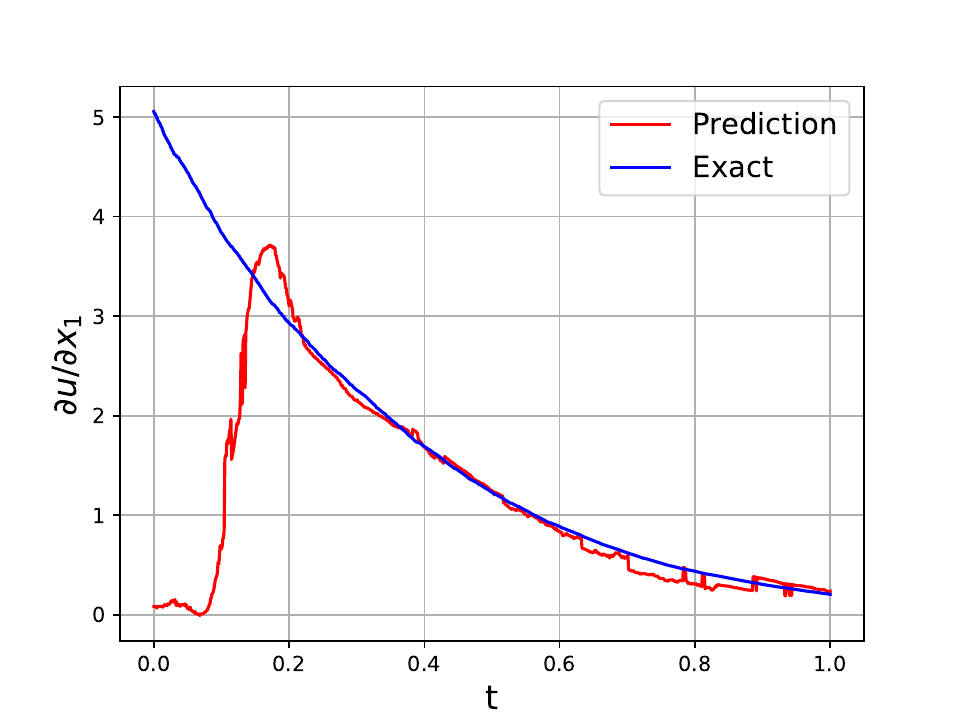}
   \caption{\footnotesize First derivative}
\end{subfigure}
\begin{subfigure}[b]{0.24\textwidth}
   \includegraphics[width=\linewidth]{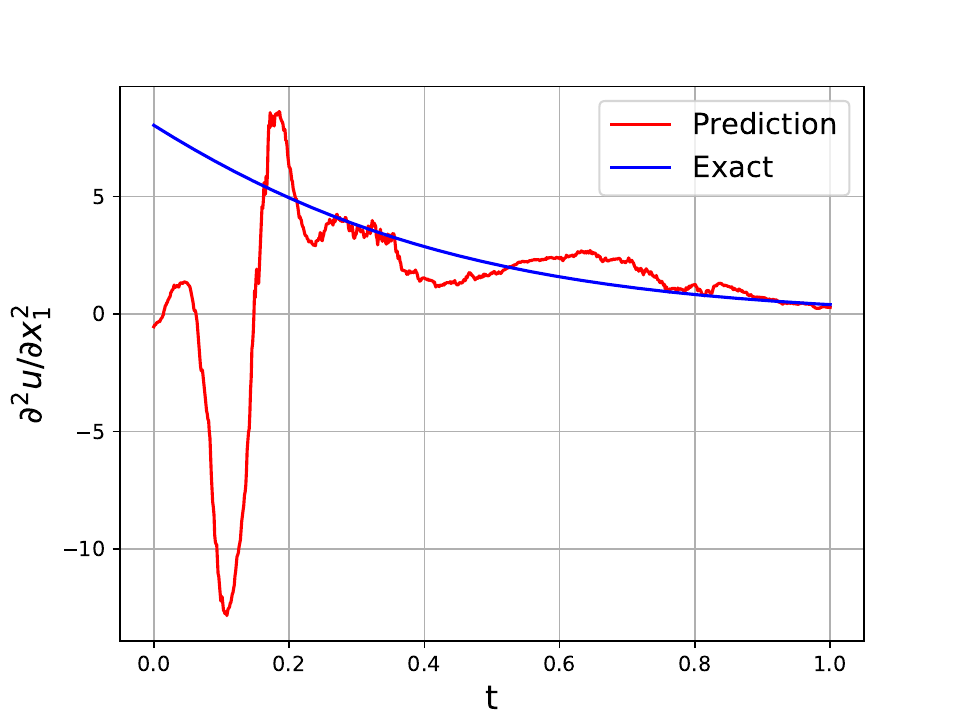}
   \caption{\footnotesize Second derivative}
\end{subfigure}

\caption{\footnotesize Comparison between the ground-truth and KANO-predicted solutions for the periodic linear--quadratic example of \cite{pham2021neural} in \emph{low training data regime}. The figure shows two randomly selected trajectories projected onto the \((x_1, x_2)\)-plane, together with the corresponding values of \(u\), \(\partial u / \partial x_1\), and \(\partial^2 u / \partial x_1^2\) along these paths.}
\label{fig:lq_low_samples}
\end{figure}

\section*{Acknowledgements}
Takashi Furuya was supported by JSPS KAKENHI Grant Number JP24K16949, 25H01453, JST CREST JPMJCR24Q5, JST ASPIRE JPMJAP2329.
Anastasis Kratsios acknowledges financial support from an NSERC Discovery Grant No.\ RGPIN-2023-04482 and No.\ DGECR-2023-00230, and they acknowledge that resources used in preparing this research were provided, in part, by the Province of Ontario, the Government of Canada through CIFAR, and companies sponsoring the Vector Institute\footnote{\href{https://vectorinstitute.ai/partnerships/current-partners/}{https://vectorinstitute.ai/partnerships/current-partners/}}; they would also like to thank Behnoosh Zamanlooy for her support.
Dylan Possama\"{i} gratefully acknowledges partial support by the SNF project MINT 205121-219818.


\appendix
\section{Proof of PDE results}
\label{s:Proofs__PDEResults}
\subsection{Proof of Theorem~\ref{thm:semilinear}}
\label{s:Proofs__PDEResults__ss:semilinear}
This appendix contains the proofs of our paper's main theoretical guarantees.
\subsubsection{Well-posedness}
\label{s:Proof__ss:WellPosedness}

Let  $G_{\gamma,\mu,\lambda} (x,y)$ be a (real-valued) Green’s function for $-\nabla \cdot \gamma \nabla + \mu \cdot \nabla +\lambda$ with a Dirichlet boundary condition, i.e., { for $y \in \mathcal{D}$,}
\[
-\nabla \cdot \gamma \nabla G_{\gamma,\mu,\lambda} (\cdot, y) + \mu \cdot \nabla G_{\gamma,\mu,\lambda} (\cdot, y)
+ 
\lambda
G_{\gamma,\mu,\lambda} (\cdot, y)
= -\delta(\cdot-y)\ \mathrm{in} \ \mathcal{D},
\]
\[
G_{\gamma,\mu,\lambda} (\cdot, y) = 0  \ \mathrm{on} \ \partial \mathcal{D}.
\]
\begin{lemma}
\label{lem:Linfty-Lp-Green}
Let \rm\Cref{ass:gamma} hold.
Then, we have
\[
G_{\gamma,\mu,\lambda} \in W^{s,p}(\mathcal{D} \times \mathcal{D};\mathbb{R}).
\]
where $1 \leq p < \frac{d}{d-1}$ and $1 \le s < 2$.
\end{lemma}

\begin{proof}
From \cite[Theorem 8.1]{kim2019green}\footnote{Note that our setting is that $\gamma$ and $\mu$ are smooth. Thus, they are uniformly Dini continuous, which implies that they are of Dini mean oscillation.}, the Green function $G_{\gamma,\mu,\lambda}(x,y)$ for the operator $Lu\coloneqq-\nabla \cdot \gamma \nabla u + \mu \cdot \nabla u +\lambda u$  can be estimated as for $\beta \in \mathbb{N}^d_0$ with $|\beta| \leq 1$ 
\begin{equation}
\big\| \partial_{x}^{\beta} G_{\gamma,\mu,\lambda}(x, y)\big\| 
\leq C_0 \|x-y\|^{1-d},
\label{eq:green-esti-1}
\end{equation}
where $C_0>0$ is a constant depending on $\mathcal{D}$, $d$, $\beta$, $\gamma$, $\mu$, and $\lambda$. 
Also, applying \cite[Theorem 8.1]{kim2019green} to the Green function $g_{\gamma,\mu,\lambda}(y,x)$ for the adjoint operator $L^{\top}u=-\nabla \cdot( \gamma^{\top} \nabla u + \mu u) +\lambda$, the Green function $g_{\gamma,\mu,\lambda}(y,x)$ can be estimated, for $\beta \in (\mathbb{N}^\star)^d$ with $\|\beta\| \leq 1$ by
\[
\big\| \partial_{y}^{\beta} g_{\gamma,\mu,\lambda}(y, x)\big\| 
\leq C_0 \|y-x\|^{1-d}.
\]
With \cite[Proposition 6.13]{kim2019green} and \Cref{ass:gamma}.$(iii)$, we see that $G(x,y)=g(y,x)$ ($x \neq y$), which implies that
\begin{equation}
\big\| \partial_{y}^{\beta} G_{\gamma,\mu,\lambda}(x, y)\big\| 
\leq C_0 \|x-y\|^{1-d}.
\label{eq:green-esti-2}
\end{equation}
We now choose $R>0$ such that $\mathcal{D} \subset B_{\mathbb{R}^\smalltext{d}}(0,R)$.
Using \eqref{eq:green-esti-2}, we estimate that for $x \in \mathcal{D}$ and $\beta \in (\mathbb{N}^\star)^d$ with $\|\beta\| \leq 1$ 
\begin{align}
\int_{\mathcal{D}} \big\|\partial_{x}^{\beta} G_{\gamma,\mu,\lambda}(x,y) \big\|^{p}\mathrm{d}y 
\lesssim
\int_{\mathcal{D}} \|x-y \|^{(1-d)p}\mathrm{d}y 
\nonumber
=
\int_{x\smallertext{-}\mathcal{D}} \|z \|^{(1-d)p}\mathrm{d}z
&\leq 
\int_{B_{\smalltext{\mathbb{R}}^\tinytext{d}}(0,2R)} \|z \|^{(1-d)p}\mathrm{d}z
\\
&
\lesssim
\int_{0}^{2R} r^{(1-d)p} r^{d-1} \mathrm{d}r
= \int_{0}^{2R} r^{(d-1)(1-p)} \mathrm{d}r
\lesssim 1,
\label{eq:est:green-q}
\end{align}
where we have used that $1 < p <\frac{d}{d-1}$. 
We can obtain the estimate for the derivative with respect to $y$ similarly, using now \eqref{eq:green-esti-2}. 
Note that we use the symbol $\lesssim$ to omit a multiplicative constant that is independent of $x$ on the left-hand side.
\end{proof}

Using the Green function $G_{\gamma,\mu,\lambda}(x,y)$, we define an integral operator encoding (\ref{eq:semilinear}) by:
\begin{equation}
\label{eq:integral-semilinear}
u(x)
\coloneqq
\int_{\mathcal{D}}G_{\gamma,\mu,\lambda}(x,y)\big(\tilde{f} (y,u(y)) -f(y)\big)\mathrm{d}y + w_{g}(x), \; x \in \mathcal{D},
\end{equation}
where $f_0 \in W^{1, \infty}(\mathcal{D};\mathbb{R})$ and $w_{g}(x) \in W^{\frac{d+4}{2},2}(\mathcal{D}; \mathbb{R})$ is the unique solution of 
\[
- \nabla \cdot \gamma \nabla w_{g} + \mu \cdot \nabla w_{g} + \lambda w_{g} = 0, \; \mathrm{on} \; D, \; 
w_{g}=g,\; \mathrm{on} \; \partial D.
\]
where $g \in W^{\frac{d+3}{2},2}( \partial \mathcal{D})$. 
Note that, it is well known that a linear elliptic equation has the unique solution $w_g$ (see, e.g., \cite{gilbarg2001elliptic}).
By the Sobolev embedding theorem (see, \emph{e.g.}, \citeauthor*{evans2010partial} \cite[Section 5.6.3]{evans2010partial}) we have 
\[
W^{(d + 4)/2,2}(\mathcal{D}) \subset C^{(d + 4)/2 -d/2-1, \xi_0}(\overline{\mathcal{D}}) \subset W^{1, \infty}(\mathcal{D}),
\]
where $0<\xi_0 <1$ is a constant. Hence, $w_g \in W^{1, \infty}(\mathcal{D})$. 
We define next the mapping $T$ by 
\[
T(u)(x)
\coloneqq 
\int_{\mathcal{D}}G_{\gamma,\mu,\lambda}(x,y)\big(\tilde{f} (y,u(y)) -f_0(y)\big)\mathrm{d}y + w_{g}(x), \; x \in \mathcal{D},
\]
We set 
\[
B_{W^{\smalltext{1}\smalltext{,} \smalltext{\infty}}}(0,\delta)
\coloneqq \big\{
u \in W^{1, \infty}(\mathcal{D} ; \mathbb{R}) : \|u\|_{W^{\smalltext{1}\smalltext{,} \smalltext{\infty}}(\mathcal{D} ; \mathbb{R})} \leq \delta
\big\},
\]
\[
B_{W^{\smalltext{(d+3)/2}\smalltext{,} \smalltext{2}}}(0,\delta)
\coloneqq \big\{
g \in W^{(d+3)/2, 2}(\partial \mathcal{D} ; \mathbb{R}) : \|g\|_{W^{\smalltext{d+3)/2}\smalltext{,} \smalltext{2}}(\partial \mathcal{D} ; \mathbb{R})} \leq \delta
\big\}.
\]
Then, $B_{W^{1, \infty}}(0,\delta)$ is a closed subset in $W^{1,\infty}(\mathcal{D} ; \mathbb{R})$.

\begin{lemma}
\label{lem:sol-semilinear} 
Let {\rm\Cref{ass:gamma,ass:semilinear-term,ass:choice-delta-p}} hold.  
Let $f \in  B_{W^{\smalltext{1}\smalltext{,} \smalltext{\infty}}(\mathcal{D};\mathbb{R})}(0, \delta^2)$ and $g \in B_{W^{\smalltext{(}\smalltext{d}\smalltext{+}\smalltext{3}\smalltext{) }\smalltext{/}\smalltext{2}, \smalltext{2}}(\partial \mathcal{D}; \mathbb{R})}(0, \delta^2)$. Then, the map $T : B_{W^{\smalltext{1}\smalltext{,} \smalltext{\infty}}}(0,\delta) \longrightarrow B_{W^{\smalltext{1}\smalltext{,} \smalltext{\infty}}}(0,\delta)$ is a $\rho$-contraction where $\rho \in (0,1)$ is defined in {\rm\Cref{ass:choice-delta-p}}. 
In particular, there exists a unique solution of \eqref{eq:integral-semilinear} in $B_{W^{\smalltext{1}\smalltext{,} \smalltext{\infty}}}(0,\delta)$.
\end{lemma}

\begin{proof}
We see that for $ x \in \mathcal{D}$
\begin{align*}
T(w)(x)
&\coloneqq \int_{\mathcal{D}}G_{\gamma,\mu,\lambda}(x,y)\big[\tilde{f}(y,w(y)) -f_0(y)\big]\mathrm{d}y + w_{g}(x)
\\
&=\int_{\mathcal{D}}G_{\gamma,\mu,\lambda}(x,y)\Bigg(\sum_{h = 2}^{H}\frac{\partial^{h}_{z}\tilde{f}(y,0)}{h !}w(y)^{h}-f_0(y) \Bigg)\mathrm{d}y + w_{g}(x)
\\
&=\sum_{h = 2}^{H} 
\frac{1}{h!}\int_{\mathcal{D}}G_{\gamma,\mu,\lambda}(x,y)\partial^{h}_{z}\tilde{f}(y,0)w(y)^{h}\mathrm{d}y - \int_{\mathcal{D}}G_{\gamma,\mu,\lambda}(x,y)f_0(y)\mathrm{d}y + w_{g}(x).
\end{align*}
First, we will show that $T:B_{W^{\smalltext{1}\smalltext{,} \smalltext{\infty}}}(0,\delta) \longrightarrow B_{W^{\smalltext{1}\smalltext{,} \smalltext{\infty}}}(0,\delta)$. 
Let $w \in B_{W^{\smalltext{1}\smalltext{,} \smalltext{\infty}}}(0,\delta)$.
Using this, that $f_0$, and $w_g$ are both in $B_{W^{1,\smalltext{\infty}}}(0, \delta^2)$, and \Cref{lem:Linfty-Lp-Green}, we see that for any $\beta \in (\mathbb{N}^\star)^d$ with $\|\beta\| \leq 1$, we have
\begin{align}
\big\|\partial_{x}^{\beta} T(w)(x)\big\|
& \lesssim 
\int_{\mathcal{D}} \big\|\partial_{x}^{\beta} G_{\gamma,\mu,\lambda}(x,y)\big\| \Bigg(\sum_{h = 2}^{H} 
\frac{1}{h !}| w(y)|^{h} + | f_0(y)|
\Bigg) \mathrm{d}y + \big\|\partial_{x}^{\beta}w_{g}(x)\big\|
\nonumber
\\ 
& \lesssim
\delta^2
\int_{\mathcal{D}}\big\|\partial_{x}^{\beta} G_{\gamma,\mu,\lambda}(x,y)\big\| \mathrm{d}y + \delta^2 
\lesssim \delta^2.
\end{align}
This means that $T(w) \in W^{1, \infty}(\mathcal{D} ; \mathbb{R})$.
We also see that 
\begin{equation}
\label{eq:C-1}
\|T(w)\|_{W^{\smalltext{1}\smalltext{,}\smalltext{\infty}}(\mathcal{D}; \mathbb{R})} \leq C_1 \delta^2, 
\end{equation}
where $C_1 >0$ is a constant depending on $p$, $d$, $\mathcal{D}$, $\tilde{f}$, $\gamma$, and $\mu$.
By choosing $\delta >0$ in \Cref{ass:choice-delta-p}, we have $Tw \in B_{W^{\smalltext{1}\smalltext{,} \smalltext{\infty}}}(0,\delta)$. 

\medskip
Next, we will show that $T:B_{W^{\smalltext{1}\smalltext{,} \smalltext{\infty}}}(0,\delta) \longrightarrow B_{W^{\smalltext{1}\smalltext{,} \smalltext{\infty}}}(0,\delta)$ is a contraction mapping. 
Let $(w_1,w_2) \in B_{W^{\smalltext{1}\smalltext{,} \smalltext{\infty}}}(0,\delta)\times B_{W^{\smalltext{1}\smalltext{,} \smalltext{\infty}}}(0,\delta)$. 
Since 
\begin{align*}
w_1(y)^{h}-w_2(y)^{h} 
& 
= 
\Bigg(\sum_{i=0}^{h-1}w_1(y)^{h-1-i}w_2(y)^i \Bigg)
\big( w_1(y)-w_2(y)\big),
\end{align*}
we deduce that for any $\beta \in (\mathbb{N}^\star)^d$ with $\|\beta\| \leq 1$, 
by H\"{o}lder's inequality and \Cref{lem:Linfty-Lp-Green}
\begin{align*}
\big\|\partial_{x}^{\beta} T(w_1)(x) - \partial_{x}^{\beta} T(w_2)(x)\big\|
& \lesssim 
\sum_{h = 2}^{H} 
\frac{1}{h !}
\int_{\mathcal{D}} \big\|\partial_{x}^{\beta} G_{\gamma,\mu,\lambda}(x,y)\big\| \big| w_1(y)^{h} - w_2(y)^{h} \big| \mathrm{d}y 
\\ 
& \leq
\sum_{h = 2}^{H} 
\frac{1}{h !}
\sum_{i=0}^{h-1}
\int_{\mathcal{D}}\big\|\partial_{x}^{\beta} G_{\gamma,\mu,\lambda}(x,y)\big\|\big|w_1(y)^{h-1-i}w_2(y)^i\big|\big| w_1(y) - w_2(y) \big| \mathrm{d}y
\\ 
& \leq
\sum_{h = 2}^{H} 
\frac{h}{h !}
\delta^{h-1}
\int_{\mathcal{D}}\big\|\partial_{x}^{\beta} G_{\gamma,\mu,\lambda}(x,y)\big\| 
\lesssim \delta \big\| w_1 - w_2 \big\|_{W^{\smalltext{1}\smalltext{,} \smalltext{\infty}}(\mathcal{D};\mathbb{R})}.
\end{align*}
Then, we have that 
\begin{equation}
\label{eq:C-2}
\big\|T(w_1) -T(w_2) \big\|_{W^{\smalltext{1}\smalltext{,} \smalltext{\infty}}(\mathcal{D}; \mathbb{R})} \leq C_2 \delta \| w_1 - w_2 \|_{W^{\smalltext{1}\smalltext{,} \smalltext{\infty}}(\mathcal{D}; \mathbb{R})} 
= \rho \| w_1 - w_2 \|_{W^{\smalltext{1}\smalltext{,} \smalltext{\infty}}(\mathcal{D}; \mathbb{R})},
\end{equation}
where $C_2 >0$ is a constant depending on $p$, $d$, $\mathcal{D}$, $\tilde{f}$, $\gamma$, and $\mu$. By choosing $\delta >0$ as in \Cref{ass:choice-delta-p}, we have that $T$ is $\rho$-contraction mapping in $B_{W^{\smalltext{1}\smalltext{,} \smalltext{\infty}}}(0,\delta)$. 
\end{proof}

Given the previous result, and using Banach's fixed-point theorem, the following solution operator is well-defined
\begin{align*}
\Gamma^{\smallertext{+}} : B_{W^{\smalltext{1},\smalltext{\infty}}}(0, \delta^2) \times B_{W^{\smalltext{(}\smalltext{d}\smalltext{+}\smalltext{3}\smalltext{)}\smalltext{/}\smalltext{2}, \smalltext{2}}}(0, \delta^2) &\longrightarrow B_{W^{\smalltext{1}\smalltext{,}\smalltext{\infty}}}(0,\delta)\\
(f_0, g) &\longmapsto u,
\end{align*}
where, $u$ is the unique solution of \Cref{eq:integral-semilinear} in $B_{W^{\smalltext{1}\smalltext{,} \smalltext{\infty}}}(0,\delta)$.

\subsubsection{Proof of Theorem~\ref{thm:semilinear}}
\label{s:Proofs__ss:Proofofthm:semilinear}
\noindent We now prove Theorem~\ref{thm:semilinear} in a series of several steps.  \add{Throughout, the activation function applied component-wise to the neural operator layers in neural operator's neurons, \emph{i.e.}\ in~\eqref{eq:NO_Neurons}, will always be taken to be the squared-ReLU function, that is to say $\beta=(1,0,\dots,0)$ in~\eqref{eq:KAN_r} for the neural operator. }

\medskip
Let $(f_0,g) \in  B_{W^{\smalltext{1}\smalltext{,}\smalltext{\infty}}}(0, \delta^2) \times B_{W^{ \smalltext{(}\smalltext{d}\smalltext{+}\smalltext{3}\smalltext{)}\smalltext{/}\smalltext{2},2}}(0, \delta^2)$ and let $u \in B_{W^{\smalltext{1}\smalltext{,} \smalltext{\infty}}}(0,\delta)$ be a solution of (\ref{eq:integral-semilinear}), that is, $\Gamma^{\smallertext{+}}(f,g)=u$.
By~\cite[Theorem 1]{kratsios2025kolmogorov}, for any $\eps>0$, there exist Res--KANs, with representation as in \Cref{defn:SmoothedKAns}, $k_{nn}^{h}:\mathbb{R}^d \longrightarrow \mathbb{R}$, $h\in\{2,\dots,H\}$, and $k^\prime_{nn}:\mathbb{R}^d \longrightarrow \mathbb{R}$
such that 
\begin{equation}
\label{eq:approx-Green-NN-1}
\bigg\| k_{nn}^{h}(x,y) - \frac{1}{h!}G_{\gamma,\mu,\lambda}(x,y)\partial^{h}_{z}\tilde{f}(y,0)\bigg\|_{W^{\smalltext{1}\smalltext{,} \smalltext{p}}_{\smalltext{x}\smalltext{,}\smalltext{y}}(\mathcal{D}\times\mathcal{D};\mathbb{R})} \leq \varepsilon,\; h\in\{2,\dots,H\},
\end{equation}
and 
\begin{equation}
\label{eq:approx-Green-NN-2}
\big\|k^\prime_{nn}(x,y) - G_{\gamma,\mu,\lambda}(x,y) \big\|_{W^{\smalltext{1}\smalltext{,}\smalltext{p}}_{\smalltext{x}\smalltext{,}\smalltext{y}}(\mathcal{D}\times\mathcal{D};\mathbb{R})} \leq \varepsilon,
\end{equation}
where depths $\widehat{L}(k_{nn}^{h})$
and $\widehat{L}(k^\prime_{nn})$ are of order $\mathcal{O}(1)$, while the width of $\widehat{W}(k_{nn}^{h})$ and $\widehat{W}(k'_{nn})$ are of order $\mathcal{O}(\varepsilon^{- \frac{1}{(s-1)p}})$. Then, we define by 
\[
\widehat{L} := \widehat{L}(\Gamma) := \max\{ \widehat{L}(k_{nn}^{1}),...,\widehat{L}(k_{nn}^{H}), \widehat{L}(k^\prime_{nn}) \}, \quad 
\widehat{W} := \widehat{W}(\Gamma) := \max\{ \widehat{W}(k_{nn}^{1}), ..., \widehat{W}(k_{nn}^{H}), \widehat{W}(k'_{nn})\},
\]
Then, they are estimated by 
\begin{equation}\label{eq:quanti-esti-nonlinearity}
\begin{cases}
\widehat{L} \leq C , \\[0.8em]
\widehat{W} \leq C \varepsilon^{- \frac{1}{(s-1)p} },
\end{cases}
\end{equation}
where $C > 0$ is a constant depending on $d$, $s$, $H$, and $p$. 
We can then define the map $T_{\smallertext{N}\smallertext{N}}$ by 
\begin{align}
\label{eq:definition_truncarted__T}
T_{\smallertext{N}\smallertext{N}}(u)(x)
&
\coloneqq
\sum_{h=2}^{H}
\int_{\mathcal{D}}k_{nn}^{h}(x,y) 
(u(y))^{h} \mathrm{d}y
-\int_{\mathcal{D}}k^\prime_{nn}(x,y) f(y) \mathrm{d}y
+ w_{g}(x).
\end{align}

\begin{lemma}
\label{lem:step2}
There exists a constant $C_4 >0$ depending on $p$, $d$, $\mathcal{D}$, $\gamma$, $\mu$, and $\lambda$ such that for any $u \in B_{W^{\smalltext{1}\smalltext{,}\smalltext{\infty}}(\mathcal{D}; \mathbb{R})}(0, \delta)$
\[
    \big\|T(u)-T_{\smallertext{N}\smallertext{N}}(u)\big\|_{W^{\smalltext{1}\smalltext{,}\smalltext{\infty}}(\mathcal{D}; \mathbb{R})} \leq C_4 \varepsilon. 
\] 
\end{lemma}
\begin{proof}
Let $u \in B_{W^{\smalltext{1}\smalltext{,}\smalltext{\infty}}(\mathcal{D}; \mathbb{R})}(0, \delta)$.
We see that for $\beta \in (\mathbb{N}^\star)^d$ with $\|\beta\| \leq 1$, 
\begin{align}
\big| \partial_x^{\beta} T(u)(x) - \partial_x^{\beta} T_{\smallertext{N}\smallertext{N}}(u)(x)\big|
\nonumber
&
\leq 
\sum_{h=2}^{H}
\bigg\| k_{nn}^{h}(x,y) - \frac{1}{h!}G_{\gamma,\mu,\lambda}(x,y)\partial^{h}_{z}f(y,0) \bigg\|_{W^{\smalltext{1}\smalltext{,} \smalltext{p}}_{\smalltext{x}\smalltext{,}\smalltext{y}}(\mathcal{D};\mathbb{R})}
\bigg(\int_{\mathcal{D}} |u(y)^{h}|^{p^{\smalltext{\prime}}} \mathrm{d}y\bigg)^{1/p^{\smalltext{\prime}}}
\nonumber
\\
&\quad
+
\big\| k^\prime_{nn}(x,y) - G_{\gamma,\mu,\lambda}(x,y)\big\|_{W^{\smalltext{1}\smalltext{,} \smalltext{p}}_{\smalltext{x}\smalltext{,}\smalltext{y}}(\mathcal{D};\mathbb{R})} 
\bigg(\int_{\mathcal{D}} |
 f(y) |^{p^{\smalltext{\prime}}} \mathrm{d}y\bigg)^{1/p^{\smalltext{\prime}}}\leq C_4 \delta^2 \varepsilon < \varepsilon,
\label{eq:est-step1-4}
\end{align}
which is exactly the desired result.
\end{proof}

\begin{lemma}
\label{lem:T-H-N-M-net-ball}
$T_{\smallertext{N}\smallertext{N}}$ maps $B_{W^{\smalltext{1}\smalltext{,}\smalltext{\infty}}}(0, \delta)$ to itself.
\end{lemma}

\begin{proof}
Fix $u \in B_{W^{\smalltext{1}\smalltext{,}\smalltext{\infty}}(\mathcal{D}; \mathbb{R})}(0, \delta)$.
Using \eqref{eq:C-1} and \eqref{eq:est-step1-4}, we see that 
\[
\big\|T_{\smallertext{N}\smallertext{N}}(u)\big\|_{W^{\smalltext{1}\smalltext{,}\smalltext{\infty}}(\mathcal{D};\mathbb{R})} 
\leq 
\big\|T_{\smallertext{N}\smallertext{N}}(u)\big\|_{W^{\smalltext{1}\smalltext{,} \smalltext{\infty}}(\mathcal{D};\mathbb{R})} 
+
\big\|T(u) - T_{\smallertext{N}\smallertext{N}}(u)\big\|_{W^{\smalltext{1}\smalltext{,} \smalltext{\infty}}(\mathcal{D};\mathbb{R})} 
\lesssim \delta^2. 
\]

Thus, we have that 
\begin{equation}
\label{eq:C-7}
\|T_{\smallertext{N}\smallertext{N}}(u)\|_{W^{\smalltext{1}\smalltext{,}\smalltext{\infty}}(\mathcal{D}; \mathbb{R})} \leq C_3 \delta^2,
\end{equation}
where $C_3>0$ is a constant depending on $s$, $p$, $d$, $\mathcal{D}$, and $\gamma$. By the choice of $\delta$ in \Cref{ass:choice-delta-p}, we see that $\|T_{\smallertext{N}\smallertext{N}}(u)\|_{W^{\smalltext{1}\smalltext{,}\smalltext{\infty}}(\mathcal{D}; \mathbb{R})} \leq \delta$.

\end{proof}

We can now define for an arbitrary positive integer $J$, the map $\Gamma _\smallertext{J}: B_{W^{\smalltext{1}\smalltext{,}\smalltext{\infty}}}(0,\delta^2) \times B_{W^{\smalltext{1}\smalltext{,}\smalltext{\infty}}}(0,\delta^2) \longrightarrow W^{1,\infty}(\mathcal{D}; \mathbb{R})$ by 
\[
\Gamma_{\smallertext{J}}(f_0, w_g) \coloneqq   \underbrace{T_{NN} \circ \cdots \circ T_{\smallertext{N}\smallertext{N}}}_{J\; \text{\rm times}}(0)
\eqqcolon  T_{\smallertext{N}\smallertext{N}}^{[\smallertext{J}]}(0).
\]

\begin{lemma}
Let $J \coloneqq\lceil \log (1/\varepsilon)/\log (1/\rho) \rceil \in \mathbb{N}$. 
Then, there exists a constant $C_5>0$ depending on $p$, $d$, $\mathcal{D}$, $\gamma$, $\mu$, and $\lambda$ such that for all $(f_0,g) \in  B_{W^{\smalltext{1}\smalltext{,}\smalltext{\infty}}}(0, \delta^2) \times B_{W^{ \smalltext{(}\smalltext{d}\smalltext{+}\smalltext{3}\smalltext{)}\smalltext{/}\smalltext{2},2}}(0, \delta^2)$
\[
\big\| \Gamma^{\smallertext{+}}(f_0,g) - \Gamma_\smallertext{J}(f_0,w_g) \big\|_{W^{\smalltext{1}\smalltext{,}\smalltext{\infty}}(\mathcal{D})}
\leq C_5 \varepsilon.
\]
\end{lemma}
\begin{proof}
From \Cref{lem:sol-semilinear}, $T:B_{W^{\smalltext{1}\smalltext{,} \smalltext{\infty}}}(0,\delta) \longrightarrow B_{W^{1, \infty}}(0,\delta)$ is $\rho$-contraction mapping, which implies that 
\begin{align}
\big\|\Gamma^{\smallertext{+}}(f_0,g) - T^{\smallertext{[}\smallertext{J}\smallertext{]}}(0) \big\|_{W^{\smalltext{1}\smalltext{,}\smalltext{\infty}}(\mathcal{D}; \mathbb{R})} 
 = 
\big\|T^{\smallertext{[}\smallertext{J}\smallertext{]}}(u) - T^{\smallertext{[}\smallertext{J}\smallertext{]}}(0) \|_{W^{\smalltext{1}\smalltext{,}\smalltext{\infty}}(\mathcal{D}; \mathbb{R})} 
& \lesssim 
\rho^{J}\|u\|_{W^{\smalltext{1}\smalltext{,}\smalltext{\infty}}(\mathcal{D}; \mathbb{R})}
 \leq
\rho^{J} \delta
\lesssim \varepsilon,
\label{eq:est-net-N-1}
\end{align} 
where $u$ is the unique solution of \eqref{eq:integral-semilinear} in $B_{W^{\smalltext{1}\smalltext{,}\smalltext{\infty}}}(0,\delta)$.
Next, we see that 
\begin{align}
\big\|T^{\smallertext{[}\smallertext{J}\smallertext{]}}(0)  - \Gamma(f_0,w_g) \big\|_{W^{\smalltext{1}\smalltext{,}\smalltext{\infty}}(\mathcal{D}; \mathbb{R})} \nonumber
&= 
\big\|T^{\smallertext{[}\smallertext{J}\smallertext{]}}(0) - T_{\smallertext{N}\smallertext{N}}^{\smallertext{[}\smallertext{J}\smallertext{]}}(0)  \|_{W^{\smalltext{1}\smalltext{,}\smalltext{\infty}}(\mathcal{D}; \mathbb{R})}
\nonumber
\\
\nonumber
& 
\leq 
\sum_{h=1}^{\smallertext{J}}
\Big\|\big(T^{\smallertext{[}\smallertext{J}\smallertext{-}h\smallertext{+}1\smallertext{]}} \circ T_{\smallertext{N}\smallertext{N}}^{\smallertext{[}h\smallertext{-}1\smallertext{]}} \big)(0) - \big(T^{\smallertext{[}\smallertext{J}\smallertext{-}h\smallertext{]}} \circ T_{\smallertext{N}\smallertext{N}}^{\smallertext{[}h\smallertext{]}} \big)(0)  \Big\|_{W^{\smalltext{1}\smalltext{,}\smalltext{\infty}}(\mathcal{D}; \mathbb{R})}
\\
\nonumber
&
\leq 
\sum_{h=1}^{\smallertext{J}} 
\rho^{\smallertext{J}\smallertext{-}h} 
\Big\|\big(T \circ T_{\smallertext{N}\smallertext{N}}^{[h\smallertext{-}1]} \big)(0) - \big(T_{\smallertext{N}\smallertext{N}} \circ T_{\smallertext{N}\smallertext{N}}^{[h\smallertext{-}1]} \big)(0)  \Big\|_{W^{\smalltext{1}\smalltext{,}\smalltext{\infty}}(\mathcal{D}; \mathbb{R})}
\\
&=
\sum_{h=1}^{\smallertext{J}} 
\rho^{\smallertext{J}\smallertext{-}h} 
\big\|T(u_{h}) - T_{\smallertext{N}\smallertext{N}}(u_{h})  \big\|_{W^{\smalltext{1}\smalltext{,}\smalltext{\infty}}(\mathcal{D}; \mathbb{R})},
\label{eq:est-net-N-2}
\end{align}
where, we see that, by \Cref{lem:T-H-N-M-net-ball}
\[
u_{h}  \coloneqq   T_{\smallertext{N}\smallertext{N}}^{[h\smallertext{-}1]}(0) \in B_{W^{\smalltext{1}\smalltext{,}\smalltext{\infty}}}(0, \delta).
\]
Note that we define $T_{\smallertext{N}\smallertext{N}}^{[0]}\coloneqq\mathrm{Id}$. By \Cref{lem:step2}, we see that 
\begin{align*}
&
\big\|T(u) - T_{\smallertext{N}\smallertext{N}}(u)  \big\|_{W^{\smalltext{1}\smalltext{,}\smalltext{\infty}}(\mathcal{D}; \mathbb{R})}
\leq C_4 \varepsilon,
\end{align*}
which implies that with \eqref{eq:est-net-N-2}
\begin{align}
\label{eq:est-net-N-3}
&
\big\|T^{\smallertext{[}\smallertext{J}\smallertext{]}}(0)  - \Gamma(f_0,w_g) \big\|_{W^{\smalltext{1}\smalltext{,}\smalltext{\infty}}(\mathcal{D}; \mathbb{R})} 
\leq 
\sum_{h=1}^{\smallertext{J}} \rho^{\smallertext{J}\smallertext{-}h}
C_5 \varepsilon
\leq 
\sum_{h=0}^{\infty} \rho^{h}
C_5 \varepsilon
=
\frac{C_5}{1-\rho}
\varepsilon
\lesssim \varepsilon.
\end{align}

Thus, by \Cref{eq:est-net-N-1,eq:est-net-N-3}, we conclude that 
\begin{align*}
\big\|\Gamma^{\smallertext{+}}(f_0,g) - \Gamma(f_0,w_g) \big\|_{W^{\smalltext{1}\smalltext{,}\smalltext{\infty}}(\mathcal{D}; \mathbb{R})} 
&
\leq 
\big\|\Gamma^{\smallertext{+}}(f_0,g) - T^{\smallertext{[}\smallertext{J}\smallertext{]}}(0) \big\|_{W^{\smalltext{1}\smalltext{,}\smalltext{\infty}}(\mathcal{D}; \mathbb{R})}  +
\big\|T^{\smallertext{[}\smallertext{J}\smallertext{]}}(0) -  \Gamma(f_0,w_g)\big\|_{W^{\smalltext{1}\smalltext{,}\smalltext{\infty}}(\mathcal{D}; \mathbb{R})} 
\lesssim \varepsilon.
\end{align*}
\end{proof}

Let us remind the reader that $\Gamma_{\smallertext{J}}$ is defined by
\[
\Gamma_{\smallertext{J}}(f_0, w_g) =  \underbrace{T_{\smallertext{N}\smallertext{N}} \circ \cdots \circ T_{\smallertext{N}\smallertext{N}}}_{\smallertext{J}\; \text{times}}(0)
= T_{\smallertext{N}\smallertext{N}}^{\smallertext{[}\smallertext{J}\smallertext{]}}(0).
\]
where the operator $T_{\smallertext{N}\smallertext{N}}$ is defined by
\begin{align*}
T_{\smallertext{N}\smallertext{N}}(u)(x)
&
=
\sum_{h=2}^{\smallertext{H}}
\int_{\mathcal{D}}k_{nn}^{h}(x,y) 
(u(y))^{h} \mathrm{d}y
-\int_{\mathcal{D}}k^\prime_{nn}(x,y) f_0(y) \mathrm{d}y
+ w_{g}(x)
=
\sum_{h=2}^{\smallertext{H}}
\int_{\mathcal{D}}k_{nn}^{h}(x,y) 
(u(y))^{h} \mathrm{d}y
+ v_{f_0,g}(x)
\end{align*}
where 
\[
v_{f_0,g}(x) \coloneqq  
-\int_{\mathcal{D}}k^\prime_{nn}(x,y) f_0(y) \mathrm{d}y
+ w_{g}(x)
\]
We see that $\Gamma_{\smallertext{J}}(f_0, w_g)(x) = v_{\smallertext{J}}(x)$
where $v_{0} \coloneqq  0$ and 
\begin{align*}
&
v_{j+1}(x)
\coloneqq   
\sum_{h=2}^{\smallertext{H}}
\int_{\mathcal{D}}k_{nn}^{h}(x,y) 
(v_j (y))^{h} \mathrm{d}y
+ v_{f_0,g}(x), 
\; j\in\{0,\dots,J-1\}.
\end{align*}

We define 
\[
W^{(0)} \coloneqq 
\begin{pmatrix}
0 & 1 \\
0 & 1 \\
\end{pmatrix}
\in \mathbb{R}^{2 \times 2},
\]
and let $K^{(0)}_N : W^{1, \infty}(\mathcal{D};\mathbb{R})^2 \longrightarrow W^{1,\infty}(\mathcal{D};\mathbb{R})^2$ be defined by 
\[
\bigg(K^{(0)}
\begin{pmatrix}
f_0 \\
w_{g}
\end{pmatrix}
\bigg)(x)  \coloneqq  
\int_{\mathcal{D}}
k^{(0)}_{\smallertext{N}\smallertext{N}}(x,y)
\begin{pmatrix}
f(y) \\
w_{g}(y)
\end{pmatrix} \mathrm{d}y ,
\]
where 
\[
k^{(0)}_{\smallertext{N}\smallertext{N}}(x,y) \coloneqq 
\begin{pmatrix}
k^\prime_{\smallertext{N}\smallertext{N}}(x,y) & 0 \\
k^\prime_{\smallertext{N}\smallertext{N}}(x,y) & 0 \\
\end{pmatrix}
\in \mathbb{R}^{2 \times 2}.
\]
We therefore compute
\[
W^{(0)}
\begin{pmatrix}
f_{0}(x) \\
w_{g}(x)
\end{pmatrix}
+ 
\bigg(K^{(0)}
\begin{pmatrix}
f_{0} \\
w_{g}
\end{pmatrix}
\bigg)(x) 
=
\begin{pmatrix}
v_{f_0,g}(x) \\
v_{f_0,g}(x) 
\end{pmatrix}
=
\begin{pmatrix}
v_{f_0,g}(x) \\
v_{1}(x) 
\end{pmatrix}.
\]
Next, we define $F_{\smallertext{{\rm ReQU}}} : \mathbb{R}^{2} \longrightarrow \mathbb{R}^{H}$ by
\[
F_{ReQU}(u)
\coloneqq 
\begin{pmatrix}
u_1 \\ 
(u_2)^{2} \\
\vdots \\ 
(u_2)^{H}
\end{pmatrix}, \quad
u=(u_1, u_2) \in \mathbb{R}^2,
\]
which can have an exact implementation by a ReQU neural networks (see \citeauthor*{li2020better} \cite[Theorem 3.1]{li2020better}). 

\medskip
We define
\[
W = 
\begin{pmatrix}
1 & 0 & \cdots & 0 \\
1 & 0 &\cdots & 0
\end{pmatrix}
\in \mathbb{R}^{2 \times \smallertext{H}},
\]
and $K : W^{1,\infty}(\mathcal{D};\mathbb{R})^{\smallertext{H}} \longrightarrow W^{1,\infty}(\mathcal{D};\mathbb{R})^{2}$, for $u=(u_1,...,u_{H}) \in W^{1,\infty}(\mathcal{D};\mathbb{R})^{\smallertext{H}\smallertext{+}1}$
\begin{align*}
&
(K u)(x)  \coloneqq  
\int_{\mathcal{D}}
k_{\smallertext{N}\smallertext{N}}(x,y)u(y)\mathrm{d}y
=
\begin{pmatrix}
0
\\[0.5em]
\displaystyle\sum_{h=2}^{\smallertext{H}}
\int_{\mathcal{D}}k_{nn}^{h}(x,y) 
u_h(y) \mathrm{d}y 
\end{pmatrix},
\end{align*}
where 
\[
k_{NN}(x,y) \coloneqq 
\begin{pmatrix}
0 & 0 & \cdots &  0 \\ 
0 & k_{nn}^{2}(x,y) & \cdots & k_{nn}^{\smallertext{H}}(x,y) 
\end{pmatrix}
\in
\mathbb{R}^{2 \times \smallertext{H}},
\]

Then, we have that for $j\in\{1,...,J-1\}$
\begin{align*}
\bigg((W + K) \circ F_{\smallertext{\rm ReQU}} 
\begin{pmatrix}
v_{f_0,g} \\
v_{j}
\end{pmatrix}
\bigg)
(x)
=
W
\begin{pmatrix}
v_{f_0,g}(x) \\
(v_j(x))^2 \\ 
\vdots \\ 
(v_j(x))^{\smallertext{H}}
\end{pmatrix}
+ 
K
\begin{pmatrix}
v_{f_0,g} \\
(v_j)^2 \\ 
\vdots \\ 
(v_j)^{\smallertext{H}}
\end{pmatrix}
(x)
&=
\begin{pmatrix}
v_{f_0,g}(x) 
\\
\displaystyle\sum_{h=2}^{\smallertext{H}}
\int_{\mathcal{D}}k_{nn}^{h}(x,y) 
(v_j(y))^h \mathrm{d}y + v_{f_0,g}(x)
\end{pmatrix}
\\
&= 
\begin{pmatrix}
v_{f_0,g}(x) 
\\
v_{j+1}(x)
\end{pmatrix}.
\end{align*}
Denoting $W^\prime  \coloneqq   (0, 1) \in \mathbb{R}^{1\times 2}$, we finally obtain that  
\begin{align*}
\Gamma_{\smallertext{J}}(f, w_g) 
&
=
W' 
\circ
\Big(
\underbrace{(W + K) \circ F_{\smallertext{\rm ReQU}} \circ \cdots \circ (W + K) \circ F_{\smallertext{\rm ReQU}}}_{\smallertext{J}\smallertext{-}1\; \text{times}}
\Big)
\circ
\big(W^{(0)} + K^{(0)}\big) \begin{pmatrix}
f \\
w_g
\end{pmatrix}.
\end{align*}
Since the ReQU network can be represented by the KANs network \cite[Theorem 3.2]{wang2024expressiveness}, we have, by the above construction,
\[
\Gamma \in \mathcal{NO}^{\smallertext{L}, \smallertext{W}, \smallertext{I},\alpha}_{\hat{\smallertext{L}}, \hat{\smallertext{W}}} (W^{1,\infty}(\mathcal{D};\mathbb{R})^{2}, W^{1,\infty}(\mathcal{D};\mathbb{R})).
\]
Moreover, the depth $L=L(\Gamma)$ and width $W=W(\Gamma)$ of the neural operator $\Gamma$ can be estimated via 
\begin{align*}
&
L(\Gamma) \lesssim J 
\leq C \log (\varepsilon^{-1})
, \;
W(\Gamma) \lesssim H \leq C.
\end{align*}
This concludes our proof of Theorem~\ref{thm:semilinear}; where, again, $\alpha=s$ and $I\coloneqq \lceil \alpha \rceil$.

\subsection{Proof of Theorem~\ref{thrm:generalapprox}}
\label{s:Proofs__PDEResults__ss:ProofGeneralTheorem}
The proof of our second main result relies on some tools from multi-resolution analysis and the wavelet theory of Besov spaces.  We, therefore, now overview the necessary material.
\subsubsection{Additional background}
\label{a:Background}
%
In what follows, we use $S(\mathbb{R}^d)$ to denote the Schwartz space on $\mathbb{R}^d$ and consider the space of distributions defined as the topological dual $D(\mathcal{D})^{\prime}$.  We define the restriction operator sending any distribution $g\in S(\mathbb{R}^d)$ to $g|_{\mathcal{D}}\in D(\mathcal{D})^{\prime}$ defined by restriction of its action to test functions $\varphi\in D(\mathcal{D})$ \emph{i.e.}
\[
    g|_{\mathcal{D}}(\varphi)\coloneqq  g(\varphi)
.
\]

\subsubsection{From wavelet para-bases to Besov spaces on Euclidean spaces}
\label{a:Background__aa:BesovTriebelLizorkin__WaveletBasises}

Fix $u\in \mathbb{N}$ and $(\sigma_\smallertext{S},\sigma_\smallertext{W})\in C^u(\mathbb{R})\times C^u(\mathbb\R)$
satisfy \Cref{ass:wavlets}; that is to say that  $\sigma_\smallertext{S}$ and $\sigma_\smallertext{W}$ are Daubechies father (also known as scaling function) and mother wavelets (also known as wavelet function) respectively, in the sense of \cite{daubechies1988orthonormal}.
For each $j\in \mathbb{N}$ define the sets
\[
        G^j
    \coloneqq  
        \begin{cases}
            \{S,W\}^d, \mbox{ if } j=0,
        \\
            \{S,W\}^{d\star}\coloneqq\{S,W\}^d\setminus\{(S,\dots,S)\}, \mbox{ if } j>0.
        \end{cases}
\]
Now, for each `scale' $j\in \mathbb{N}$, location $m\in \mathbb{Z}^d$, and each $G\in G^j$, define the tensorised Daubechies wavelet by
\begin{equation}
\label{eq:Daubechies_Wavelet}
        \widetilde{\Psi}_{G,m}^j(x)
    \coloneqq  
        2^{jd/2}
        \prod_{i=1}^d
            \sigma_{\smallertext{G}_\smalltext{i}}\big(
                    2^j x_i
                -
                    m_i
            \big),\; x\in\R^d,
\end{equation}
where $G\coloneqq(G_1,\dots,G_d)$.  
Let $\mathcal{O}\coloneqq  \{(j,G,m): j\in \mathbb{N},\; G\in G^j,\; m\in \mathbb{Z}^d\}$ and for each $(j,G,m)\in \mathcal{O}$ let 
\[
        \frac1{(\beta_{\smallertext{G},m}^j)^2}
    \coloneqq  
        \int_{\mathbb{R}^\smalltext{d}}\, 
            \big(
                \tilde{\Psi}_{\smallertext{G},m}^j(x)
            \big)^2
        \mathrm{d}x,
    \mbox{\rm and}\;
        \Psi_{\smallertext{G},m}^j
        \coloneqq  
           \frac1{\beta_{\smallertext{G},m}^j}
            \tilde{\Psi}_{\smallertext{G},m}^j(x)
,\; x\in\R^d.
\]
Then, as discussed on \citeauthor*{triebel2008function} \cite[page 13]{triebel2008function}, for any $u\in \mathbb{N}$ we have that $(\Psi_{\smallertext{G},m}^j)_{(j,\smallertext{G},m)\in \mathcal{O}}$ is an orthonormal basis of $L^2(\mathbb{R}^d)$, and for every $f\in L^2(\mathbb{R}^d)$
\begin{equation}
\label{eq:wavelet_expansion}
        f
    =
        \sum_{j\in \mathbb{N}}
        \sum_{G\in G^\smalltext{j}}
        \sum_{m\in \mathbb{Z}^\smalltext{d}}\,
            \lambda_{\smallertext{G},m}^j
            2^{\smallertext{-}jd/2}
                \Psi_{\smallertext{G},m}^j,\;
\mbox{\rm where}\; 
        \lambda_{\smallertext{G},m}^j
    \coloneqq
        2^{jd/2}
        \int_{ \mathbb{R}^\smalltext{d}}\,
                f(x)
                \Psi_{\smallertext{G},m}^j(x)
            \mathrm{d}x,
\end{equation}
where the series converge in $L^2(\mathbb{R}^d)$.

\medskip
A key properties of Besov spaces, from the approximation theoretic lense, is that they are entirely determined by the decay/convergence rates of the sequences $(\lambda_{\smallertext{G},m}^j)_{(j,\smallertext{G},m)\in \mathcal{O}}$, defined in~\eqref{eq:wavelet_expansion}.  Indeed, for $(q,r)\in(0,+\infty]^2$ and $s\in \mathbb{R}$, if 
\begin{equation}
\label{eq:u_smoothness_condtion}
        u
    >
        \max\{s,\sigma_q-s\},\;
\mbox{where}\;
    \sigma_q\coloneqq  d\max\bigg\{0,\frac{1}{q}-1\bigg\},
\end{equation}
as shown in~\citep[Theorem 1.20]{triebel2008function}, $f\in S(\mathbb{R}^d)^{\prime}$ belongs to the Besov space $\overline{B}_{q,r}^s(\mathbb{R}^d)$ if and only if the sequence $\lambda_{\cdot}\coloneqq  (\lambda_{\smallertext{G},m}^j)_{(j,\smallertext{G},m)\in \mathcal{O}}$, defined by\eqref{eq:wavelet_expansion}, satisfies
\begin{equation}
\label{eq:Definition_bpqs}
\|\lambda_{\cdot}\|_{b_{\smalltext{q}\smalltext{,}\smalltext{r}}^\smalltext{s}}^r
    \coloneqq  
        \sum_{j=0}^{\infty}
        2^{jr(s-d/q)}
        \sum_{G\in G^\smalltext{j}}\,
            \Bigg(
                \sum_{m\in \mathbb{Z}^\smalltext{d}}
                    |\lambda_{\smallertext{G},m}^j|^q
            \Bigg)^{r/q}
<
    \infty,
\end{equation}
with the usual modifications to the left-hand side of~\eqref{eq:Definition_bpqs} if $q$ or $r$ are infinite.  Additionally, the map $f\longmapsto (2^{jd/2}\langle f,\Psi_{\smallertext{G},m}^j\rangle_{L^\smalltext{2}(\mathbb{R}^\smalltext{d})})_{(j,\smallertext{G},m)\in \mathcal{O}}$ is a bi-Lipschitz linear isomorphism between $B_{q,r}^s(\mathbb{R}^d)$ and the (quasi--)Banach space $b_{q,r}^s$ of all sequences for which the (quasi-)norm $\|\cdot\|_{b_{\smalltext{q}\smalltext{,}\smalltext{r}}^\smalltext{s}}$ is finite.

\subsubsection{Besov spaces on domains}
\label{s:Besovbaby}
We begin with the definition of Besov spaces on any domain (proper open set with non-empty interior) $O \subset \mathbb{R}^d$, with closure $\overline{O}$. We write $\mathcal{D}(O)$ for the space of complex-valued compactly supported smooth (test) functions on $O$, topologized with the canonical (Limit of Fr\'{e}chet) LF--topology. Its dual space $D^{\prime}(O)$ is the space of \textit{distributions} on $O$, and a distribution $f\in D^{\prime}(O)$ is said to be supported on a set $A\subseteq O$ if $f(\varphi)=0$ for every $\varphi\in \mathcal{D}(O)$ such that $\varphi(x)=0$ for all $x\not\in A$; the support $\mathrm{supp}(f)$ is the smallest closed set $K$ with this property.  For instance, if $x\in O$ then the Dirac distribution $\delta_x(\varphi) \coloneqq \varphi(x)$ has support $\mathrm{supp}(\delta_x)=\{x\}$, see~\citep[Chapter 2, page 28]{triebel2008function} for further details and notation.
We now define the Besov (quasi-Banach) spaces on $\mathcal{D}$.
\begin{definition}[Besov spaces on domains]
\label{defn:Besov_Domain}
Let $\mathcal{D}$ be a domain, $(q,r)\in(0,+\infty]^2$, and $s\in \mathbb{R}$.  The Besov space $\widetilde{B}_{q,r}^s(\overline{\mathcal{D}})$ consists of all $f\in B_{q,r}^s(\mathbb{R}^d)$ supported in the closure $\overline{\mathcal{D}}$ and $\widetilde{B}_{q,r}^s(\mathcal{D})$ consists of all distributions $f\in D(\mathcal{D})^{\prime}$ for which there exists some $g\in B_{q,r}^s(\overline{\mathcal{D}})$ such that $f=g|_{\mathcal{D}}$.  
In either case, $\mathfrak{D}\in \{\mathcal{D},\bar{\mathcal{D}}\}$, we equip $\widetilde{B}_{q,r}^s(\mathfrak{D})$ with the interpolation norm
\[
\|f\|_{\tilde{B}_{\smalltext{q}\smalltext{,}\smalltext{r}}^\smalltext{s}(\mathfrak{D})}
    \coloneqq  
        \inf\Big\{ \|
                g
\|_{B_{\smalltext{q}\smalltext{,}\smalltext{r}}^\smalltext{s}(\mathbb{R}^\smalltext{d})}:g\in \widetilde{B}_{q,r}^s(\overline{\mathcal{D}}),\; f=g|_{{\mathfrak{D}}}\Big\}
.
\]
We define the Besov spaces $\overline{B}_{q,r}^s(\mathcal{D})$ as follows
\begin{equation}
\overline{B}^s_{q,r}(\mathcal{D}) \coloneqq
\begin{cases}
\widetilde{B}^s_{q,r}(\mathcal{D}),\; \text{\rm if}\; 0 < q \leq \infty,\; 0 < r \leq \infty,\; s > \sigma_q, \\[0.5em]
B^0_{q,r}(\mathcal{D}),\; \text{\rm if}\; 1 < q < \infty,\; 0 < r \leq \infty,\; s = 0, \\[0.5em]
B^s_{q,r}(\mathcal{D}),\; \text{\rm if}\; 0 < q \leq \infty,\; 0 < r \leq \infty,\; s < 0.
\end{cases}
\end{equation}
\end{definition}

Following~\citep[Section 3]{triebel2008function}, we now construct wavelet systems on arbitrary domains (open subsets $\Omega \subset \mathbb{R}^n$) using Whitney decompositions; an object which acts almost as a leitmotif in analysis from our PDE problems to fundamental result in the geometry of functions spaces~\cite{fefferman2006whitney,fefferman2014sobolev}.
The idea is to partition $\Omega$ into dyadic cubes whose sizes adapt to the distance from the boundary, and then build localized wavelet bases on these cubes—maintaining the regularity and cancellation properties of classical $\mathbb{R}^n$ wavelets while conforming to the geometry of $\Omega$.

These spaces can themselves be characterized in a similar way using compactly supported Daubechies wavelets.  We fix a so-called \textit{approximate lattice} $\mathbb{Z}_{\mathcal{D}}\subset \mathcal{D}$ consisting of points $\mathbb{Z}_{\mathcal{D}}=(x_r^j)_{(j,k)\in \mathbb{N}\times \{1,\dots,N_\smalltext{j}\}}$ where, for each $j\in \mathbb{N}$, $N_j\in 
\overline{\mathbb{N}}\coloneqq  
\mathbb{N}\cup\{\infty\}$ for which there exist positive constants $c_1$, $c_2$, $c_3$ satisfying the approximate `lattice separation condition' at any scale $j\in \mathbb{N}$
\begin{equation}
\label{eq:Aprx_lattice_Separation}
    \big|
        x_r^j
        -
        x_{r^\smalltext{\prime}}^j
    \big|
\ge 
    \frac{c_1}{2^j}
\end{equation}
and the separation from the `boundary condition' at scale $j\in \mathbb{N}$
\begin{equation}
\label{eq:Aprx_lattice_Separation2}
    \inf_{\{z\in \mathbb{R}^\smalltext{d}: \|z-x_\smalltext{r}^\smalltext{j}\|\le c_\smalltext{2}/2^\smalltext{j}\}}\;
    \inf_{u\in \partial \mathcal{D}}
        \|
            z
            -
            u
        \|
\ge 
    \frac{c_3}{2^j}
.
\end{equation}
Clearly the constants $c_1$, $c_2$, and $c_3$ may be chosen to guarantee the existence of such a $\mathbb{Z}_{\mathcal{D}}$ for any \textit{domain} $\mathcal{D}$.  Intuitively, $\mathbb{Z}_{\mathcal{D}}$ acts precisely as the dyadic lattices $\bigcup_{j\in \mathbb{N}}2^{\smallertext{-}j}\mathbb{Z}^d$ does in $\mathbb{R}^d$ but is contained entirely within $\mathcal{D}$ and condition~\eqref{eq:Aprx_lattice_Separation} vacuously holds when $\mathcal{D}$ is replaced by the Euclidean space.  

\medskip
For any $L\in \mathbb{N}$, to be specified retroactively, we denote $\sigma_\smallertext{S}^\smallertext{L}(\cdot)\coloneqq  \sigma_\smallertext{S}(2^\smallertext{L}\cdot)$, $\sigma_{\smallertext{W}}^\smallertext{L}(\cdot)\coloneqq  \sigma_{\smallertext{W}}(2^\smallertext{L}\cdot)$, and $\Psi_{\smallertext{G},m}^{j,\smallertext{L}}\coloneqq  \Psi_{\smallertext{G},m}^j(2^\smallertext{L}\cdot)$ for each $(j,G,m)\in \mathcal{O}$. In other words, the factor $L$ rescales our setup and we will choose it so that our problem is properly `shrunk' within our domain and aligned to the approximate lattice $\mathbb{Z}_{\mathcal{D}}$.

\medskip
We are now ready to define wavelet classes tailored to general domains; we follow the terminology in~\citep[Definition 2.4]{triebel2008function}, the existence of which is known (see \emph{e.g.}~\cite[Theorem 2.33]{triebel2008function}).
\begin{definition}[$u$-wavelets]
\label{def:u_wavelets}
Let $\mathcal{D}$ be an arbitrary domain in $\mathbb{R}^n$ with $\mathcal{D} \neq \mathbb{R}^n$ and let $\mathbb{Z}_\mathcal{D}$ ads well as $L \in \mathbb{N}$ and $u \in \mathbb{N}$ be as above. Let $K \in \mathbb{N}$, $D > 0$ and $c_4 > 0$. Then, consider a sub-family of $\{\Psi_{G,m}^j:\, j\in \mathbb{N}_+,\, G\in G^j,\, m\in \mathbb{Z}_{\mathcal{D}}\}$ 
\begin{equation}
\label{eq:WaveletBasis}
\big\{ \Phi_r^j : j \in \mathbb{N};  r \in\{ 1, \ldots, N_j\} \big\},\;  \text{\rm where}\; N_j \in \overline{\mathbb{N}}.
\end{equation}
satisfying: $\mathrm{supp} (\Phi_r^j) \subset B_{\mathbb{R}^d}\big(x_r^j, c_2 2^{-j}\big)$,  $j \in \mathbb{N}$,
is called a $u$-wavelet system $($with respect to $\mathcal{D})$ if it consists of the following three possible types of functions
\begin{enumerate}
    \item[$(i)$] {\rm \textbf{basic wavelets:}}
$
\Phi_r^0 = \Psi_{\smallertext{G},m}^{0,\smallertext{L}}$ for some $G \in \{S, W\}^d$, and $m \in \mathbb{Z}^d;
$
\item[$(ii)$] {\rm \textbf{interior wavelets:}}
$
\Phi_r^j = \Psi_{\smallertext{G},m}^{j,\smallertext{L}}$ for each $j \in \mathbb{N}$, and $m\in \mathbb{Z}_{\mathcal{D}}$
such that $ \mathrm{dist}(x_r^j, \bar{\mathcal{D}}) \geq c_4 2^{\smallertext{-}j},
$
for some $G \in \{S, W\}^{d\star};$
\item[$(iii)$]  {\rm \textbf{boundary wavelets:}}
$
\Phi_r^j = \sum_{\{m^\smalltext{\prime}\in \Z^\smalltext{d}:\|m - m^\smalltext{\prime}\| \leq K\}} d_{m,m^\smalltext{\prime}}^j \Psi_{\tilde{\smallertext{F}}, m^\smalltext{\prime}}^{j,\smallertext{L}}$, for each $ j \in \mathbb{N}$ for which $\mathrm{dist}(x_r^j, \Gamma) < c_4 2^{\smallertext{-}j},
$
for some $m \coloneqq m(j, r) \in \mathbb{Z}^d$ and $d_{m,m^\smalltext{\prime}}^j \in \mathbb{R}$ with
\begin{equation}
\sum_{\{m^\smalltext{\prime}\in\Z^\smalltext{d}:\|m - m^\smalltext{\prime}\| \leq K\}} |d_{m,m^\smalltext{\prime}}^j| \leq D,\; \text{\rm and}\;  \mathrm{supp} \big(\Psi_{\tilde{\smallertext{F}}, m^\smalltext{\prime}}^{j,\smallertext{L}}\big) \subset B(x_r^j, c_2 2^{\smallertext{-}j}).
\end{equation}
\end{enumerate}
\end{definition}
We may now adapt the definition of the sequence spaces $b_{q,r}^s$, given in~\eqref{eq:Definition_bpqs}, to suit the approximate lattice $\mathbb{Z}_{\mathcal{D}}$, and thus the domain $\mathcal{D}$.
\begin{definition}[Sequence space $b_{q,r}^s$]
\label{defn:SequenceSpaces_bpqs_Omega}
Let $\mathcal{D}$ be an arbitrary domain in $\mathbb{R}^n$ with $\mathcal{D} \neq \mathbb{R}^n$, let $\mathbb{Z}_\mathcal{D}$ be as above, $s \in \mathbb{R}$, and $(q,r)\in(o,\infty]^2$. Then $b^s_{q,r}(\mathbb{Z}_\mathcal{D})$ is the collection of all sequences
\begin{equation}
\lambda \coloneqq \big\{ \lambda_r^j \in \mathbb{C} : j \in \mathbb{N},\; r \in\{ 1, \ldots, N_j\} \big\},\; \text{\rm for some}\; N_j \in \overline{\mathbb{N}},
\end{equation}
such that
\begin{equation}
\|\lambda\|_{b^\smalltext{s}_{\smalltext{q}\smalltext{,}\smalltext{r}}(\mathbb{Z}_{\smalltext{\mathcal{D}}})}^q 
\coloneqq  
\sum_{j=0}^\infty 2^{j(s \smallertext{-} n/q)r} \Bigg( 
    \sum_{k=1}^{N_\smalltext{j}} |\lambda_k^j|^q 
\Bigg)^{r/q} 
< \infty
.
\end{equation}
\end{definition}

As we will see shortly, the wavelet system in~\eqref{eq:WaveletBasis} is a Schauder basis for several Besov spaces on domains, provided these domains possess a basic level of generic `thickness' and regularity of their boundaries.  We begin by first noting the relationship between the Besov sequence and function spaces, with the same indices, if the domain has a regular enough boundary.

\medskip
A {domain} $\mathcal{D}\subseteq \mathbb{R}^d$ is said to be \textit{special Lipschitz} if there exists a Lipschitz-continuous map $\beta:\mathbb{R}^{d\smallertext{-}1}\longrightarrow \mathbb{R}$ such that 
\begin{equation*}
        \mathcal{D} 
    = 
        \big\{
                (\tilde{x},x_d)\in \mathbb{R}^{d\smallertext{-}1}\times \mathbb{R}
            :\,
                \beta(x) < x_d
        \big\}
.
\end{equation*}
A \textit{bounded Lipschitz domain} $\mathcal{D}\subset \mathbb{R}^d$ is a bounded domain $\mathcal{D}$ for which there exists a finite number of open balls $(B_1,\dots,B_N)$, for some $N\in\N^\star$, where for $n\in\{1,\dots,N\}$ we have 
\[
B_n\coloneqq  \big\{x\in \mathbb{R}^d: \|x-x^{(n)}\|<r^{(n)}\big\},\; \text{\rm for some}\; x^{(n)}\in \partial\mathcal{D},\; \text{\rm and some}\; r^{(n)}>0,
\]
such that $(B_n)_{n\in\{1,\dots,N\}}$ is a cover of $\partial\mathcal{D}$, and there exist rotations of special Lipschitz domains $(\mathcal{D}_1,\dots,\mathcal{D}_N)\subseteq (\mathbb{R}^d)^N$ for which
\[
    B_n\cap \mathcal{D}
    =
    B_n\cap \mathcal{D}_n,\; n\in\{1,\dots,N\}.
\]

Now, given any domain with Lipschitz boundary, we may characterise the inclusion of any square-integrable function into a wide array of Besov spaces depending on its associated sequence $\lambda$ belonging to the `little Besov' sequence space with the same indices.  The following result is \cite[Corollary 4.28]{triebel2006theory}.
\begin{lemma}[Wavelet para--bases in Besov and Triebel--Lizorkin spaces on bounded Lipschitz domains]
\label{lem:parabasis}
Fix $(q,r)\in(1,\infty)^2$. For $K>0$ small enough, if $5d/2<K$ and $s\in (-K,K)$ then $f\in \mathcal{D}(\mathcal{D})^{\prime}$ belongs to $\overline{B}_{q,r}^s(\mathcal{D})$ $($resp.\ $\overline{F}_{q,r}^s(\mathcal{D}))$ if and only if admits the representation
\begin{equation}
\label{eq:Expansion}
        f 
    = 
        \sum_{(j,\smallertext{G},m)\in S^{\smalltext{\mathcal{D}}}}
    \lambda_{\smallertext{G},m}^j2^{\smallertext{-}j d/2}
            \Psi_{\smallertext{G},m}^j,
\end{equation}
and the following 
holds
    \[
        \Big\|
            \big(
                2^{j(s\smallertext{-}d/q)}
                \big\|
                    (
                        \lambda_{\smallertext{G},m}^j
                    )_{(\smallertext{G},m)\in S^{\smalltext{\mathcal{D}}}_\smalltext{j}}
                \big\|_{\ell^q}
            \big)_{j\in \N}
        \Big\|_{\ell^p} < \infty
.\]
In what follows, given any $f\in \bar{B}_{q,r}^s$ we write $\lambda(f)\coloneqq (\lambda_{\smallertext{G},m}^j)_{j,G,m\in S^{\mathcal{D}}}$ for the sequence defined in~\eqref{eq:Expansion}; provided that it is unique.  We denote the linear map $f\mapsto \lambda(f)$ by $I$.
\end{lemma}
\Cref{lem:parabasis} does not guarantee that the wavelet expansions themselves are uniquely determined. In general, these wavelet `bases' are only frames.  However, the next result shows that this is not necessarily the case for $E$-thick domains.

\medskip
We say that a domain is \textit{exterior thick}, or $E$-thick for short, if there are constants $0<c_\smallertext{L}\le c_\smallertext{U}$ and $j_0\ge 0$ such that for every $j\in \mathbb{N}$ with $j\ge j_0,$ there is a $d$-dimensional `interior' cube $Q\subset \mathcal{D}$ with side-length 
\[
        c_\smallertext{L}2^{\smallertext{-}j}
    \le 
        \max\bigg\{\ell(Q)
        ,
            \sup_{z\in Q^\smalltext{i}}\inf_{u\in \partial\mathcal{D}}
                \|z-u\|
        \bigg\}
    \le 
        c_\smallertext{U}2^{\smallertext{-}j}
\]
where $Q^i$ denotes the interior of any cube $Q$ in the norm relative topology on $\mathcal{D}$ and $\ell(Q)$ denotes its side-length; i.e.\ $\ell(Q)\coloneqq \sup_{x,y\in Q}\, \|x-y\|_{\infty}$; where $\|\cdot\|_{\infty}$ denotes the $\infty$-norm on $\mathbb{R}^d$.
In the case of a thick exterior domain, we obtain a Schauder basis using our $u$-wavelet expansion, see \cite[Theorem 3.13 (ii)]{triebel2008function}.
\begin{theorem}[Wavelet-based Schauder bases]
Let $\mathcal{D}$ be an $E$-thick domain in $\mathbb{R}^d$. Define for $u \in \mathbb{N}^\star$
\[
\big\{\Phi_r^j : j \in \mathbb{N},\; r \in\{1, \ldots, N_j\} \big\},\; \text{\rm for some}\; N_j \in \mathbb{N},
\]
an orthonormal $u$-wavelet basis in $L_2(\mathcal{D})$. Then let $\overline{B}^s_{q,r}(\mathcal{D})$ be the space in {\rm \cite[Equation (3.46)]{triebel2006theory}} and let
\[
u > \max\big\{s, \sigma_{q,r} - s\big\}, \; s \ne 0.
\]
Then $f \in \mathcal{D}^\prime(\mathcal{D})$ is an element of $\overline{B}^s_{q,r}(\mathcal{D})$ if and only if it can be represented as
\[
f = \sum_{j=0}^\infty \sum_{k=1}^{N_\smalltext{j}} \lambda_k^j 2^{\smallertext{-}j d/2} \Phi_k^j, \; \lambda \in b^s_{q,r}(\mathbb{Z}_\mathcal{D}),
\]
with convergence holding in $\mathcal{D}^\prime(\mathcal{D})$ and locally in any spaces $\overline{B}^\sigma_{q,r}(\mathcal{D})$ with $\sigma_{q,r} < s$.
\medskip
Furthermore, if $f \in \overline{B}^s_{q,r}(\mathcal{D})$ then the representation is unique with $\lambda = \lambda(f)$ as in~\eqref{eq:Expansion} and $I$ the linear map in Lemma~\ref{lem:parabasis} is an bi-Lipschitz isomorphism of Banach spaces mapping $\overline{B}^s_{q.r}(\mathcal{D})$ onto $b^s_{q,r}(\mathbb{Z}_\mathcal{D})$. If, in addition, $q < \infty$, $r < \infty$, then $(\Phi_k^j)_{\{(j,k)\in\N^\smalltext{2}:k\in\{1,\dots,N_\smalltext{j}\}\}}$ is an unconditional basis in $\overline{B}^s_{q,r}(\mathcal{D})$.
\end{theorem}

Having covered the necessary background, we now prove our universal approximation result, see \Cref{prop:Universality} below.
\subsection{Proof of universal approximation}
\label{s:Proofs_UniversalApproximation}
We now express the previous result in terms of neural networks.  

\begin{lemma}[Wavelet implementation on domains]
\label{lem:SchauderImplementation}
Let $\mathcal{D}$ be a bounded domain with Lipschitz boundary\footnote{The following result holds, more general on $(\epsilon,\delta)$-domains and thus on any Lipschitz domain; however, we will not need that level of generality in the remainder of our paper.}, let $\sigma_\smallertext{W}$ and $\sigma_\smallertext{S}$ satisfy {\rm\Cref{ass:wavlets}} and $s\ge 2$.
Let $G\in \{S,W\}^{d\star}$, $j\in \mathbb{N}$, and $m\in \mathbb{Z}_{\mathcal{D}}$. Then there exists a {\rm Res--KAN} 
$\widehat{\Psi}_{\smallertext{G},m}^j:\mathbb{R}^d\longrightarrow \mathbb{R}$ of depth $d$, width at-most $2d+1$, and using at-most $(5d^2 + 25d + 2)/2$ non-zero parameters satisfying
\[
        \Psi_{\smallertext{G},m}^j(x)
    =
        \widehat{\Psi}_{\smallertext{G},m}^j,\; x\in\R^d.
\]
\end{lemma}

Our proof will use a recent result,~\citep[Lemma 1]{kratsios2025kolmogorov}, which shows that the $d$-ary multiplication operator can be \textit{exactly} implemented using Res--KANs, but only locally.  This is in contrast to ReLU MLPs, which can only \textit{approximate} it locally.

\begin{lemma}[Exact multiplication on arbitrarily large hypercubes]
For every $d\in \mathbb{N}^\star$ and each $M>0$, there exists a {\rm Res--KAN} $\times^2_d:\mathbb{R}^d\longrightarrow \mathbb{R}$ satisfying for each $x\in [-M,M]^d$
\[
    \times^2_d(x) = \prod_{i=1}^d x_i.
\]
Moreover $\times^2_d$ has depth $d$, width at-most $2d+1$, and at-most $(5d^2 + 21d)/2$ non-zero parameters.
\end{lemma}
We can now proceed with the 
\begin{proof}[{Proof of Lemma~\ref{lem:SchauderImplementation}}]
Recall that \Cref{ass:wavlets}, implies that $\sigma_\smallertext{S}$ in~\eqref{eq:KAN_r} is a scaling function (father wavelet) and $\sigma_\smallertext{W}$ in~\eqref{eq:KAN_r} is the corresponding mother wavelet.  In fact, by \Cref{ass:wavlets}, both are Daubechies wavelets and are thus are in $C^u(\mathbb{R})$ and \textit{compactly supported}. By their continuity, they are thus bounded.  Whence, there is some $M>0$ such that $\sigma_{\smallertext{G}}(\mathbb{R})\subseteq [-M,M]$ for each $G\in \{S,W\}$.

\medskip
Consequently, for every specification $G=(G_1,\dots,G_d)\in \{S,W\}^{d\star}$, for every $j\in \mathbb{Z}$, we may represent the $($multivariate$)$ Daubechies wavelet $\Psi_{\smallertext{G},m}^j$, defined by rescaling the associated un-normalised wavelet $\widetilde{\Psi}_{\smallertext{G},m}^j$ in~\eqref{eq:Daubechies_Wavelet}, by

\begin{align*}
 \nonumber   \Psi_{G,m}^j(\cdot) =
    \prod_{i=1}^d
    \frac{2^{jd/2}}{\beta^j_{\smallertext{G},\smallertext{W}}}
    \sigma_{\smallertext{G}_\smalltext{i}}
    \big(
        2^{jd/2} \cdot - m
    \big) =
    \Bigg(
        \prod_{i=1}^d
        \frac{2^{jd/2}}{\beta^j_{\smallertext{G},\smallertext{W}}}
    \Bigg)
    \prod_{i=1}^d
    \sigma_{\smallertext{G}_\smalltext{i}}
    \big(
        W_0^j \cdot - m
    \big)
& \eqqcolon
    \kappa_{\smallertext{G},\smallertext{W}}^j
    \prod_{i=1}^d
    \sigma_{\smallertext{G}_\smalltext{i}}
    \big(
        W_0^j \cdot - m
    \big)
\\
\numberthis
\label{eq:KAN_product_represenation}
& =  
    \kappa_{\smallertext{G},\smallertext{M}}^j
    \times_d^2
    \circ 
    \sigma_{\smallertext{G}_\smalltext{i}}
    \big(
        W_0^j \cdot - m
    \big),
\end{align*}
where $\beta_{G,W}^j\coloneqq\|\Psi_{G,m}^j\|_{L^2(\mathbb{R})}$ 
where $W_0^j\coloneqq  2^{jd/2}\mathrm{I}_d$, $m\in \mathbb{Z}^d$ and where~\eqref{eq:KAN_product_represenation} holds by \cite[Lemma 1]{kratsios2025kolmogorov} (having chosen $M$ large enough); where $\times^2_d:\mathbb{R}^d\longrightarrow \mathbb{R}$ is a Res--KAN with depth $d$, width at-most $2d+1$, and at-most $\frac{5d^2 + 21d}{2} $ non-zero parameters.

\medskip
Now, making use of the chosen structure of the `non-spline' factor in our trainable activation function $\sigma_{\beta:I}$ in~\Cref{eq:KAN_r}, for each $i\in\{1,\dots,d\}$, if $G_i=S$ we set $\beta_i=(1)\oplus 0_{I+1}$ and if $G_i=W$ we set $\beta_i=(0)\oplus (1)\oplus 0_{I}$
Then,~\eqref{eq:KAN_product_represenation} can be re-expressed as
\begin{align*}
\numberthis
\label{eq:activation_choice}
    \Psi_{\smallertext{G},m}^j
& \coloneqq  
    \kappa_{\smallertext{G},\smallertext{W}}^j
    \times_d^2
    \circ 
    \sigma_{\smallertext{G}_\smalltext{i}}
    \big(
        W_0^j \cdot - m
    \big)
\end{align*}
Now by \cite[Lemma 1]{kratsios2025kolmogorov}, $\times_d^2$ can be implemented by a ReLU MLP of depth $d$, width $2d+1$, and using at-most $(5d^2 + 21d)/2 $ non-zero parameters.  Consequently, $\times^2_d$ is representable/implementable by a ReLU MLP with depth $d$, width at-most $2d+1$, and using at-most $(5d^2 + 25d + 2)/2$ non-zero parameters. 
\end{proof}

A direct consequence of the previous result is the following.
\begin{proposition}[Res--KAN basis of Besov spaces]
\label{prop:SchauderBasis}
Let $\mathcal{D}$ be a bounded exterior-thick domain, $(q,r)\in(1,\infty)^2$, and $s\ge 2$.  Then, there is a Schauder basis
\begin{equation}
\label{eq:ResKANBasis}
    \big\{ \widehat{\Phi}_k^j : j \in \mathbb{N},\;  k \in\{ 1, \ldots, N_j\} \big\},\; \text{\rm for some}\; N_j \in \overline{\mathbb{N}},
\end{equation}
of $\overline{B}_{q,r}^s(\mathcal{D})$ consisting of $u$-wavelets.  Moreover, for each such $k,j$, $\widehat{\Phi}_k^j$ is implementable by a {\rm Res--KAN} of depth $d$, width at-most $2d+1$, and using at-most $(5d^2 + 25d + 2)/2$ non-zero parameters.
\end{proposition}
\begin{proof}
This is a direct consequence of \Cref{lem:SchauderImplementation}, \Cref{def:u_wavelets}, and of \cite[Theorem 3.13 $ (ii)$]{triebel2008function}.
\end{proof}

We now prove the universality of our models in the class of H\"{o}lder continuous maps between Besov spaces; recall the notation~\eqref{eq:NO-union}.
We write $\operatorname{Hld}(\bar{B}^s_{q,r}(\mathcal{D}),\bar{B}^s_{q,r}(\mathcal{D}))$ for the set of all 
$\alpha$--H\"{o}lder continuous maps from $\bar{B}^s_{q,r}(\mathcal{D})$ to itself, for some $0<\alpha\le 1$.
\begin{proposition}[Universal approximation]
\label{prop:Universality}
Let $d\in \mathbb{N}_+$, $s>0$, and $\mathcal{D}$ be a bounded exterior-thick domain in $\mathbb{R}^d$, $(q,r)\in(1,\infty)^2$ and $2\le s$, and let $I \coloneqq \lceil s\rceil$.  
If $\sigma_\smallertext{S}$ and $\sigma_\smallertext{W}$ satisfy {\rm\Cref{ass:wavlets}}, then $\mathcal{NO}_{I,\alpha}$ is dense in $\operatorname{Hld}(\bar{B}^s_{q,r}(\mathcal{D}),\bar{B}^2_{q,r}(\mathcal{D}))$ for the $($relative$)$ topology induced by the topology of uniform convergence on compact sets.
\end{proposition}


\begin{proof}
Since $\mathcal{D}$ is exterior-thick, $s\ge 2$, $(q,r)\in(1,\infty)^2$, $\sigma_\smallertext{S}$ and $\sigma_\smallertext{W}$ satisfy \Cref{ass:wavlets}, and we set $I\coloneqq  \lceil s\rceil$ then, \Cref{prop:SchauderBasis} guarantees that we may exhibit a Schauder basis of $\overline{B}^s_{q,r}(\mathcal{D})$ consisting only of Res--KANs, as in~\eqref{eq:ResKANBasis}. 

\medskip
Pick an enumeration $\big(\widehat{\Psi}_{k_\smalltext{\ell}}^{j_\smalltext{\ell}}\big)_{\ell\in \mathbb{N}}$ thereof.  
Now, let $\mathfrak{F}$ consist of all functions $\widehat{F}:\bar{B}^s_{q,r}(\mathcal{D})\longrightarrow \bar{B}^2_{q,r}(\mathcal{D})$ of the form in~\citep[Equation 16]{galimberti2022designing} and \citep[Definition 6 (Neural filters)]{galimberti2022designing}
\begin{equation}
\label{eq:CNO}
    \widehat{F}
\coloneqq  
    \Big(
            \widehat{\Psi}^{j_\smalltext{1}}_{k_\smalltext{1}}
        ,
            \dots
        ,
            \widehat{\Psi}^{j_\smalltext{K}}_{k_\smalltext{K}}
    \Big)^{\top}
    \,
    \widehat{f}_{\smalltext{\rm ReLU}}
    \circ 
  \begin{pmatrix}
      \displaystyle  \int_{\mathbb{R}^\smalltext{d}}\,
            f(x)
            \widehat{\Psi}^{j_\smalltext{1}}_{k_\smalltext{1}}
        \mathrm{d}x\\
        \vdots\\
\displaystyle\int_{\mathbb{R}^\smalltext{d}}\,
            f(x)
            \widehat{\Psi}^{j_\smalltext{K}}_{k_\smalltext{K}}
        \mathrm{d}x
        \end{pmatrix}
\end{equation}
for some $K\in\N^\star$, and where $\widehat{f}_{\smallertext{\rm ReLU}}:\mathbb{R}^K\longrightarrow \mathbb{R}^K$ is a ReLU feed-forward neural network defined as iteratively mapping any $x\in \mathbb{R}^K$ to the vector $\hat{f}_{\operatorname{ReLU}}(x)\coloneqq x_{L+1}$ defined recursively by
\begin{equation}
\label{eq:rep_ReLU_setup}
\begin{aligned}
     x_{L+1} & \coloneqq W_{L+1} x_L \in \mathbb{R}^{d_{L+1}}\coloneqq \mathbb{R}^{d_K}
 \\
        x_{\ell+1}
    & \coloneqq
        \mathrm{ReLU}\big(W_{\ell}x_{\ell}+b_{\ell}\big)
    \in \mathbb{R}^{d_{\ell+1}},
        \; x\in \mathbb{R}^K,\; L\in \mathbb{N}_+,
        \,\,\mbox{ for } \ell\in \{0,\dots,L\}
\\ 
x_0 & \coloneqq x \in \mathbb{R}^{d_0}\coloneqq \mathbb{R}^{d_K}
        .  
        \end{aligned}
\end{equation}
where the layer widths are $(d_{0},\dots,d_{\smallertext{L}\smallertext{+}1})\in (\mathbb{N}_+)^{L\smallertext{+}2}$,
$K=d_0=d_{\smallertext{L}\smallertext{+}1}$, and for each such $\ell$, we have $W_{\ell}\in \mathbb{R}^{d_{\smalltext{\ell}\smalltext{+}\smalltext{1}}\times d_{\smalltext{\ell}}}$, as well as $b_{\ell}\in \mathbb{R}^{d_{\smalltext{\ell}\smalltext{+}\smalltext{1}}}$.  

\medskip
Since $\big(\widehat{\Psi}_{k_\smalltext{\ell}}^{j_\smalltext{\ell}}\big)_{\ell\in \mathbb{N}}$ is a Schauder basis of the Banach space $\bar{B}^s_{q,r}(\mathcal{D})$ and of $\bar{B}^2_{q,r}(\mathcal{D})$ then~\citep[Theorem 1]{galimberti2022designing} implies that $\mathfrak{F}$ is dense in $\operatorname{Hld}(\bar{B}^s_{q,r}(\mathcal{D}),\bar{B}^2_{q,r}(\mathcal{D}))$ for the (relative) topology induced by the topology of uniform convergence on compact sets. In other words, for every compact $\mathcal{K}\subseteq \bar{B}^s_{q,r}(\mathcal{D})$, every $\varepsilon>0$, and $0<\alpha\le 1$, and every $\alpha$--H\"{o}lder continuous map $f:\bar{B}^s_{q,r}(\mathcal{D})\longrightarrow \bar{B}^2_{q,r}(\mathcal{D})$, there is some $\widehat{F}\in \mathfrak{F}$ satisfying 
\begin{equation}
\label{eq:BesovUniformApproxiomation}
    \sup_{u\in \mathcal{K}}
        \|
            F(u)
            -
            \widehat{F}(u)
        \big\|_{W^{\smalltext{2}\smalltext{,}\smalltext{p}}(\mathcal{D})}
    <
        \varepsilon
.
\end{equation}
To deduce our claim, we will show that $\mathfrak{F}\subseteq \mathcal{NO}_{\smallertext{I},\alpha}$.  Let $\widehat{F}$ be an arbitrary element of $\mathfrak{F}$, which thus admits a representation as in~\eqref{eq:CNO}.  

\medskip
Now, for every $\ell \in \{0,\dots,L-1\}$, let $b^{\ell}(x)\coloneqq  \mathbf{0}_{(d+d_{\smalltext{\ell}\smalltext{+}\smalltext{1}})\times (d+d_{\smalltext{\ell}\smalltext{+}\smalltext{1}})}x +\mathbf{0}_{d}\oplus b_{\ell}$ be a constant Res--KAN, see \Cref{eq:DEF_SRKs}, where $\mathbf{0}_{(d+d_{\smalltext{\ell}\smalltext{+}\smalltext{1}})\times (d+d_{\smalltext{\ell}\smalltext{+}\smalltext{1}})}x$ is the $(d+d_{\smalltext{\ell}\smalltext{+}\smalltext{1}})\times (d+d_{\smalltext{\ell}\smalltext{+}\smalltext{1}})$ zero matrix and $\mathbf{0}_d\in \mathbb{R}^d$ is the zero vector therein.  Now, for every $\ell\in \{1,\dots,L-1\}$ define the matrix $W^{\ell}\coloneqq  \mathbf{0}_{d\times d}\otimes W_{\ell}$, where $\otimes$ denotes the Kronecker product and let $W^{L}\coloneqq  (0_{K\times d}|W_{L})$ denotes the column-wise concatenation of the matrix $0_{K\times d}$ with the matrix $W_{L}$
.
Now, for each $\ell\in\{1,\dots,L\}$ let $\beta_{\ell}\coloneqq(0,0,1,0,\dots,0)\in \mathbb{R}^{d_{\smalltext{\ell}\smalltext{+}\smalltext{1}}+2}$.  With these specifications, we see that the KANO $\Gamma$ with representation~\eqref{def:neural-operator} (where $d_{\rm in}=1$ and $d_{\rm out}=1$) is exactly equal to $\widehat{F}$.  We have thus shown that $\mathfrak{F}\subseteq \mathcal{NO}_{\smallertext{I},\alpha}$, which concludes our proof.
\end{proof}

\subsection{Stability estimate of general solution operator}
\label{s:StabilityEstimate}

\begin{lemma}[Linear stability of perturbations to PDE]
\label{lem:HolderRegularity}
Under {\rm\Cref{ass:Regularity_GeneralCase__DomainRegularity,ass:Regularity_GeneralCase__EllipticPDEFunctionalStructure}}, 
if $r>0$ and $k>1+\max\{1,d/p\}$ then there exists a constant $L_{2,k,\smallertext{\mathcal{D}}}>0$ such that the non-linear operator 
\begin{equation}
\label{eq:NonLinearPDEOperator}
\begin{aligned}
\Gamma_{\smallertext{\rm Gen}}:\mathcal{X}_k(r)
& 
\longrightarrow
W^2_p(\mathcal{D})
\\
(\bar{G}_0,g) & \longmapsto u_{\bar{\smallertext{G}}_\smalltext{0},g},
\end{aligned}
\end{equation}
is $L_{2,k,\smallertext{\mathcal{D}}}$--Lipschitz continuous.
\end{lemma}
\begin{proof}
Under~\Cref{ass:Regularity_GeneralCase__DomainRegularity,ass:Regularity_GeneralCase__EllipticPDEFunctionalStructure} we may apply~\cite[Theorem 14.1.3]{krylov2018sobolev} to deduce that for every $((\bar{G}_0,g),(\bar{G}_0^{\prime},g^{\prime}))\in \mathcal{X}\times\Xc$ and the respective solutions $u_{\bar{\smallertext{G}}_\smalltext{0},g},u_{\bar{\smallertext{G}}_\smalltext{0}^{\smalltext{\prime}},g^{\smalltext{\prime}}}$ (which exist by~\cite[Theorem 14.1.5]{krylov2018sobolev}) to their elliptic PDE in~\eqref{eq:EllipticPDE_FullyNonLinear} with $G+\bar{G}_0$ and $G+\bar{G}_0^{\prime}$ respectively instead of $G$, we have the estimate
\begin{equation}
\label{eq:Lipschitz_Continuity_Pt1__normbound}
    \|
        u_{\smallertext{\bar{G}}_\smalltext{0},g}
        -
        u_{\smallertext{\bar{G}}_\smalltext{0}^{\smalltext{\prime}},g^{\smalltext{\prime}}}
    \|_{W_\smalltext{p}^\smalltext{2}(\mathcal{D})} 
\lesssim
    \|\bar{G}_0-\bar{G}_0^{\prime}\|_{L^\smalltext{p}(\mathcal{D})}
    +
    \|g-g^{\prime}\|_{W^{\smalltext{2}\smalltext{,}\smalltext{p}}(\mathcal{D})}
    +
    \|u_{\smallertext{\bar{G}}_\smalltext{0},g}-u_{\smallertext{\bar{G}}_\smalltext{0}^{\smalltext{\prime}},g^{\smalltext{\prime}}}\|_{C(\mathcal{D})},
\end{equation}
where $\lesssim$ suppress a multiplicative constant depending only on $c_1$, $c_2$, $R_0$, $\delta$, $L_F$, $\omega_F$, and on the domain $\mathcal{D}$.  Next, applying~\citep[Lemma 6.6.10]{krylov2018sobolev} we deduce that there is an absolute constant $C>0$ such that $\|u_{\smallertext{\bar{G}}_\smalltext{0},g}-u_{\smallertext{\bar{G}}_\smalltext{0}^{\smalltext{\prime}},g^{\smalltext{\prime}}}\|_{C(\mathcal{D})}\le C \sup_{x\in \partial \mathcal{D}}|g(x)-{g}^\prime(x)|=\|g-g^\prime\|_{C(\partial \mathcal{D})}$.  Consequently,~\eqref{eq:Lipschitz_Continuity_Pt1__normbound} may be bounded above by
\begin{align*}
    \|
        u_{\smallertext{\bar{G}}_\smalltext{0},g}
        -
        u_{\smallertext{\bar{G}}_\smalltext{0}^{\smalltext{\prime}},g^{\smalltext{\prime}}}
    \|_{W_\smalltext{p}^\smalltext{2}(\mathcal{D})}
& \lesssim
    \|\bar{G}_0-\bar{G}_0^{\prime}\|_{L^\smalltext{p}(\mathcal{D})}
    +
    \|g-g^{\prime}\|_{W^{\smalltext{2}\smalltext{,}\smalltext{p}}(\mathcal{D})}
    +
        \|g-g^{\prime}\|_{C(\partial \mathcal{D})}
\\
& \le 
    \|\bar{G}_0-\bar{G}_0^{\prime}\|_{W^{\smalltext{2}\smalltext{,}\smalltext{p}}(\mathcal{D})}
    +
        \|g-g^{\prime}\|_{W^{\smalltext{2}\smalltext{,}\smalltext{p}}(\mathcal{D})}
    +
        \|g-g^{\prime}\|_{C(\partial \mathcal{D})}
\\
& \le 
        \|\bar{G}_0-\bar{G}_0^{\prime}\|_{W^{\smalltext{2}\smalltext{,}\smalltext{p}}(\mathcal{D})}
    +
        \|g-g^{\prime}\|_{W^{\smalltext{2}\smalltext{,}\smalltext{p}}(\mathcal{D})}
    +
        \|g-g^{\prime}\|_{C(\mathcal{D})}
\\
& \le 
        \|\bar{G}_0-\bar{G}_0^{\prime}\|_{W^{\smalltext{2}\smalltext{,}\smalltext{p}}(\mathcal{D})}
    +
        \|g-g^{\prime}\|_{W^{\smalltext{2}\smalltext{,}\smalltext{p}}(\mathcal{D})}
    +
        \|g-g^{\prime}\|_{W^{\smalltext{k}\smalltext{,}\smalltext{p}}(\mathcal{D})}
\\
& \le 
        \widetilde{C}_{2,k,\smallertext{\mathcal{D}}}\,
        \|\bar{G}_0-\bar{G}_0^{\prime}\|_{W^{\smalltext{k}\smalltext{,}\smalltext{p}}(\mathcal{D})}
    +
        \widetilde{C}_{2,k,\smallertext{\mathcal{D}}}\,
        \|g-g^{\prime}\|_{W^{\smalltext{k}\smalltext{,}\smalltext{p}}(\mathcal{D})}
    +
        \|g-g^{\prime}\|_{W^{\smalltext{k}\smalltext{,}\smalltext{p}}(\mathcal{D})}
\\
& \le 
        L_{2,k,\smallertext{\mathcal{D}}}\,
        \Big(
            \|\bar{G}_0-\bar{G}_0^{\prime}\|_{W^{\smalltext{k}\smalltext{,}\smalltext{p}}(\mathcal{D})}
        +
            \|g-g^{\prime}\|_{W^{\smalltext{k}\smalltext{,}\smalltext{p}}(\mathcal{D})}
        \Big),
\end{align*}
where we used in the fourth line the Sobolev embedding Theorem~\cite[Section 5.6.3]{evans2010partial}, which holds provided that $k \le 1+\lceil \tfrac{d}{p} \rceil$, where the existence of the constant $\widetilde{C}_{2,k,\smallertext{\mathcal{D}}}>0$ (which only depends on $2$, $k$, and on $\mathcal{D}$) as well as the validity of the fifth line are ensured since we have assumed that $2<k$ so that the Rellich–-Kondrachov Theorem~\cite[Proposition 4.4]{taylor2023partial} implies that $W^{2,p}(\mathcal{D})$ is compactly embedded in $W^{k,p}(\mathcal{D})$, and $C := 2\widetilde{C}_{2,k,\smallertext{\mathcal{D}}}+1>1$.    
\end{proof}

We are now ready to establish our approximability result for the solution operator corresponding to the more general class of fully non-linear elliptic PDEs.
\begin{proof}[{Proof of Theorem~\ref{thrm:generalapprox}}]
Under~\Cref{ass:Regularity_GeneralCase__DomainRegularity,ass:Regularity_GeneralCase__EllipticPDEFunctionalStructure}, \Cref{lem:HolderRegularity} applies and guarantees that the non-linear operator $\Gamma_{\smallertext{\rm Gen}}$, defined in~\eqref{eq:NonLinearPDEOperator}, is $L_{2,k,\smallertext{\mathcal{D}}}$--Lipschitz continuous on $\mathcal{X}_k(r)$.  Now, since $2<k<\infty$ and $\sigma_\smallertext{S}$ and $\sigma_\smallertext{W}$ satisfy \Cref{ass:wavlets}, we may apply \Cref{prop:Universality} to deduce that for every $\varepsilon>0$ and every non-empty compact subset $\mathcal{X}\subseteq \mathcal{X}_k(r)$%
(in the relative topology induced by inclusion in $W^{2,p}(\mathcal{D})\times W^{k,p}(\mathcal{D}))$ equipped with the norm topology)~%
there exists $\hat{\Gamma}\in \mathcal{NO}_{\lceil k\rceil,1}$ satisfying the uniform estimate
\begin{equation}
\label{eq:uniform_estimate__inproof}
    \sup_{(\bar{\smallertext{G}}_\smalltext{0},g)\in\mathcal{X}}\,
    \big\|
            \Gamma_{\smallertext{\rm Gen}}(\bar{G}_0,g)
        -
            \hat{\Gamma}(\bar{G}_0,g)
    \big\|_{W^{\smalltext{2}\smalltext{,}\smalltext{p}}(\mathcal{D})}
<
    \varepsilon
.
\end{equation}
Noting that, by definition, $u_{\smallertext{\bar{G}}_\smalltext{0},g}=\Gamma_{\smallertext{\rm Gen}}(\bar{G}_0,g)$ for each $(\bar{G}_0,g)\in\mathcal{X}
$ concludes the proof.
\end{proof}

\section{Proof of stochastic results}
\label{s:Proofs__StochResults}
To derive the stochastic counterparts of our results, we emphasise that our approach does not rely on any unconventional lifting channels---such as those introduced in~\cite{furuya2024simultaneously}---which are non-standard within the operator learning literature and were originally proposed to enforce additional smoothness.  
Instead, we are able to combine the Bernstein and Sobolev inequalities with It\^{o}-type formulas in a compatible manner, without imposing excessive smoothness assumptions on the PDE solutions. This is achieved through the following transfer principle, which requires conditions we borrow from \citeauthor*{demarco2011smoothness} \cite{demarco2011smoothness}.
\begin{assumption}[Regularity of the forward process]
\label{ass:smooth_ass_density}
\begin{enumerate}
\item[$(i)$] there is $\eta\ge 0$ such that $\mu$ and $\gamma$ in~\eqref{eq:FBSDE_ForwardProcess} are of class $C^{\infty}$ on $\R^d\setminus \overline{B_{\mathbb{R}^d}(0,\eta)}$.
Moreover, for every $R>0$ and $x_0\in\R^d$, $\mu$ and $\gamma$ are smooth on $B_{\mathbb{R}^d}(x_0,3R)\subset \R^d\setminus \overline{B_{\mathbb{R}^d}(0,\eta)};$

\item[$(ii)$]
there exist positive exponents $q$ and $\bar{q}>0$, as well as constants 
$0<C_0<1$, $C_k>0$ (for every multi–index $\alpha$ with $|\alpha|=k\ge 1$) such that
\begin{align}
\label{eq:Growth}
|\partial_{\alpha}\mu^i(x)| + |\partial_{\alpha}\gamma^{i,j}(x)|
& \le 
C_k(1+\|x\|^q),
\; x\in\R^d,\; (i,j)\in \{1,\dots,d\}^2,
\\
\label{eq:Ellipticity}
C_0\|x\|^{-\bar{q}}\mathrm{I}_d
& \le 
\gamma(x)\gamma(x)^{\top},\; \|x\|>\eta;
\end{align}
\item[$(iii)$] for every $p>0$, $\sup_{0\le s\le t}\E^\P[\|X_s\|^p]<\infty;$

\item[$(iv)$] \eqref{eq:FBSDE_ForwardProcess} admits a strong solution.
\end{enumerate}
\end{assumption}
Under these conditions, the process $X$ admits for every $t\in(0,T]$ 
a smooth density satisfying some
Gaussian-type decay and derivative bounds, as shown in~\citep[Theorem 2.2]{demarco2011smoothness}.  
In what follows, if it exists, for any time $t\ge 0$, we denote the density of the law $X_t$ with respect to the Lebesgue measure on $B_\smallertext{R}(y_0)$, for any $y_0\in \mathcal{D}$ and $R>0$, by $\rho_{t,y_\smalltext{0}} \in L^1(B_\smallertext{R}(y_0); [0,\infty))$, where 
\[
L^1(B_\smallertext{R}(y_0); [0,\infty)) 
\coloneqq 
\big\{u\in L^1(B_\smallertext{R}(y_0)):u(x)\ge 0,\; \text{\rm Lebesgue--a.e.}\}
.
\]

\begin{lemma}[Transfer trick]
\label{lem:transfer_lemma}
Let $1\le s<\infty$, $1\le r\le \infty$, $x_0\in \mathcal{D}$ be such that $\mathcal{D}\subseteq B_R(x_0)$ be a compact domain, and $(u,\hat{u})\in W^{s,r}(\mathcal{D})\times W^{s,r}(\mathcal{D})$ be such that
\begin{equation}
\label{eq:SobolevBound}
        \|u-\hat{u}\|_{W^{\smalltext{s}\smalltext{,}\smalltext{r}}(\mathcal{D})}
    \le 
        \varepsilon
.
\end{equation}
Suppose that $X$ satisfies~\eqref{eq:FBSDE_ForwardProcess} and {\rm\Cref{ass:smooth_ass_density}} and $\tau$ is the first exit time of $X$ from $\Dc$.  
If $r$ is finite, then additionally assume that there is some $0<\delta_{\smallertext{\mathcal{D}}}$ such that $d(0,\mathcal{D})\coloneqq\inf_{x\in \mathcal{D}} \|x\|_2\ge \delta_{\smallertext{\mathcal{D}}}$ and fix a time-window $0<T_{\smallertext{-}}<T_{\smallertext{+}}$.  
Then
\begin{equation}
\label{eq:Lp_bound}
        \begin{aligned}
            \mathbb{E}^\P\Bigg[
                \int_{T_\smalltext{-}}^{T_\smalltext{+}}
                    \sum_{|\beta|\le s}
                    \big\|
                        D^{\beta}u(X_t)
                        -
                        D^{\beta}\hat{u}(X_t)
                    \big\|
                \mathrm{d}t
            \Bigg]
&\lesssim_{r,\smallertext{T}_{\smalltext{+}},\smallertext{\mathcal{D}}}
        \varepsilon
        \bigg(
            C_{\smallertext{T}_\smalltext{+}}
        +
            \frac{1}{
                T_{\smallertext{-}}^{3d/2-1}}
        \bigg),\; \mbox{\rm if}\; 1\le r < \infty,
        \\[0.5em]
            \mathrm{essup}^{\mathbb{P}}
               \bigg\{ \sup_{0\le t\le \tau}
                \big\|
                    D^{\beta} 
                    u(X_t(\omega))
                    -
                    D^{\beta}\, \hat{u}(X_t(\omega))
                \big\|\bigg\}
            &\le \varepsilon,\;
             \mbox{\rm if}\; r=\infty,
        \end{aligned}
\end{equation}
where $C_{\smallertext{T}_\smalltext{+}}>0$ is a constant depending only on $T_\smallertext{+}$.
\end{lemma}
\begin{proof}
For the case where $r=\infty$, simply note that $X_{t\vee \tau}\in \mathcal{D}$. $\mathbb{P}$--a.s.  Thus, for $\P$--almost every $\omega\in \Omega$ we have that
\[
    \sum_{|\beta|\le s}
        \big\|
            D^{\beta}
            u(X_t(\omega))
            -
            D^{\beta} \hat{u}(X_t(\omega))
        \big\|
    \le 
        \sup_{x\in \mathcal{D}} 
            \big\|D^{\beta} (u-\hat{u})(x)\big\|
    =
        \|u-\hat{u}\|_{W^{\smalltext{s}\smalltext{,}\smalltext{r}}(\mathcal{D})}
    \le 
        \varepsilon,
\]
where the last inequality holds since $s\ge 1$.  Consequently,~\eqref{eq:Lp_bound} holds.  

\medskip
We now turn our attention to the case where $1\le r<\infty$.  Define $\tau^{\star} \coloneqq T_\smallertext{+}\wedge(\tau\vee T_\smallertext{-})$.  Note that, if $t\in [T_\smallertext{-},T_\smallertext{+}]$ then $X_{t\wedge \tau^{\smalltext{\star}}}\in \bar{\mathcal{D}}$, $\P$--a.s.
In particular, since $\mathcal{D}$ is bounded, then for any $t\ge 0$, $X_{t\wedge \tau^{\smalltext{\star}}}\in L^{\infty}([0,T_\smallertext{+}]\times \Omega,\mathbb{R}^d)$; whence, we may apply the Fubini--Tonelli theorem to deduce that
\allowdisplaybreaks
\begin{align}
\label{eq:Fubiniboy}
        \mathbb{E}^\P\Bigg[
            \int_{T_\smallertext{-}}^{T_\smallertext{+}}
                \sum_{|\beta|\le s}
                \big\|
                    D^{\beta}u(X_t)
                    -
                    D^{\beta}\hat{u}(X_t)
                \big\|
            \mathrm{d}t
        \Bigg]
&=
    \int_{T_\smallertext{-}}^{T_\smallertext{+}}
        \mathbb{E}^\P\Bigg[
            \sum_{|\beta|\le s}
            \big\|
                D^{\beta}u(X_t)
                -
                D^{\beta}\hat{u}(X_t)
            \big\|
        \Bigg]  
    \mathrm{d}t
.
\end{align}
Now, since we are operating under \Cref{ass:smooth_ass_density}, we may apply~\citep[Theorem 2.2]{demarco2011smoothness} to show that $\rho_{t,x_\smalltext{0}}\in L^1_\smallertext{+}(B_R(x_0))$ exists and there is a constant $C_{r,T_\smallertext{+}}>0$, depending only on $r$ and $T_\smallertext{+}$, such that for every $x\in B_R(x_0)$ we have
\begin{equation}
\label{eq:density_bound}
    |\rho_{t,x_\smalltext{0}}(x)|
\le 
    C_{r,T_\smallertext{+}}
    \bigg(
            1
        +
            \frac{1}{t^{3d/2}}
    \bigg)\|x\|^{\smallertext{-}r}
.
\end{equation}
In particular, since $\mathcal{D}\subseteq B_R(x_0)$ then~\eqref{eq:density_bound} holds for every $x\in \mathcal{D}$.  Consequently,~\eqref{eq:Fubiniboy} and~\eqref{eq:density_bound} imply that
\allowdisplaybreaks
\begin{align*}
\nonumber
        \mathbb{E}^\P\Bigg[
            \int_{T_\smallertext{-}}^{T_\smallertext{+}}
                \sum_{|\beta|\le s}
                \big\|
                    D^{\beta}u(X_t)
                    -
                    D^{\beta}\hat{u}(X_t)
                \big\|
            \mathrm{d}t
        \Bigg]
&=
    \int_{T_\smallertext{-}}^{T_\smallertext{+}}
        \int_{\mathcal{D}}
            p_{t,x_\smalltext{0}}(x)
                \sum_{|\beta|\le s}
                \big\|
                    D^{\beta}u(x)
                    -
                    D^{\beta}\hat{u}(x)
                \big\|
            \mathrm{d}x
    \mathrm{d}t
\\
&
\le 
    \int_{T_\smallertext{-}}^{T_\smallertext{+}}
        \bigg(
        \int_{\mathcal{D}}
            p_{t,x_\smalltext{0}}(x)^{r^\smalltext{\prime}}
        \mathrm{d}x\bigg)^{1/r^\smalltext{\prime}}
\\
\nonumber
& \quad\times
        \Bigg(
            \int_{\mathcal{D}}
                \sum_{|\beta|\le s}
                \big\|
                    D^{\beta}u(x)
                    -
                    D^{\beta}\hat{u}(x)
                \big\|^r
            \mathrm{d}x
        \Bigg)^{1/r}
    \mathrm{d}t
\\
\nonumber
&
\le 
    \int_{T_\smallertext{-}}^{T_\smallertext{+}}
\bigg(
        \int_{\mathcal{D}}
            C_{r,\smallertext{T}_\smalltext{+}}^{r^\smalltext{\prime}}
            \bigg(
                    1
                +
                    \frac{1}{t^{3d/2}}
            \bigg)^{r^\smalltext{\prime}}
            \|x\|^{\smallertext{-}(rr^\smalltext{\prime})}
        \mathrm{d}x\bigg)^{1/r^\smalltext{\prime}}
\\
\nonumber
&\quad \times
        \Bigg(
            \int_{\mathcal{D}}
                \sum_{|\beta|\le s}
                \big\|
                    D^{\beta}u(x)
                    -
                    D^{\beta}\hat{u}(x)
                \big\|^{r}
            \mathrm{d}x
    \Bigg)^{1/r}
    \mathrm{d}t,
 \end{align*}
where the second line follows by H\"{o}lder's inequality with $\tfrac{1}{r}+\tfrac{1}{r^\prime}=1$ (since $1<r<\infty$).  
Now, the term 
\[\Bigg(
            \int_{\mathcal{D}}
                \sum_{|\beta|\le s}
                \big\|
                    D^{\beta}u(x)
                    -
                    D^{\beta}\hat{u}(x)
                \big\|^r
            \mathrm{d}x
        \Bigg)^{1/r},
\]
is precisely the $W^{\lfloor s \rfloor ,r}(\mathcal{D})$ norm of $(u-\hat{u})$, which is bounded above by the $W^{s,r}(\mathcal{D})$-norm, which in turn is bounded above by $\varepsilon$, recall~\eqref{eq:SobolevBound}.  Hence
\allowdisplaybreaks
\begin{align*}
        \mathbb{E}^\P\Bigg[
            \int_{T_\smallertext{-}}^{T_\smallertext{+}}
                \sum_{|\beta|\le s}
                \big\|
                    D^{\beta}u(X_t)
                    -
                    D^{\beta}\hat{u}(X_t)
                \big\|
            \mathrm{d}t
        \Bigg]
&
\le 
    \varepsilon
    \int_{T_\smallertext{-}}^{T_\smallertext{+}}
        \bigg(
            \int_{\mathcal{D}}
                C_{r,T_\smallertext{+}}^{r^\smalltext{\prime}}
                \bigg(
                        1
                    +
                        \frac{1}{t^{3d/2}}
                \bigg)^{r^\smalltext{\prime}}
                \,
                \|x\|^{\smallertext{-}(rr^\smalltext{\prime})}
            \mathrm{d}x
        \bigg)^{1/r^\smalltext{\prime}}
    \mathrm{d}t
\\
&
\le
    C_{r,T_\smallertext{+}}
    \varepsilon
        \frac{
            \mathrm{Vol}(\mathcal{D})^{1/r^\smalltext{\prime}}
        }{\delta_{\smallertext{\mathcal{D}}}^r}
    \int_{T_\smallertext{-}}^{T_\smallertext{+}}
                \bigg(
                        1
                    +
                        \frac{1}{t^{3d/2}}
                \bigg)
    \mathrm{d}t
\\
&
\le
    C_{r,T_\smallertext{+}}
    \varepsilon
        \frac{
            \mathrm{Vol}(\mathcal{D})^{1/r^\smalltext{\prime}}
        }{\delta_{\smallertext{\mathcal{D}}}^r}
   \bigg(
        T_\smallertext{+}-T_\smallertext{-}
    +
        \frac{
                T_\smallertext{-}^{1-3d/2}
                -
                T_\smallertext{+}^{1-3d/2}
        }{
            3d/2-1}\bigg)
            \\
&\leq  \varepsilon
{C}_{p,T_\smallertext{+},\mathcal{D}}
\bigg(
    C_{T_\smallertext{+}}
+
    \frac{1}{
        T_\smallertext{-}^{3d/2-1}}\bigg)
\end{align*}
where we used the assumption that $d(\mathcal{D},0)\ge \delta_{\smallertext{\mathcal{D}}}>0$ and a simple supremum-bound, and where we defined 
\[
{C}_{p,T_\smallertext{+},\mathcal{D}}
\coloneqq
C_{p,T_\smallertext{+}} 
        \frac{2
            \mathrm{Vol}(\mathcal{D})^{1/r^\smalltext{\prime}}
        }{(3d-2)\delta_{\smallertext{\mathcal{D}}}^r},\; \text{\rm and}\; C_{T_\smallertext{+}}\coloneqq
        \bigg(\frac{3d}2-1\bigg)T_\smallertext{+}.
        \]

\end{proof}

\section{Experimental details} \label{sec:experimental_details}

\subsection{Periodic semi-linear case} \label{sec:periodic_semilinear}

We consider a periodic example from \cite{chassagneux2023learning} in $d = 5$ dimension, with $T=1$, in which the forward SDE is given by 
\begin{align*}
\mathrm dX_t^{(i)} &= b_i\!\big(X_t^{(i)}\big)\,\mathrm dt
+ \sigma_{i,i}\!\big(X_t^{(i)}\big)\,\mathrm dW_t^{(i)},
\; i\in\{1,\dots,d\},
\end{align*}
and the coefficients of the SDE are given by
\begin{align*}
b_i(x) &\coloneqq 0.2\,\sin(2\pi x_i), \;
\sigma_{i,j}(x) \coloneqq \frac{1}{\sqrt{d}\,\pi}\Big(0.25 + 0.1\cos(2\pi x_i)\Big)\,\mathbf{1}_{\{i=j\}},\; (i,j)\in\{1,\dots,d\}^2.
\end{align*}
The coefficients of the backward SDE 
\begin{align*}
\mathrm dY_t &= -\,f\!\big(t,X_t,Y_t,Z_t\big)\,\mathrm dt \;+\; Z_t\cdot \mathrm dW_t,
\; Y_T = g\!\big(X_T\big),
\end{align*}

are given by
\begin{gather*}
g(x) \coloneqq \frac{1}{\pi}\Bigg(\sin\bigg(2\pi \sum_{i=1}^d x_i\bigg) + \cos\bigg(2\pi \sum_{i=1}^d x_i\bigg)\Bigg), \\[4pt]
f(t,x,y,z) \coloneqq 2\pi^2y \sum_{i=1}^d \sigma_{i,i}(x)^2
- \sum_{i=1}^d \frac{b_i(x)}{\sigma_{i,i}(x)} z_i
+ h(t,x),
\end{gather*}
where
\begin{equation*}
h(t,x) \coloneqq 2\Bigg(\cos\bigg(2\pi \sum_{i=1}^d x_i + 2\pi (T-t)\bigg)
-\sin\bigg(2\pi \sum_{i=1}^d x_i + 2\pi (T-t)\bigg)\Bigg).
\end{equation*}

The explicit solution $u$ is given by 
\begin{equation*}
u(t,x) = \frac{1}{\pi}\big(\sin (\theta(t,x)) + \cos(\theta(t,x)) \big),
\end{equation*}
where
\begin{equation*}
\theta(t,x) \coloneqq 2\pi\Bigg(\sum_{i=1}^d x_i + (T - t)\Bigg).
\end{equation*}
The spatial derivatives of $u$ are given by 
\begin{equation*}
\frac{\partial u}{\partial x_i}(t,x) 
= 2\big(\cos(\theta(t,x)) - \sin(\theta(t,x)) \big), \; i\in\{1,\dots,d\},
\end{equation*}
and
\begin{equation*}
\frac{\partial^2 u}{\partial x_i \partial x_j}(t,x)
= -4\pi\big(\sin(\theta(t,x)) + \cos(\theta(t,x))\big),\; (i,j)\in\{1,\dots,d\}^2 .
\end{equation*}

\subsection{Linear--quadratic (LQ) case} \label{sec:linear_quadratic}
We consider a linear--quadratic case from \cite{pham2021neural} in $d = 5$ dimension, with $T=1$. The forward SDE is a controlled process $X_t$ in $\mathbb{R}^d$, defined by
\[
\mathrm{d}X_t = (A X_t + B \alpha_t)\mathrm{d}t + D \alpha_t \mathrm{d}W_t,
\]
where $\alpha_t$ is a control process in $\mathbb{R}$, $(B, D) \in \mathbb{R}^d\times\R^d$ and $A \in \mathbb{R}^{d\times d}$. The quadratic cost that is minimised is
\begin{equation*}
J(\alpha) \coloneqq \mathbb{E} \\bigg[ \int_{0}^{T} \big( X_t^{\top} Q X_t + \alpha_t^{2} N \big) \mathrm{d}t + X_T^{\top} P X_T \bigg],
\end{equation*}
where $P$ and $Q$ are non-negative, symmetric $d\times d$ matrices and $N>0$.

\medskip
The Bellman PDE associated with this process admits an explicit solution given by a quadratic form
\[
u(t,x) = x^T K(t) x,
\]
where $K(t)$ solves the Ricatti equation
\[
\dot{K} + A^{\top}K + KA + Q - \frac{KBB^{\top}K}{N + D^{\top}KD} = 0, \; K(T) = P.
\]
In all the simulations, we set
\[
A = \mathrm{I}_d, \; B = D = \mathrm{I}_d, \; Q = P = \frac{1}{d}\mathrm{I}_d, \; N = d.
\]
The stochastic coefficients associated to the controlled process are set to
\[
\sigma = \frac{1}{\sqrt{d}}\mathrm{I}_d,\; \textnormal{and} \; \mu(t,x) = x. 
\]
In our isotropic setup, the Riccati matrix remains proportional to the identity, \emph{i.e.}
\[
K(t) = k(t)\mathrm{I}_d.
\]
Then, the explicit forms of the spatial derivatives of $u$ are given by
\[
\nabla_x u(t,x)=2K(t)x = 2k(t)x,\; D_x^2 u(t,x)=2K(t)=2k(t)\mathrm{I}_d.
\]
To compute the solution $u$ and its derivatives, we employ a fourth-order Runge-–Kutta (RK4) scheme to numerically approximate $K(t)$ (the solution of the Riccati equation).

\subsection{Architectural details} \label{sec:architecture}
The KANO architecture follows a lift-–process–=project design. The input features are first lifted to a higher-dimensional latent space using a feed-forward network, producing an initial latent representation \(v^{(0)}\). 

\medskip
After lifting, a composition of several KANO blocks is applied to iteratively refine this latent field:
\[
v^{(\ell+1)} = \Phi^{(\ell)}(v^{(\ell)}, x), \; \ell \in\{ 0, \dots, L-1\},
\]
where each block \(\Phi^{(\ell)}\) performs a structured operator update combining coordinate encoding, spectral convolution, and residual connection. Each KANO block consists of three main components 
\begin{enumerate}
    \item  a \textbf{positional encoder} maps the spatial coordinates through a Res--KAN network, producing coordinate-dependent features
\[
v_{\mathrm{pos}} = b(x);
\]
\item a \textbf{spectral kernel path} performs a spectral convolution in the frequency domain, analogous to the Fourier neural operator (FNO) \cite{li2021fourier}. Specifically, the feature field is transformed via a two-dimensional fast Fourier transform (FFT), filtered by learnable complex-valued multipliers, and then mapped back to the spatial domain
\[
v_{\mathrm{kf}}(x) = \mathcal{F}^{-1}\big(\hat{W}(k)\mathcal{F}[v_{\mathrm{in}}](k)\big),
\]
where \(\mathcal{F}\) and \(\mathcal{F}^{-1}\) denote the forward and inverse Fourier transforms, and \(\hat{W}(k)\) are learnable complex weights restricted to a finite number of Fourier modes and parametrised as Res--KANs;

\item  a \textbf{residual path} applies a Res--KAN transformation on the tensor obtained by concatenating $\big(v_{\mathrm{pos}}, v_{\mathrm{kf}},  v_{\mathrm{in}}\big)$.
\end{enumerate}

After stacking \(L\) such KANO blocks, the resulting field \(v^{(L)}\) is projected back to the target dimension through a final projection layer. This composition enables multiscale feature extraction, efficient global coupling through spectral convolution, and local adaptivity through Res--KAN-based non-linear transformations.

\medskip
We restrict our training to a $2$D uniform grid that spans the first two coordinates of the $d$-dimensional space, while conditioning the model pointwise on the remaining $d - 2$ coordinates. The procedure for generating random training samples is described in detail in \Cref{sec:training_pipeline}. Our model is trained to approximate 2D slices of the solution along the $(x_1, x_2)$-coordinates in $\mathbb{R_+}\times\mathbb{R}^d$. Once trained, the model can be evaluated at any point in time and space by approximating the solution over these 2D slices and querying the corresponding $(x_1, x_2)$ values (see \Cref{sec:inference_pipeline} for details). This type of restricted operator learning is efficient due to the following reasons.

\begin{itemize}
\item \textbf{Uniform grids enable efficient kernels.} During training, the coordinates $(x_1, x_2)$ are placed on a uniform grid, enabling convolution-like kernel layers to be computed efficiently via FFTs. This reduces the per-layer complexity from dense $O(s^4)$ to $O(s^2 \log s)$, making spectral kernels both computationally efficient and numerically stable.

\item \textbf{Learning high-dimensional maps through 2D evaluations.}
The operator is evaluated over the full $s^2$ grid simultaneously, while the remaining coordinates $(x_3, \dots, x_d)$ and time $t$ are provided as additional input channels. This setup allows the network to capture intrinsic symmetries in the problem and to perform {restricted operator learning}, approximating $u(t, x)$ across $\mathbb{R}^d$ by predicting values at multiple $2$D locations in parallel.

\item \textbf{2D offers the optimal balance; 3D becomes costly.}
Extending the FFT-based grid to three dimensions increases computational and memory demands to $O(s^3 \log s)$ per pass and substantially raises activation and storage costs. In practice, $2$D grids strike the best balance between expressivity (capturing many spatial query points per sample) and efficiency, while still encoding $d$-dimensional dependencies through the auxiliary input channels.
\end{itemize}

Note that spectral convolution on uniform grids is employed to improve the training efficiency of the model. In operator learning settings, various efficient kernel architectures exist, see \citeauthor*{kovachki2023neural} \cite{kovachki2023neural}, including convolution-based kernels, see \citeauthor*{raonic2023convolutional} \cite{raonic2023convolutional}, wavelet-based kernels, see \citeauthor*{tripura2022wavelet} \cite{tripura2022wavelet}, and transformer-based kernels, see \citeauthor*{herde2024poseidon} \cite{herde2024poseidon} or \citeauthor*{li2023transformer} \cite{li2023transformer}, among others. The choice of the spectral kernel here is made solely to demonstrate that training a neural operator in the 2BSDE setting is feasible. 

\subsection{Training pipeline} \label{sec:training_pipeline}

In all our experiments, we draw samples from the domain uniformly. To draw a random training sample, we first draw a random time, as well as random locations for the $d-2$ dimensions (the first 2 dimensions $(x1, x2)$ are already sampled on uniform grids),
\[
t\in[0,T],\;
c=(x_3,\dots,x_d)\in[0,1)^{d-2}.
\]
To get the training samples, we evaluate the model on a \emph{uniform} $s\times s$ grid for the first two coordinates
\[
\mathcal G \coloneqq\bigg\{(x_1^p,x_2^q): x_1^p=\frac{p}{s-1},\; x_2^q=\frac{q}{s-1},\; (p,q)\in\{0,\dots,s-1\}^2\bigg\},
\]
and denote $N\coloneqq s^2$ and $X\coloneqq\big((x_{1n},x_{2n})\big)_{n\in\{1,\dots,N\}}$ the grid.

\medskip
At each grid node $n$, the model receives the feature vector
\[
\phi_n \coloneqq
\big(t, \,X,\, x_3,\dots,x_d\big)\in\mathbb{R}^{1+2+(d-2)}=\mathbb{R}^{d+1},
\]
\emph{i.e.} time and the \((d-2)\) extra coordinates are \emph{channels} constant across the 2d grid. A neural operator $F_\theta$ maps these inputs to the $\mathbb{R}^{s\times s}$ field,
\[
\hat u_\theta\big(t, X, x_3,\dots,x_d\big) = F_\theta\big(\phi_n\big)\in\mathbb{R}^{s\times s}.
\]

\subsection{Inference pipeline} \label{sec:inference_pipeline}

At test time, the learned approximation $\hat u_\theta$ can be evaluated at any query $(t,x)$ in the domain by either of the following.

\begin{itemize}
  \item \emph{Spectral/Fourier synthesis.} 
  If the decoder is spectral, we evaluate the Fourier--like synthesis operator at the desired coordinates to obtain $\hat u_\theta(t,x)$ directly. 
  This is naturally suited to periodic problems and preserves differentiability with respect to $(t,x)$, enabling gradients to be obtained by automatic differentiation.

  \item \emph{Grid interpolation.}
  When the model outputs values on a uniform $s\times s$ grid in $(x_1,x_2)$ at a given time $t$, we interpolate that grid to any $(x_1,x_2)$ in the domain (\emph{e.g.} bilinear/bicubic interpolation). 
  This route is simple, fast, and it requires no change to the trained model.
\end{itemize}

To evaluate the models along random paths, we generate $d$-dimensional SDE trajectories using the Euler--Maruyama scheme,
\[
X_{n+1}^{(i)} = X_n^{(i)} + b_i(X_n^{(i)})\Delta t
+ \sigma_{i,i}(X_n^{(i)})\sqrt{\Delta t}\xi_n^{(i)},
\; \xi_n^{(i)} \sim \mathcal{N}(0,1).
\]
The trained model is then evaluated along these trajectories, and its predictions are compared against the exact solution $u$ and its first- and second-order partial derivatives. Derivatives of the neural operator are approximated using first-order finite difference scheme. To obtain model outputs at arbitrary spatial locations, we employ bilinear interpolation over the $(x_1, x_2)$ grid.


\bibliography{bibliographyDylan}

\begin{thebibliography}{97}
\providecommand{\natexlab}[1]{#1}
\providecommand{\url}[1]{\texttt{#1}}
\expandafter\ifx\csname urlstyle\endcsname\relax
  \providecommand{\doi}[1]{doi: #1}\else
  \providecommand{\doi}{doi: \begingroup \urlstyle{rm}\Url}\fi

\bibitem[Acciaio et~al.(2024)Acciaio, Kratsios, and Pammer]{acciaio2024designing}
B.~Acciaio, A.~Kratsios, and G.~Pammer.
\newblock Designing universal causal deep learning models: the geometric (hyper) transformer.
\newblock \emph{Mathematical Finance}, 34\penalty0 (2):\penalty0 671--735, 2024.

\bibitem[Adcock et~al.(2022)Adcock, Brugiapaglia, Dexter, and Moraga]{adcock2022efficient}
B.~Adcock, S.~Brugiapaglia, N.~Dexter, and S.~Moraga.
\newblock On efficient algorithms for computing near-best polynomial approximations to high-dimensional, {H}ilbert-valued functions from limited samples.
\newblock \emph{ArXiv preprint arXiv:2203.13908}, 2022.

\bibitem[Adcock et~al.(2024)Adcock, Dexter, and Moraga~Scheuermann]{adcock2024optimal}
B.~Adcock, N.~Dexter, and S.~Moraga~Scheuermann.
\newblock Optimal deep learning of holomorphic operators between banach spaces.
\newblock In A.~Globerson, L.~Mackey, D.~Belgrave, A.~Fan, U.~Paquet, J.~Tomczak, and C.~Zhang, editors, \emph{Proceedings of the 38th conference on advances in neural information processing systems $($NeurIPS 2024$)$, December 10--15, 2024, Vancouver, British Columbia, Canada}, volume~37, pages 27725--27789, 2024.

\bibitem[Adcock et~al.(2025)Adcock, Brugiapaglia, Dexter, and Moraga]{adcock2025near}
B.~Adcock, S.~Brugiapaglia, N.~Dexter, and S.~Moraga.
\newblock Near-optimal learning of {B}anach-valued, high-dimensional functions via deep neural networks.
\newblock \emph{Neural Networks}, 181\penalty0 (106761):\penalty0 1--25, 2025.

\bibitem[Alvarez et~al.(2024)Alvarez, Ekren, Kratsios, and Yang]{alvarez2024neural}
G.~Alvarez, I.~Ekren, A.~Kratsios, and X.~Yang.
\newblock Neural operators can play dynamic {S}tackelberg games.
\newblock \emph{ArXiv preprint arXiv:2411.09644}, 2024.

\bibitem[Arabpour et~al.(2024)Arabpour, Armstrong, Galimberti, Kratsios, and Livieri]{arabpour2024low}
R.~Arabpour, J.~Armstrong, L.~Galimberti, A.~Kratsios, and G.~Livieri.
\newblock Low-dimensional approximations of the conditional law of {V}olterra processes: a non-positive curvature approach.
\newblock \emph{ArXiv preprint arXiv:2405.20094}, 2024.

\bibitem[Beck et~al.(2019)Beck, E, and Jentzen]{beck2019machine}
C.~Beck, W.~E, and A.~Jentzen.
\newblock Machine learning approximation algorithms for high-dimensional fully nonlinear partial differential equations and second-order backward stochastic differential equations.
\newblock \emph{Journal of Nonlinear Science}, 29\penalty0 (4):\penalty0 1563--1619, 2019.

\bibitem[Benth et~al.(2023)Benth, Detering, and Galimberti]{benth2023neural}
F.~E. Benth, N.~Detering, and L.~Galimberti.
\newblock Neural networks in {F}r{\'e}chet spaces.
\newblock \emph{Annals of Mathematics and Artificial Intelligence}, 91\penalty0 (1):\penalty0 75--103, 2023.

\bibitem[Bilokopytov and Xanthos(2025)]{bilokopytov2025universal}
E.~Bilokopytov and F.~Xanthos.
\newblock A universal approximation theorem and its applications to vector lattice theory.
\newblock \emph{ArXiv preprint arXiv:2507.20219}, 2025.

\bibitem[Bolcskei et~al.(2019)Bolcskei, Grohs, Kutyniok, and Petersen]{bolcskei2019optimal}
Helmut Bolcskei, Philipp Grohs, Gitta Kutyniok, and Philipp Petersen.
\newblock Optimal approximation with sparsely connected deep neural networks.
\newblock \emph{SIAM Journal on Mathematics of Data Science}, 1\penalty0 (1):\penalty0 8--45, 2019.

\bibitem[Cao and Wan(2022)]{cao2022expansion}
D.~Cao and J.~Wan.
\newblock Expansion of {G}reen's function and regularity of {R}obin's function for elliptic operators in divergence form.
\newblock \emph{Annali della Scuola Normale Superiore di Pisa - Classe di Scienze}, to appear, 2022.

\bibitem[Chan et~al.(2015)Chan, Jia, Gao, Lu, Zeng, and Ma]{chan2015pcanet}
T.-H. Chan, K.~Jia, S.~Gao, J.~Lu, Z.~Zeng, and Y.~Ma.
\newblock {PCAN}et: simple deep learning baseline for image classification?
\newblock \emph{IEEE Transactions on Image Processing}, 24\penalty0 (12):\penalty0 5017--5032, 2015.

\bibitem[Chassagneux et~al.(2023)Chassagneux, Chen, Frikha, and Zhou]{chassagneux2023learning}
J.-F. Chassagneux, J.~Chen, N.~Frikha, and C.~Zhou.
\newblock A learning scheme by sparse grids and {P}icard approximations for semilinear parabolic pdes.
\newblock \emph{IMA Journal of Numerical Analysis}, 43\penalty0 (5):\penalty0 3109--3168, 2023.

\bibitem[Chen and Chen(1993)]{chen1993approximations}
T.~Chen and H.~Chen.
\newblock Approximations of continuous functionals by neural networks with application to dynamic systems.
\newblock \emph{IEEE Transactions on Neural Networks}, 4\penalty0 (6):\penalty0 910--918, 1993.

\bibitem[Chen and Chen(1995)]{chen1995universal}
T.~Chen and H.~Chen.
\newblock Universal approximation to nonlinear operators by neural networks with arbitrary activation functions and its application to dynamical systems.
\newblock \emph{IEEE Transactions on Neural Networks}, 6\penalty0 (4):\penalty0 911--917, 1995.

\bibitem[Cheridito et~al.(2007)Cheridito, Soner, Touzi, and Victoir]{cheridito2007second}
P.~Cheridito, H.~M. Soner, N.~Touzi, and N.~Victoir.
\newblock Second-order backward stochastic differential equations and fully nonlinear parabolic {PDE}s.
\newblock \emph{Communications on Pure and Applied Mathematics}, 60\penalty0 (7):\penalty0 1081--1110, 2007.

\bibitem[Cuchiero et~al.(2023)Cuchiero, Schmocker, and Teichmann]{cuchiero2023global}
C.~Cuchiero, P.~Schmocker, and J.~Teichmann.
\newblock Global universal approximation of functional input maps on weighted spaces.
\newblock \emph{ArXiv preprint arXiv:2306.03303}, 2023.

\bibitem[Daubechies(1988)]{daubechies1988orthonormal}
I.~Daubechies.
\newblock Orthonormal bases of compactly supported wavelets.
\newblock \emph{Communications on Pure and Applied Mathematics}, 41\penalty0 (7):\penalty0 909--996, 1988.

\bibitem[Daubechies(1992)]{daubechies1992ten}
I.~Daubechies.
\newblock \emph{Ten lectures on wavelets}, volume~61 of \emph{CBMS--NSF regional conference series in applied mathematics}.
\newblock Society for Industrial and Applied Mathematics, Philadelphia, Pennsylvania, 1992.

\bibitem[de~Hoop et~al.(2022)de~Hoop, Lassas, and Wong]{de2022deep}
M.~V. de~Hoop, M.~Lassas, and C.~A. Wong.
\newblock Deep learning architectures for nonlinear operator functions and nonlinear inverse problems.
\newblock \emph{Mathematical Statistics and Learning}, 4\penalty0 (1):\penalty0 1--86, 2022.

\bibitem[de~Marco(2011)]{demarco2011smoothness}
S.~de~Marco.
\newblock Smoothness and asymptotic estimates of densities for {SDE}s with locally smooth coefficients and applications to square root--type diffusions.
\newblock \emph{The Annals of Applied Probability}, 21\penalty0 (4):\penalty0 1282--1321, 2011.

\bibitem[DeVore and Sharpley(1993)]{devore1993besov}
R.~A. DeVore and R.~C. Sharpley.
\newblock Besov spaces on domains in $\mathbb{{R}^d}$.
\newblock \emph{Transactions of the American Mathematical Society}, 335\penalty0 (2):\penalty0 843--864, 1993.

\bibitem[DeVore et~al.(2021)DeVore, Hanin, and Petrova]{devore2021neural}
Ronald DeVore, Boris Hanin, and Guergana Petrova.
\newblock Neural network approximation.
\newblock \emph{Acta Numerica}, 30:\penalty0 327--444, 2021.

\bibitem[Duong(2023)]{duong2023solving}
H.~Duong.
\newblock \emph{Solving high-dimensional fully nonlinear convex partial differential equations using deep learning}.
\newblock PhD thesis, Florida State University, 2023.

\bibitem[E and Wang(2018)]{e2018exponential}
W.~E and Q.~Wang.
\newblock Exponential convergence of the deep neural network approximation for analytic functions.
\newblock \emph{Science China Mathematics}, 61\penalty0 (10):\penalty0 1733--1740, 2018.

\bibitem[Evans(2010)]{evans2010partial}
L.~C. Evans.
\newblock \emph{Partial differential equations}, volume~19 of \emph{Graduate studies in mathematics}.
\newblock American Mathematical Society, 2nd edition, 2010.

\bibitem[Fefferman(2006)]{fefferman2006whitney}
C.~Fefferman.
\newblock Whitney's extension problem for ${C}^m$.
\newblock \emph{Annals of Mathematics}, 164:\penalty0 313--359, 2006.

\bibitem[Fefferman et~al.(2014)Fefferman, Israel, and Luli]{fefferman2014sobolev}
C.~Fefferman, A.~Israel, and G.~Luli.
\newblock Sobolev extension by linear operators.
\newblock \emph{Journal of the American Mathematical Society}, 27\penalty0 (1):\penalty0 69--145, 2014.

\bibitem[Firoozi et~al.(2025)Firoozi, Kratsios, and Yang]{firoozi2025simultaneously}
D.~Firoozi, A.~Kratsios, and X.~Yang.
\newblock Simultaneously solving infinitely many {LQ} mean field games in {H}ilbert spaces: the power of neural operators.
\newblock \emph{ArXiv preprint arXiv:2510.20017}, 2025.

\bibitem[Firouzi et~al.(2025)Firouzi, Yang, and Kratsios]{firouzi2025neural}
D.~Firouzi, X.~Yang, and A.~Kratsios.
\newblock Simultaneously solving infinitely many {LQ} mean field games in {H}ilbert spaces: the power of neural operators.
\newblock \emph{In preparation}, 2025.

\bibitem[Furuya and Kratsios(2024)]{furuya2024simultaneously}
T.~Furuya and A.~Kratsios.
\newblock Simultaneously solving {FBSDEs} with neural operators of logarithmic depth, constant width, and sub-linear rank.
\newblock \emph{ArXiv preprint arXiv:2410.14788}, 2024.

\bibitem[Furuya et~al.(2025)Furuya, Taniguchi, and Okuda]{furuya2025quantitative}
T.~Furuya, K.~Taniguchi, and S.~Okuda.
\newblock Quantitative approximation for neural operators in nonlinear parabolic equations.
\newblock In \emph{The thirteenth international conference on learning representations $($ICLR 2025$)$, April 24--28, 2025, Singapore}, pages 1--29, 2025.

\bibitem[Galimberti et~al.(2025)Galimberti, Kratsios, and Livieri]{galimberti2022designing}
L.~Galimberti, A.~Kratsios, and G.~Livieri.
\newblock Designing universal causal deep learning models: the case of infinite-dimensional dynamical systems from stochastic analysis.
\newblock \emph{Constructive Approximation}, to appear, 2025.

\bibitem[Germain et~al.(2022{\natexlab{a}})Germain, Lauri{\`e}re, Pham, and Warin]{germain2022deepsets}
M.~Germain, M.~Lauri{\`e}re, H.~Pham, and X.~Warin.
\newblock Deep{S}ets and their derivative networks for solving symmetric {PDE}s.
\newblock \emph{Journal of Scientific Computing}, 91\penalty0 (63):\penalty0 1--33, 2022{\natexlab{a}}.

\bibitem[Germain et~al.(2022{\natexlab{b}})Germain, Pham, and Warin]{germain2022approximation}
M.~Germain, H.~Pham, and X.~Warin.
\newblock Approximation error analysis of some deep backward schemes for nonlinear {PDE}s.
\newblock \emph{SIAM Journal on Scientific Computing}, 44\penalty0 (1):\penalty0 A28--A56, 2022{\natexlab{b}}.

\bibitem[Germain et~al.(2023)Germain, Pham, and Warin]{germain2023neural}
M.~Germain, H.~Pham, and X.~Warin.
\newblock Neural networks--based algorithms for stochastic control and {PDE}s in finance.
\newblock In A.~Capponi and C.-A. Lehalle, editors, \emph{Machine learning and data sciences for financial markets}, pages 426--452. Cambridge University Press, 2023.

\bibitem[Gilbarg and Trudinger(2001)]{gilbarg2001elliptic}
D.~Gilbarg and N.~S. Trudinger.
\newblock \emph{Elliptic partial differential equations of second order}, volume 224 of \emph{Classics in mathematics}.
\newblock Springer Berlin, Heidelberg, second edition, 2001.

\bibitem[G{\"o}deke and Fernsel(2025)]{godeke2025universal}
J.~G{\"o}deke and P.~Fernsel.
\newblock New universal operator approximation theorem for encoder--decoder architectures.
\newblock \emph{ArXiv preprint arXiv:2503.24092}, 2025.

\bibitem[Gribonval et~al.(2022)Gribonval, Kutyniok, Nielsen, and Voigtlaender]{gribonval2022approximation}
R.~Gribonval, G.~Kutyniok, M.~Nielsen, and F.~Voigtlaender.
\newblock Approximation spaces of deep neural networks.
\newblock \emph{Constructive Approximation}, 55\penalty0 (1):\penalty0 259--367, 2022.

\bibitem[Herde et~al.(2024)Herde, Raoni\'c, Rohner, K{\"a}ppeli, Molinaro, de~B{\'e}zenac, and Mishra]{herde2024poseidon}
Maximilian Herde, B.~Raoni\'c, T.~Rohner, R.~K{\"a}ppeli, R.~Molinaro, E.~de~B{\'e}zenac, and S.~Mishra.
\newblock Poseidon: efficient foundation models for {PDE}s.
\newblock In \emph{Proceedings of the 38th conference on advances in neural information processing systems $($NeurIPS 2024$)$, December 10--15, 2024, Vancouver, British Columbia, Canada}, volume~37, pages 72525--72624, 2024.

\bibitem[Hong and Kratsios(2024)]{hong2024bridging}
Ruiyang Hong and Anastasis Kratsios.
\newblock Bridging the gap between approximation and learning via optimal approximation by relu mlps of maximal regularity.
\newblock \emph{arXiv preprint arXiv:2409.12335}, 2024.

\bibitem[Horvath et~al.(2023)Horvath, Kratsios, Limmer, and Yang]{horvath2023deep}
B.~Horvath, A.~Kratsios, Y.~Limmer, and X.~Yang.
\newblock Deep {K}alman filters can filter.
\newblock \emph{SSRN preprint 4615215}, 2023.

\bibitem[Horvath et~al.(2025)Horvath, Kratsios, Limmer, and Yang]{horvath2025transformers}
B.~Horvath, A.~Kratsios, Y.~Limmer, and X.~Yang.
\newblock Transformers can solve non-linear and non-{M}arkovian filtering problems in continuous time for conditionally {G}aussian signals.
\newblock \emph{ArXiv preprint arXiv:2310.19603}, 2025.

\bibitem[Hu and Lauri{\`e}re(2024)]{hu2024recent}
R.~Hu and N.~Lauri{\`e}re.
\newblock Recent developments in machine learning methods for stochastic control and games.
\newblock \emph{Numerical Algebra, Control and Optimization}, 14\penalty0 (3):\penalty0 435--525, 2024.

\bibitem[Kim and Sakellaris(2019)]{kim2019green}
S.~Kim and G.~Sakellaris.
\newblock Green's function for second order elliptic equations with singular lower order coefficients.
\newblock \emph{Communications in Partial Differential Equations}, 44\penalty0 (3):\penalty0 228--270, 2019.

\bibitem[Korolev(2022)]{korolev2022two}
Y.~Korolev.
\newblock Two-layer neural networks with values in a {B}anach space.
\newblock \emph{SIAM Journal on Mathematical Analysis}, 54\penalty0 (6):\penalty0 6358--6389, 2022.

\bibitem[Kovachki et~al.(2023)Kovachki, Li, Liu, Azizzadenesheli, Bhattacharya, Stuart, and Anandkumar]{kovachki2023neural}
N.~B. Kovachki, Z.~Li, B.~Liu, K.~Azizzadenesheli, K.~Bhattacharya, A.~M. Stuart, and A.~Anandkumar.
\newblock Neural operator: learning maps between function spaces with applications to {PDE}s.
\newblock \emph{Journal of Machine Learning Research}, 24\penalty0 (89):\penalty0 1--97, 2023.

\bibitem[Kratsios and Furuya(2025)]{kratsios2025kolmogorov}
A.~Kratsios and T.~Furuya.
\newblock Kolmogorov--{A}rnold networks: approximation and learning guarantees for functions and their derivatives.
\newblock \emph{ArXiv preprint arXiv:2504.15110}, 2025.

\bibitem[Kratsios et~al.(2023)Kratsios, Liu, Lassas, de~Hoop, and Dokmanic]{kratsios2023universal}
A.~Kratsios, C.~Liu, M.~Lassas, M.~V. de~Hoop, and I.~Dokmanic.
\newblock Universal geometric deep learning via geometric attention.
\newblock \emph{ArXiv preprint arXiv:2304.12231}, 2023.

\bibitem[Kratsios et~al.(2024)Kratsios, Furuya, Benitez, Lassas, and de~Hoop]{kratsios2024mixture}
A.~Kratsios, T.~Furuya, J.~A.~L. Benitez, M.~Lassas, and M.~de~Hoop.
\newblock Mixture of experts soften the curse of dimensionality in operator learning.
\newblock \emph{ArXiv preprint arXiv:2404.09101}, 2024.

\bibitem[Kratsios et~al.(2025{\natexlab{a}})Kratsios, Neufeld, and Schmocker]{kratsios2025generative}
A.~Kratsios, A.~Neufeld, and P.~Schmocker.
\newblock Generative neural operators of log-complexity can simultaneously solve infinitely many convex programs.
\newblock \emph{ArXiv preprint arXiv:2508.14995}, 2025{\natexlab{a}}.

\bibitem[Kratsios et~al.(2025{\natexlab{b}})Kratsios, Schmocker, and Zimmermann]{kratsios2025deep}
A.~Kratsios, P.~Schmocker, and P.~Zimmermann.
\newblock Deep inverse problem for double phase equation.
\newblock \emph{In preparation}, 2025{\natexlab{b}}.

\bibitem[Kratsios and Zamanlooy(2022)]{kratsios2022do}
K.~Kratsios and B.~Zamanlooy.
\newblock Do {R}e{LU} networks have an edge when approximating compactly-supported functions?
\newblock \emph{Transactions on Machine Learning Research}, August:\penalty0 1--22, 2022.

\bibitem[Krylov(2018)]{krylov2018sobolev}
N.~V. Krylov.
\newblock \emph{Sobolev and viscosity solutions for fully nonlinear elliptic and parabolic equations}, volume 233 of \emph{Mathematical surveys and monographs}.
\newblock American Mathematical Society, Providence, Rhode Island, 2018.

\bibitem[Lanthaler and Stuart(2025)]{lanthaler2025parametric}
S.~Lanthaler and A.~M. Stuart.
\newblock The parametric complexity of operator learning.
\newblock \emph{IMA Journal of Numerical Analysis}, to appear, 2025.

\bibitem[Lanthaler et~al.(2022)Lanthaler, Mishra, and Karniadakis]{lanthaler2022error}
S.~Lanthaler, S.~Mishra, and G.~E. Karniadakis.
\newblock Error estimates for {D}eep{ON}ets: a deep learning framework in infinite dimensions.
\newblock \emph{Transactions of Mathematics and Its Applications}, 6\penalty0 (1):\penalty0 1--141, 2022.

\bibitem[Lanthaler et~al.(2025)Lanthaler, Li, and Stuart]{lanthaler2025nonlocality}
S.~Lanthaler, Z.~Li, and A.~M. Stuart.
\newblock Nonlocality and nonlinearity implies universality in operator learning.
\newblock \emph{Constructive Approximation}, 62:\penalty0 261--303, 2025.

\bibitem[Lefebvre et~al.(2023)Lefebvre, Loeper, and Pham]{lefebvre2023differential}
W.~Lefebvre, G.~Loeper, and H.~Pham.
\newblock Differential learning methods for solving fully nonlinear {PDE}s.
\newblock \emph{Digital Finance}, 5\penalty0 (1):\penalty0 183--229, 2023.

\bibitem[Li et~al.(2020)Li, Tang, and Yu]{li2020better}
B.~Li, S.~Tang, and H.~Yu.
\newblock Better approximations of high dimensional smooth functions by deep neural networks with rectified power units.
\newblock \emph{Communications in Computational Physics}, 27:\penalty0 379--411, 2020.

\bibitem[Li et~al.(2021)Li, Kovachki, Azizzadenesheli, Liu, Bhattacharya, Stuart, and Anandkumar]{li2021fourier}
Z.~Li, N.~Kovachki, K.~Azizzadenesheli, B.~Liu, K.~Bhattacharya, A.~Stuart, and A.~Anandkumar.
\newblock Fourier neural operator for parametric partial differential equations.
\newblock In \emph{International conference on learning representations $($ICLR 2021$)$}, pages 1--16, 2021.

\bibitem[Li et~al.(2023)Li, Meidani, and Farimani]{li2023transformer}
Z.~Li, K.~Meidani, and A.~B. Farimani.
\newblock Transformer for partial differential equations' operator learning.
\newblock \emph{Transactions on Machine Learning Research}, April:\penalty0 1--34, 2023.

\bibitem[Liu et~al.(2025)Liu, Wang, Vaidya, Ruehle, Halverson, Soljacic, Hou, and Tegmark]{liu2025kan}
Z.~Liu, Y.~Wang, S.~Vaidya, F.~Ruehle, J.~Halverson, M.~Soljacic, T.~Y. Hou, and M.~Tegmark.
\newblock {KAN}: {K}olmogorov--{A}rnold networks.
\newblock In \emph{The thirteenth international conference on learning representations $($ICLR 2025$)$}, pages 1--47, 2025.

\bibitem[Lu et~al.(2019)Lu, Jin, and Karniadakis]{lu2019deeponet}
L.~Lu, P.~Jin, and G.~E. Karniadakis.
\newblock Deep{ON}et: learning nonlinear operators for identifying differential equations based on the universal approximation theorem of operators.
\newblock \emph{ArXiv preprint arXiv:1910.03193}, 2019.

\bibitem[Lu et~al.(2021)Lu, Jin, Pang, Zhang, and Karniadakis]{lu2021learning}
L.~Lu, P.~Jin, G.~Pang, Z.~Zhang, and G.~E. Karniadakis.
\newblock Learning nonlinear operators via {D}eep{ON}et based on the universal approximation theorem of operators.
\newblock \emph{Nature Machine Intelligence}, 3:\penalty0 218--229, 2021.

\bibitem[Mallat(1989)]{mallat1989multiresolution}
S.~G. Mallat.
\newblock Multiresolution approximations and wavelet orthonormal bases of ${L}^2({\mathbb{r}})$.
\newblock \emph{Transactions of the American Mathematical Society}, 315\penalty0 (1):\penalty0 69--87, 1989.

\bibitem[Marcati and Schwab(2023)]{marcati2023exponential}
C.~Marcati and C.~Schwab.
\newblock Exponential convergence of deep operator networks for elliptic partial differential equations.
\newblock \emph{SIAM Journal on Numerical Analysis}, 61\penalty0 (3):\penalty0 1513--1545, 2023.

\bibitem[Marcati and Schwab(2025)]{marcati2025expression}
C.~Marcati and C.~Schwab.
\newblock Expression rates of neural operators for linear elliptic {PDE}s in polytopes.
\newblock \emph{Foundations of Computational Mathematics}, to appear, 2025.

\bibitem[Mhaskar(1996)]{mhaskar1996neural}
H.~N. Mhaskar.
\newblock Neural networks for optimal approximation of smooth and analytic functions.
\newblock \emph{Neural Computation}, 8\penalty0 (1):\penalty0 164--177, 1996.

\bibitem[Mhaskar and Micchelli(1995)]{mhaskar1995degree}
H.~N. Mhaskar and C.~A. Micchelli.
\newblock Degree of approximation by neural and translation networks with a single hidden layer.
\newblock \emph{Advances in Applied Mathematics}, 16\penalty0 (2):\penalty0 151--183, 1995.

\bibitem[Mhaskar and Micchelli(1992)]{mhaskar1992approximation}
H.~N. Mhaskar and Charles~A. Micchelli.
\newblock Approximation by superposition of sigmoidal and radial basis functions.
\newblock \emph{Advances in Applied Mathematics}, 13\penalty0 (3):\penalty0 350--373, 1992.

\bibitem[Munkres(2000)]{munkres2000topology}
J.~R. Munkres.
\newblock \emph{Topology}.
\newblock Prentice Hall, Inc., Upper Saddle River, NJ, second edition, 2000.

\bibitem[Neufeld and Schmocker(2023)]{neufeld2023universal}
A.~Neufeld and P.~Schmocker.
\newblock Universal approximation property of {B}anach space--valued random feature models including random neural networks.
\newblock \emph{ArXiv preprint arXiv:2312.08410}, 2023.

\bibitem[Nguwi et~al.(2024)Nguwi, Penent, and Privault]{nguwi2024deep}
J.~Y. Nguwi, G.~Penent, and N.~Privault.
\newblock A deep branching solver for fully nonlinear partial differential equations.
\newblock \emph{Journal of Computational Physics}, 499:\penalty0 112712, 2024.

\bibitem[N{\"u}sken and Richter(2021)]{nusken2021solving}
N.~N{\"u}sken and L.~Richter.
\newblock Solving high-dimensional {H}amilton--{J}acobi--{B}ellman {PDE}s using neural networks: perspectives from the theory of controlled diffusions and measures on path space.
\newblock \emph{Partial Differential Equations and Applications}, 2\penalty0 (48):\penalty0 1--48, 2021.

\bibitem[Pak et~al.(2025)Pak, Hwang, and Kim]{pak2025nonequidistant}
C.-G. Pak, H.-J. Hwang, and M.-C. Kim.
\newblock A nonequidistant multistep scheme for second order backward stochastic differential equations with applications to stochastic optimal control.
\newblock \emph{International Journal of Applied and Computational Mathematics}, 11\penalty0 (58):\penalty0 1--19, 2025.

\bibitem[Pardoux(1998)]{pardoux1998backward}
\'E. Pardoux.
\newblock Backward stochastic differential equations and viscosity solutions of systems of semilinear parabolic and elliptic {PDE}s of second order.
\newblock In L.~Decreusefond, B.~{\O}ksendal, J.~Gjerde, and {\"U}st{\"u}nel A.S., editors, \emph{Stochastic analysis and related topics VI. Proceedings of the sixth Oslo--Silivri workshop, Geilo, 1996}, volume~42 of \emph{Progress in probability}, pages 79--127, 1998.

\bibitem[Pereira et~al.(2020)Pereira, Wang, Chen, Reed, and Theodorou]{pereira2020feynman}
M.~Pereira, Z.~Wang, T.~Chen, E.~Reed, and E.~Theodorou.
\newblock Feynman--{K}ac neural network architectures for stochastic control using second-order {FBSDE} theory.
\newblock In \emph{Proceedings of the 2nd conference on learning for dynamics and control}, volume 120 of \emph{Proceedings of machine learning research}, pages 728--738, 2020.

\bibitem[Pham et~al.(2021)Pham, Warin, and Germain]{pham2021neural}
H.~Pham, X.~Warin, and M.~Germain.
\newblock Neural networks-based backward scheme for fully nonlinear {PDE}s.
\newblock \emph{SN Partial Differential Equations and Applications}, 2\penalty0 (16):\penalty0 1--24, 2021.

\bibitem[Pollard(1984)]{pollard1984convergence}
D.~Pollard.
\newblock \emph{Convergence of stochastic processes}.
\newblock Springer series in statistics. Springer New York, NY, 1984.

\bibitem[Possama{\"\i} and Tan(2015)]{possamai2015weak}
D.~Possama{\"\i} and X.~Tan.
\newblock Weak approximation of second-order {BSDE}s.
\newblock \emph{The Annals of Applied Probability}, 25\penalty0 (5):\penalty0 2535--2562, 2015.

\bibitem[Raoni\'c et~al.(2023)Raoni\'c, Molinaro, de~Ryck, Rohner, Bartolucci, Alaifari, Mishra, and de~B{\'e}zenac]{raonic2023convolutional}
B.~Raoni\'c, R.~Molinaro, T.~de~Ryck, T.~Rohner, F.~Bartolucci, R.~Alaifari, S.~Mishra, and E.~de~B{\'e}zenac.
\newblock Convolutional neural operators for robust and accurate learning of {PDEs}.
\newblock In \emph{Proceedings of the 37th conference on advances in neural information processing systems $($NeurIPS 2023$)$, December 10--16, 2023, New Orleans, Louisiana, United States of America}, volume~36, pages 77187--77200, 2023.

\bibitem[Ren and Tan(2017)]{ren2015convergence}
Z.~Ren and X.~Tan.
\newblock On the convergence of monotone schemes for path-dependent {PDE}.
\newblock \emph{Stochastic Processes and their Applications}, 127\penalty0 (6):\penalty0 1738--1762, 2017.

\bibitem[Riedi et~al.(2023)Riedi, Balestriero, and Baraniuk]{riedi2023singular}
R.~H. Riedi, R.~Balestriero, and R.~G. Baraniuk.
\newblock Singular value perturbation and deep network optimization.
\newblock \emph{Constructive Approximation}, 57\penalty0 (2):\penalty0 807--852, 2023.

\bibitem[Schneider et~al.(2025)Schneider, Ullrich, and Vybiral]{schneider2025nonlocal}
Cornelia Schneider, Mario Ullrich, and Jan Vybiral.
\newblock Nonlocal techniques for the analysis of deep relu neural network approximations.
\newblock \emph{arXiv preprint arXiv:2504.04847}, 2025.

\bibitem[Schwab et~al.(2025)Schwab, Stein, and Zech]{schwab2025deep}
C.~Schwab, A.~Stein, and J.~Zech.
\newblock Deep operator network approximation rates for {L}ipschitz operators.
\newblock \emph{Analysis and Applications}, to appear, 2025.

\bibitem[Shen et~al.(2022)Shen, Yang, and Zhang]{shen2022deep}
Z.~Shen, H.~Yang, and S.~Zhang.
\newblock Deep network approximation: achieving arbitrary accuracy with fixed number of neurons.
\newblock \emph{Journal of Machine Learning Research}, 23\penalty0 (276):\penalty0 1--60, 2022.

\bibitem[Soner et~al.(2012)Soner, Touzi, and Zhang]{soner2012wellposedness}
H.~M. Soner, N.~Touzi, and J.~Zhang.
\newblock Wellposedness of second order backward {SDE}s.
\newblock \emph{Probability Theory and Related Fields}, 153\penalty0 (1--2):\penalty0 149--190, 2012.

\bibitem[Taylor(2023)]{taylor2023partial}
M.~E. Taylor.
\newblock \emph{Partial differential equations {I}. {B}asic theory}, volume 115 of \emph{Applied mathematical sciences}.
\newblock Springer Cham, third edition, 2023.

\bibitem[Triebel(2006)]{triebel2006theory}
H.~Triebel.
\newblock \emph{Theory of function spaces. {III}}, volume 100 of \emph{Monographs in mathematics}.
\newblock Birkh{\"a}user Basel, 2006.

\bibitem[Triebel(2008)]{triebel2008function}
H.~Triebel.
\newblock \emph{Function spaces and wavelets on domains}, volume~7 of \emph{Tracts in mathematics}.
\newblock European Mathematical Society, Z{\"u}rich, 2008.

\bibitem[Tripura and Chakraborty(2022)]{tripura2022wavelet}
T.~Tripura and S.~Chakraborty.
\newblock Wavelet neural operator: a neural operator for parametric partial differential equations.
\newblock \emph{ArXiv preprint arXiv:2205.02191}, 2022.

\bibitem[Wang et~al.(2024)Wang, Siegel, Liu, and Hou]{wang2024expressiveness}
Y.~Wang, J.~W. Siegel, Z.~Liu, and T.~Y. Hou.
\newblock On the expressiveness and spectral bias of {KAN}s.
\newblock \emph{ArXiv preprint arXiv:2410.01803}, 2024.

\bibitem[Xiao et~al.(2024)Xiao, Qiu, and Nikan]{xiao2024numerical}
X.~Xiao, W.~Qiu, and O.~Nikan.
\newblock Numerical approximation based on deep convolutional neural network for high-dimensional fully nonlinear merged {PDE}s and {2BSDE}s.
\newblock \emph{Mathematical Methods in the Applied Sciences}, 47\penalty0 (7):\penalty0 6184--6204, 2024.

\bibitem[Yang et~al.(2019)Yang, Zhao, and Zhou]{yang2019explicit}
J.~Yang, W.~Zhao, and T.~Zhou.
\newblock Explicit deferred correction methods for second-order forward backward stochastic differential equations.
\newblock \emph{Journal of Scientific Computing}, 79\penalty0 (3):\penalty0 1409--1432, 2019.

\bibitem[Yarotsky(2017)]{yarotsky2017error}
D.~Yarotsky.
\newblock Error bounds for approximations with deep {ReLU} networks.
\newblock \emph{Neural Networks}, 94:\penalty0 103--114, 2017.

\bibitem[Yu et~al.(2021)Yu, Becquey, Halikias, Mallory, and Townsend]{yu2021arbitrary}
A.~Yu, C.~Becquey, D.~Halikias, M.~E. Mallory, and A.~Townsend.
\newblock Arbitrary-depth universal approximation theorems for operator neural networks.
\newblock \emph{ArXiv preprint arXiv:2109.11354}, 2021.

\bibitem[Zhou et~al.(2021)Zhou, Han, and Lu]{zhou2021actor}
M.~Zhou, J.~Han, and J.~Lu.
\newblock Actor--critic method for high dimensional static {H}amilton--{J}acobi--{B}ellman partial differential equations based on neural networks.
\newblock \emph{SIAM Journal on Scientific Computing}, 43\penalty0 (6):\penalty0 A4043--A4066, 2021.

\end{thebibliography}

\end{document}